\documentclass[11pt]{article}

\usepackage{fullpage}

\usepackage[utf8]{inputenc} 
\usepackage[T1]{fontenc}    
\usepackage{hyperref}
\usepackage{url}            
\usepackage{booktabs}       
\usepackage{amsfonts}       
\usepackage{nicefrac}       
\usepackage{microtype}      
\usepackage{algorithm}
\usepackage{algorithmic}
\usepackage{amsmath}
\usepackage{bm}
\usepackage{color}
\usepackage{mcode}
\usepackage{subcaption}
\usepackage{amsthm}
\usepackage{tabularx}
\usepackage{caption}
\usepackage{fancyhdr}
\usepackage{amssymb}
\usepackage{pifont}

\newcommand{\comment}[1]{}
\newenvironment{myitemize}[1][]{
\begin{list}{{\hei #1}} 
    {
     \setlength{\leftmargin}{0.4em}     
     \setlength{\parsep}{0.2em}         
     \setlength{\topsep}{0.2em}         
     \setlength{\itemsep}{0.2em}        
     \setlength{\labelsep}{0.4em}     
     \setlength{\itemindent}{0.1em}    
     \setlength{\listparindent}{1.4em} 
    }}
{\end{list}}

\newcommand{\op}{\mbox{op}}

\newtheorem{thm}{Theorem}
\newtheorem{lem}{Lemma}
\newtheorem{cor}{Corollary}

\newtheorem{defn}{Definition}

\newtheorem{assum}{Assumption}
\DeclareMathOperator*{\argmin}{argmin}

\newcommand{\vecto}{\mcode{vec}}
\newcommand{\vect}[1]{\mcode{vec}\left(#1\right)}

\newcommand{\Rsss}[3]{{\mathbb{R}^{#1\times#2\times#3}}}

\newcommand{\nee}{{n_\epsilon}}
\newcommand{\kt}{{k_{\bm{w}}}}
\newcommand{\Jhn}{{\hat{\bm{J}}_{n}}}
\newcommand{\led}[1]{\overset{\text{\ding{#1}}}{\leq}}

\renewcommand{\SS}{\bm{\mathbb{S}}}

\newcommand{\Ss}{\bm{\mathcal{S}}}
\newcommand{\lr}{\left\langle}
\newcommand{\rl}{\right\rangle}
\newcommand{\lv}{\left\|}
\newcommand{\rv}{\right\|}
\newcommand{\lvs}{\left|}
\newcommand{\rvs}{\right|}

\newcommand{\xmi}[1]{\bm{x}_{(#1)}}
\newcommand{\ymi}[1]{\bm{y}_{(#1)}}
\newcommand{\umi}[1]{\bm{u}^{(#1)}}
\newcommand{\vmi}[1]{\bm{v}^{(#1)}}

\newcommand{\emmi}[1]{\bm{e}_{(#1)}}
\newcommand{\wmi}[1]{\bm{w}_{(#1)}}
\newcommand{\wmii}[1]{\bm{w}^{(#1)}}
\newcommand{\wmin}[1]{\bm{w}_n^{(#1)}}

\newcommand{\wms}[1]{\bm{w}_*^{#1}}
\newcommand{\wmn}{\bm{w}^n}
\newcommand{\rmi}[1]{\bm{r}_{#1}}

\newcommand{\am}{\bm{a}}
\newcommand{\bmm}{\bm{b}}
\newcommand{\dm}{\bm{d}}
\newcommand{\emm}{\bm{e}}
\newcommand{\wm}{\bm{w}}
\newcommand{\xm}{\bm{x}}
\newcommand{\ym}{\bm{y}}
\newcommand{\zm}{\bm{z}}

\newcommand{\Wmi}[1]{\bm{W}^{(#1)}}
\newcommand{\Bmi}[1]{\bm{B}_{#1}}

\newcommand{\Cm}{\bm{C}}
\newcommand{\Dm}{\bm{D}}
\newcommand{\Gm}{\bm{G}}
\newcommand{\Pm}{\bm{P}}
\newcommand{\Fm}{\bm{F}}
\newcommand{\Mm}{\bm{M}}
\newcommand{\Nm}{\bm{N}}

\newcommand{\EE}{\mathbb{E}}

\newcommand{\Pro}{\mathbb{P}}

\newcommand{\lam}{\boldsymbol{\lambda}}

\newcommand{\Bi}[2]{{\sf{B}}^{#1}(#2)}

\newcommand{\Rs}[1]{\mathbb{R}^{#1}}
\newcommand{\Rss}[2]{\mathbb{R}^{#1\times #2}}

\newcommand{\A}{\bm{\mathcal{A}}}

\newcommand{\D}{\bm{\mathcal{D}}}

\newcommand{\Am}{\bm{{A}}}
\newcommand{\Bm}{\bm{{B}}}
\newcommand{\Em}{\bm{E}}

\newcommand{\Xm}{\bm{X}}

\newcommand{\Zm}{\bm{Z}}

\newcommand{\Jm}{\bm{J}}
\newcommand{\Hm}{\bm{H}}

\newcommand{\Wm}{\bm{W}}
\newcommand{\Qm}{\bm{Q}}
\renewcommand{\Im}{\bm{I}}

\newcommand{\f}{\ensuremath{\bm{f}}}

\newcommand{\x}{\ensuremath{\bm{x}}}

{%
\begin{list}{#1}{
\vspace{-\topsep}
\vspace{-\partopsep}
\setlength{\itemindent}{0cm}
\setlength{\rightmargin}{0cm}
\setlength{\listparindent}{0cm}
\settowidth{\labelwidth}{#1}
\setlength{\leftmargin}{\labelwidth}
\addtolength{\leftmargin}{\labelsep}
\setlength{\itemsep}{0cm}
}%
}%
{%
\end{list}
\vspace{-\topsep}
\vspace{-\partopsep}
}

{\begin{enumerate}%
}%
{\end{enumerate}}

\begin{document}

\title{\textbf{The Landscape of Deep Learning Algorithms}}

\author{Pan Zhou\footnote{National University of Singapore, Singapore. Email: pzhou@u.nus.edu} \and
Jiashi Feng\footnote{National University of Singapore, Singapore. Email: elefjia@nus.edu.sg}
}

\maketitle

\begin{abstract}
This paper studies the landscape of empirical risk  of deep  neural networks by theoretically analyzing its convergence behavior to the population risk as well as  its stationary points and properties. For an $l$-layer linear neural network, we  prove  its empirical risk uniformly converges to its population risk at the rate of $\mathcal{O}(r^{2l}\sqrt{d\log(l)}/\sqrt{n})$ with training sample size of $n$, the total weight dimension of $d$ and the magnitude bound $r$ of weight of each layer. We then derive  the stability and  generalization bounds for the empirical risk based on this result. Besides, we establish
the uniform convergence of gradient of the empirical risk  to its population counterpart. We prove the one-to-one correspondence of the non-degenerate stationary points between the empirical  and population risks with convergence guarantees, which  describes the landscape of deep neural networks. In addition, we  analyze these properties for deep nonlinear neural networks with sigmoid activation functions. We prove   similar results for convergence behavior of their empirical risks  as well as the gradients and analyze properties of their non-degenerate stationary points.

To our best knowledge, this work is the first one theoretically characterizing landscapes of deep learning algorithms. Besides, our  results provide  the sample complexity of training a good deep neural network. We also provide theoretical understanding on  how  the neural network depth $l$, the layer width, the network size $d$ and parameter magnitude determine the neural network landscapes.
\end{abstract}

\section{Introduction}
Deep learning algorithms have  achieved remarkable  practical successes in many fields, such as computer vision~\cite{hinton2006fast,szegedy2015going,he2016deep}, natural language processing~\cite{collobert2008unified,bakshi1993wave}, and speech recognition~\cite{hinton2012deep,graves2013speech}, to name a few. However,  theoretical understanding on properties  of these deep learning algorithms still  lags their practical achievements~\cite{shalev2017failures,Kawaguchi2016} due to their high non-convexity and internal complexity.
In practice, deep learning algorithms usually learn their  model parameters by minimizing the \emph{empirical risk} (a sum of losses associated to each training sample). Thus, we aim to analyze   landscape of the {empirical risk}  of deep learning algorithms  for better understanding their performance in practice.

Formally, we consider a deep  neural network  model consisting  of $l$ layers $(l\geq 2)$ which is trained by minimizing the commonly used squared loss function
 over samples $\xm\in\Rs{\dm_0}$ from unknown distribution  $\D$.
Ideally,  deep learning algorithms can find the optimal parameter $\wm^*$ by minimizing the  \textit{population risk}:
\begin{equation*}\label{PRM}
\min_{\wm} \Jm(\wm) \triangleq \EE_{\xm\sim\D}\ f(\wm,\xm),
\end{equation*}
where $\wm$ is the model parameter and $f(\wm,\xm)=\frac{1}{2}\|\vmi{l}-\ym\|_2^2$ is the squared  loss associated to the sample $\xm \sim \D$. Here $\vmi{l}$ is the output of the $l$-th layer and $\ym$ is the target output for the sample $\xm$. In practice,  as the sample distribution $\D$ is usually unknown and only finite
training samples $\left\{\xmi{i},\ymi{i}\right\} _{i=1}^{n}$  \textit{i.i.d.}\ drawn from $\D$ are provided, one usually trains the network model by minimizing  the {empirical risk}:
\begin{equation*}\label{ERM}
\min_{\wm} \Jhn(\wm) \triangleq \frac{1}{n}\sum_{i=1}^{n} f(\wm,\xmi{i}).
\end{equation*}
In this work, we characterize the landscape of empirical risk $\Jhn(\wm)$ of deep learning algorithms by analyzing its convergence behavior  to the  population risk $\Jm(\wm)$ as well as its stationary points and  properties, for both multi-layer linear and nonlinear neural networks. In particular, we first prove the uniform convergence of the empirical risk $\Jhn(\wm)$  to its population risk $\Jm(\wm)$ with the convergence rate of $\mathcal{O}(r^{2l}\sqrt{d\log(l)}/\sqrt{n})$ with training sample size of $n$, the total weight dimension of $d$ and the magnitude bound $r$ of weight of each layer. Such  result also  bounds the  generalization error of deep learning algorithms  and  implies  stability of their empirical risk.
 Besides, we establish  the uniform convergence rate $\mathcal{O}(r^{2l-1}\sqrt{ld\log(l)\max_j(\dm_j\dm_{j-1})}/\sqrt{n})$ of empirical gradients $\nabla \Jhn(\wm) $  to its population counterpart $\nabla \Jm(\wm)$ where $\dm_j$ denotes the output dimension of the $j$-th layer. Accordingly, as long  as the training sample size $n$ is sufficiently large,  any stationary point of $\Jhn(\wm)$ is also a stationary point of $\Jm(\wm)$ and vise versa.  We then further establish the exact correspondence of their non-degenerate stationary points. Indeed, the corresponding non-degenerate stationary points also uniformly converge to each other.  Such analysis results also reveal the role of  the depth $l$ of a neural network in the convergence behavior. Also, the width factor $\sqrt{\max_j(\dm_j\dm_{j-1})}$ and the total network size $d$ are critical to the convergence performance. In addition, controlling magnitudes of the parameters (weights) in deep neural networks are demonstrated to be important for performance. To our best knowledge, this work is the first one theoretically characterizing landscapes of both deep linear and nonlinear neural networks.

\section{Related Work}
To date, only a few theories are developed for understanding deep learning and they can be roughly divided into three categories. The first category aims to analyze the training error of deep learning. Bartlett~\cite{bartlett1997valid} first analyzed the misclassification probability of deep learning for two-classification problems. On the other hand, Baum~\cite{baum1988capabilities} pointed out that zero training error can be obtained when the last layer of a network has more units than training samples. However, when facing millions of training data, an extreme-wide network suffers from over-fitting problems and is impractical. Later, Soudry \textit{et al.}~\cite{soudry2016no} proved that for deep leaky rectified linear units (ReLU) networks with one single output, the training error at its any local minimum is zero if the product of the number of units in the last two layers is larger than the training sample size.

The second kind of  works~\cite{dauphin2014identifying,choromanska2015loss,Kawaguchi2016, tian2017layered} focus on analyzing the loss surfaces of highly nonconvex loss functions in deep learning, such as the distribution of stationary points. Those results may be helpful for understanding radically different practical performance of large- and small-size networks~\cite{choromanska2015open}. Among them, Dauphin~\textit{et al.}~\cite{dauphin2014identifying} experimentally verified the existence of a large number of saddle points in deep neural networks. With strong assumptions,  Choromanska~\textit{et al.}~\cite{choromanska2015loss} established  connection between the loss function of deep ReLU networks and the spherical spin-class model, describing the location of local minima. Later, Kawaguchi~\cite{Kawaguchi2016} proved the existence of degenerate saddle points for  deep linear neural networks with squared loss function  and the fact that any local minimum is also a global minimum, with slightly weaker assumptions. By utilizing dynamical system analysis, Tian~\cite{tian2017layered} declared that for two-layered bias-free networks with ReLUs, if the inputs follow Gaussian distribution, gradient algorithm with certain symmetric weight initialization can guarantee the global convergence to the true weights. Recently, Nguyen~\textit{et al}.~\cite{Quynh2017surface} proved that: for a fully connected network with squared loss and analytic activation functions, almost all the local minima are globally optimal\textemdash  when one hidden layer has more units than training samples and the network structure from this layer is pyramidal.

Thirdly, some recent works try to alleviate the analysis difficulty by relaxing the problems into easier ones. For instance, by utilizing the kernel strategy, Zhang~\textit{et al.}~\cite{zhang2016l1} transformed $\ell_1$-regularized multi-layer networks into single-layer convex problems which have almost the same loss as that of the original one with high probability. Later they adopted similar strategy and transformed convolutional neural network into a convex problem~\cite{zhang2016convexified}. In this way, saddle points and local minima can be avoided and the learning efficiency is also higher.

However, there are no works that analyze the landscape of the empirical risk  of deep learning algorithms. Notice, some previous works analyzed the empirical risk for single-layer optimization problems. For example, Negahban~\textit{et al.}~\cite{negahban2009unified} proved that for a regularized convex program, the minima of empirical risk uniformly converges to the true minima  of the population risk under certain conditions. Mei~\textit{et al.}~\cite{mei2016landscape} analyzed the   convergence behavior of empirical risk for nonconvex problems. However, they only considered the single-layer nonconvex problems and their analysis demands strong  sub-Gaussian and sub-exponential assumptions on  the gradient and Hession of empirical risk respectively. In contrast, we get rid of these assumptions.  Besides, they did not analyze the convergence rate of the empirical risk, stability and generalization error of deep learning which is presented in our work. Gonen~\textit{et al.}~\cite{gonen2017fast} proved that for nonconvex problems without degenerated saddle points, the difference between empirical risk and population risk can be bounded. Unfortunately, the loss of deep learning is highly nonconvex and has degenerated saddle points~\cite{fyodorov2007replica,dauphin2014identifying,Kawaguchi2016}. Thus, their analysis results are not applicable to deep learning.

\section{Preliminaries \label{sec:Preliminaries}}
Throughout the paper, we denote matrices by
boldface capital letters, \textit{e.g.} $\Am$. Vectors are denoted by boldface
lowercase letters, \textit{e.g.} $\am$, and scalars are denoted by
lowercase letters, \textit{e.g.} $a$. We define the $r$-radius ball as $\Bi{d}{r}\triangleq\{\zm\in\Rs{d}\,|\,\|\zm\|_{2}\leq r\}$.
For explaining the results, we also need the vectorization operation $\vecto(\cdot)$. It is defined as $\vecto(\Am)=\left(\Am(:,1);\cdots;\Am(:,t)\right)\in\Rs{st}$ that vectorizes $\Am\in\Rss{s}{t}$ along its columns. We use  $d\!=\!\sum_{j=1}^{l}\!\dm_j\dm_{j-1}$ to denote the total weight parameter dimension, where $\dm_j$ denotes the output dimension of the $j$-th layer (see blow).

Here we briefly describe deep linear and nonlinear neural network models.  Suppose both networks consist of $l$ layers. We use $\umi{j}$ and $\vmi{j}$ to respectively denote the input and output of the $j$-th layer, $\forall j=1,\ldots, l$.

{\textbf{Deep linear neural networks:}} the function of the $j$-th layer is formulated as
\begin{equation*}\label{formMNN}
\umi{j}\triangleq\Wmi{j}\vmi{j-1}\in\Rs{\dm_j}\,,\quad \vmi{j}\triangleq \umi{j}\in\Rs{\dm_j},\  \forall j=1,\cdots,l,
\end{equation*}
where $\vmi0=\xm$ is the input of the network; $\Wmi{j}\in\Rss{\dm_{j}}{\dm_{j-1}}$
is the weight matrix of the $j$-th layer.

{\textbf{Deep nonlinear neural networks:}} here we use the sigmoid function as the non-linear activation function. Accordingly, the function within the $j$-th layer is written as
\begin{equation*}\label{formMNN}
\umi{j}\triangleq\Wmi{j}\vmi{j-1}\in\Rs{\dm_j}\,,\quad \vmi{j}\triangleq h_j(\umi{j})=(\sigma(\umi{j}_1); \cdots;\sigma(\umi{j}_{\dm_j}))\in\Rs{\dm_j},\ \forall j=1,\cdots,l,
\end{equation*}
where $\umi{j}_{i}$ denotes the $i$-th entry of $\umi{j}$ and $\sigma(\cdot)$ is the sigmoid function, \textit{i.e.}, $\sigma(a)=1/(1+\mcode{e}^{-a})$.
Following the common practice in deep learning, both network models adopt the squared loss function. For notational simplicity, we further define $\emm\triangleq \vmi{l}-\ym$
as the output error vector, where $\vmi{l}$
is output of the network and $\ym\in\Rs{\dm_l}$ is the target output. Then the squared loss is defined  as $f(\wm,\xm)=\frac{1}{2}\|\emm\|_2^2$, where $\wm=(\wmi{1};\cdots;\wmi{l})\in\Rs{d}$ contains all the weights in the network in which $\wmi{j}=\vect{\Wmi{j}}\in\Rs{\dm_j\dm_{j-1}}$. Then the \textit{empirical risk}  $\Jhn(\wm)$ is computed  as
\begin{equation}\label{qrwwrwe}
\Jhn(\wm)=\frac{1}{n}\sum_{i=1}^{n} f(\wm,\xmi{i})=\frac{1}{2n}\sum_{i=1}^{n}\|\emmi{i}\|_2^2,
\end{equation}
where $\emmi{i}$ represents the output error of the $i$-th sample $\xmi{i}$.

\section{Results for Deep Linear Neural Networks}
We first prove the uniform convergence of the empirical risk  to the population risk for deep linear neural networks. Based on this result, we also give stability and generalization bounds. Subsequently, we present the uniform convergence guarantee of the empirical gradient to its population counterpart, and then analyze properties of  non-degenerate stationary points of the empirical risk.

In the analysis, we assume that the input data $\xm$ are $\tau^2$-sub-Gaussian and meanwhile have bounded magnitude, as stated in Assumption~\ref{assumption12}.
\begin{assum}\label{assumption12}
	The input datum  $\xm \in \mathbb{R}^{\dm_0} $ has zero mean and is $\tau^2$-sub-Gaussian. That is, $\xm$ obeys
	\begin{align*}
	\EE[\exp\left(\langle \lam,\xm\rangle\right)] \leq \exp\left(\frac{\tau^2\|\lam\|_2^2}{2}\right),\ \forall \lam\in\Rs{\dm_0}.
	\end{align*}
	Besides, the magnitude $\xm$ are bounded as $\|\xm\|_2\leq r_x$, where $r_x$ is a positive universal constant.
\end{assum}
Note that any random vector $\zm$ with independent random bounded entries is  sub-Gaussian and  satisfies Assumption~\ref{assumption12}~\cite{VRMT}. Moreover, for the parameters $\tau$ and $r_x$,
we have $\tau=\|\zm\|_\infty\leq \|\zm\|_2\leq r_x$. Here the assumption of having bounded magnitude generally holds for real data (\emph{e.g.}, images and speech signal). In addition, we also assume the weight parameters  $\wmi{j}$ of each layer to be bounded. We use $\wm\in\Omega$ to denote the constraint $\{\wm\,|\,\wmi{j}\in\Bi{\dm_j\dm_{j-1}}{\rmi{j}},\, \forall j=1,\cdots, l\}$ where $\rmi{j}$ is a constant. For notational simplicity, we let $r=\max_j \rmi{j}$. This is a common and reasonable assumption. For instance,
Xu~\textit{et al.}~\cite{xu2012robustness} use such an assumption for robustness analysis of deep neural networks.

Though we only analyze deep linear neural networks in this section, with making proper assumptions our  results can be generalized to deep ReLU neural networks by applying the results from  Choromanska~\textit{et al.}~\cite{choromanska2015loss} and Kawaguchi~\cite{Kawaguchi2016}~\textemdash~they transformed deep ReLU neural networks into deep linear neural networks. We will leave this for future work.

\subsection{Uniform Convergence,  Stability and Generalization   of Empirical Risk }
Theorem~\ref{thm:stability} gives the uniform convergence results of empirical risk for deep linear neural networks.
\begin{thm} \label{thm:stability}
Suppose Assumption~\ref{assumption12} on the input data $\xm$ holds and the activation functions in deep neural network are linear. Then there exist two universal constants $c_{f'}$ and $c_f$ such that if $n\geq c_{f'}\max(lr_x^4/(\dm_l d\varepsilon^2\tau^4\log(l)), d\log(l)/\dm_l)$, then
\begin{equation}\label{uniform_loss}
\sup_{\wm\in\Omega} \left|  \Jhn(\wm)- \Jm(\wm)\right|\leq \epsilon_l\triangleq c_f \tau  \max\left(\sqrt{\dm_l}\tau r^{2l},r^{l}\right) \sqrt{\frac{d\log(nl)+\log(8/\varepsilon)}{n}}
\end{equation}
holds with probability at least $1-\varepsilon$. Here $l$ is the number of layers in the neural network, $n$ is the sample size and $\dm_{l}$ is the dimension of the final layer.
\end{thm}
From  Theorem~\ref{thm:stability}, one can observe that with increasingly larger sample size $n$, the difference between empirical risk and population risk decreases monotonically. In particular, when $n\to +\infty$, we have $|  \Jhn(\wm)- \Jm(\wm)|\to 0$. Then according to the definition of uniform convergence~\cite{vapnik1998statistical,shalev2010learnability}, we have under the distribution $\D$, the uniform convergence rate of the empirical risk of a deep linear neural network to its population risk  is $\mathcal{O}(1/\sqrt{n})$ (up to a $\log$ factor). Theorem~\ref{thm:stability} also characterizes the role of the depth $l$ in a deep network model for obtaining small difference between the empirical risk and population risk. Specifically, a deeper neural network will incur larger difference between empirical and population risk. Thus it needs more training samples for achieving good generalization performance. This result also matches the one in~\cite{bartlett1997valid} for achieving small misclassification probability in deep learning. Also, due to the factor $d$ in the convergence rate, a network of larger size also require more training samples. Theorem~\ref{thm:stability} also suggests one should not choose the weight $\wm$ with large magnitude (reflected by the factor $r$ in the theorem) for the sake of convergence rate. Therefore, adding regularization over the weight $\wm$, such as the commonly used $\|\wm\|_{2}^2$ and $\|\wm\|_1$, indeed help avoid over-fitting.

Based on  Theorem~\ref{thm:stability}, we proceed to analyze  the stability property of the empirical risk  and the  convergence rate of the  generalization error in expectation.  Let $\Ss=\{\xmi{1},\cdots,\xmi{n}\}$ denote the sample set in which the samples are~\textit{i.i.d.}~drawn from $\D$. When the optimal solution $\wmn$ to problem~\eqref{qrwwrwe} is computed by deterministic algorithms, then the generalization error is defined as $\epsilon_g=\Jhn(\wmn)-\Jm(\wmn)$. But one usually employs randomized algorithms (\textit{e.g.} stochastic gradient descent, SGD) for computing $\wmn$. For this case, stability and generalization error in expectation defined in Definition~\ref{defi1} are  used.
\begin{defn}{\textbf{(Stability and generalization in expectation)}}~\cite{vapnik1998statistical,shalev2010learnability, gonen2017fast}\label{defi1}
Assume randomized algorithm $\Am$ is employed, $(\xmi{1}',\cdots,\xmi{n}')\sim\D$ and $\wmn=\argmin_{\wm} \Jhn(\wm)$ is the empirical risk minimizer (ERM).
For every $j\in[n]$, suppose $\wms{j}=\argmin_{\wm} \frac{1}{n-1} \sum_{i\neq j} f_i(\wm,\xmi{i})$. We say that the ERM is on average stable with stability rate $\epsilon_s$ under distribution $\D$ if $\left|\EE_{\Ss\sim\D,\Am,(\xmi{1}',\cdots,\xmi{n}')\sim\D}\right.$ $\left.\frac{1}{n} \sum_{j=1}^{n}\left(f_j(\wms{j}, \xmi{j}') -f_j(\wmn,\xmi{j}')\right)\right|\leq \epsilon_s.$ The ERM is said to have generalization error with convergence rate $\epsilon_g$ under distribution $\D$ if we have $\left|\EE_{\Ss\sim\D,\Am}\left(\Jm(\wmn)-\Jhn(\wmn)\right)\right|\leq \epsilon_g.$
\end{defn}
Stability is useful for measuring the sensibility of empirical risk to the input and generalization error  measures the effectiveness of ERM on new data. Generalization error in expectation is especially useful for deep learning algorithms considering its internal randomness (from SGD optimization).
Now we present the results on stability and generalization performance of deep linear neural networks.
\begin{cor}\label{stability and generalization}
Suppose Assumption~\ref{assumption12} on the input data $\xm$ holds and the activation functions in deep neural network are linear.
Then with probability at least $1-\varepsilon$, both the stability rate and the generalization error rate of ERM of deep linear neural network are at least $\epsilon_l$:
\begin{equation*}
\Bigg|\EE_{\Ss\sim\D,\Am,(\xmi{1}',\cdots,\xmi{n}'\!)\sim\D}\frac{1}{n}\! \sum_{j=1}^{n}\!\!\left(\!f_j(\wms{j},\!\xmi{j}'\!)\!-\!\!f_j(\wmn,\xmi{j}'\!)\! \right)\!\!\Bigg|\!\leq\! \epsilon_l,\, \ \Bigg|\EE_{\Ss\sim\D,\Am}\!\left(\!\Jm(\wmn\!)\!- \!\Jhn(\wmn\!)\!\right)\!\!\Bigg|\! \leq\! \epsilon_l,
\end{equation*}
where $\epsilon_l$ is defined in Eqn.~\eqref{uniform_loss}.
\end{cor}
According to Corollary~\ref{stability and generalization}, both the stability rate and the convergence rate of generalization error are $\mathcal{O}(1/\sqrt{n})$. This result indicates that deep learning empirical risk  is stable and its output is robust to  slight change over the input training data. When $n$ is sufficiently large, small generalization error of deep learning algorithms is also  guaranteed. Such result is helpful for explaining the practically good generalization performance of deep learning algorithms on new data.

	\textbf{Remark 1 }
Some existing works, \textit{e.g}~\cite{bartlett2003vapnik,neyshabur2015norm}, also analyzes the generalization ability of a deep neural network model. However,  their results differ from ours in the following sense. In the following discussion,  for notational simplicity, we use $\EE_{\Ss\sim\D} (\Jhn(\wm)-\Jm(\wm))$ to denote the generalization error $\EE_{\Ss\sim\D,\Am,(\xmi{1}',\cdots,\xmi{n}'\!)\sim\D}\frac{1}{n}\! \sum_{j=1}^{n}\!\!\left(\!f_j(\wms{j},\!\xmi{j}'\!)\!-\!\!f_j(\wmn,\xmi{j}'\!)\! \right)$.  Based on VC-dimension techniques, Bartlett~\textit{et al.}~\cite{bartlett2003vapnik} proved that with probability at least $1-\varepsilon$, $|\EE_{\Ss\sim\D} (\Jhn(\wm)-\Jm(\wm))| \leq  \mathcal{O}(\sqrt{(\gamma \log^2(n)+\log(1/\varepsilon))/n})$. Here $\gamma$ is the shattered parameter and can be as large as the VC-dimension of the network model,~\textit{i.e.} at the order of $\mathcal{O}(ld \log(d)+l^2 d)$. In contrast, the generalization error bound derived in Corollary~\ref{stability and generalization} is $|\EE_{\Ss\sim\D} (\Jhn(\wm)-\Jm(\wm))| \leq  \mathcal{O}(\sqrt{(d \log(nl)+\log(1/\varepsilon))/n})$ which is   tighter. Indeed, we obtain a faster  convergence rate for $\sup_{\wm\in\Omega} |\Jhn(\wm)- \Jm(\wm)|$ in Theorem~\ref{thm:stability}  than the known generalization error rate established in~\cite{bartlett2003vapnik}, although the former is more challenging to bound.

	\textbf{Remark 2 }
 In~\cite{neyshabur2015norm}, Neyshabur~\textit{et al.} proved that: for a fully-connected neural network model with ReLU activation functions  and  bounded input entries, its Rademacher complexity is $\mathcal{O}\left(r^l/\sqrt{n}\right)$ (see Corollary 2 in~\cite{neyshabur2015norm}). Then by applying Rademacher complexity based argument~\cite{Shwartz2014A}, we have $ | \EE_{\Ss\sim\D} (\Jhn(\wm)-\Jm(\wm))| \leq \mathcal{O} ((r^l+\sqrt{\log(1/\varepsilon)})/\sqrt{n})$ with probability at least $1-\varepsilon$. But our Theorem~\ref{thm:stability} provides the \emph{uniform} convergence guarantee $\sup_{\wm\in\Omega} |\Jhn(\wm)- \Jm(\wm)| \leq \mathcal{O}(r^l\sqrt{d/n})$. By comparison, our uniform convergence result holds even for the worst case where the model is not well trained. Indeed, such uniform convergence (with the sign of $\sup$) is much more difficult to bound. Applying $\epsilon$-net arguments  is possible to obtain uniform convergence bound from the generalization result in~\cite{neyshabur2015norm} but the resulted uniform convergence rate will be slower than ours. This is because $\epsilon$-net argument considers the whole parameter space $\mathbb{R}^d$ and will introduce a factor that  is at least at the order of  $\mathcal{O}(\sqrt{d})$ into the convergence rate. So our uniform convergence rate is  tight.

	\textbf{Remark 3 }	
 The generalization bound in Corollary~\ref{stability and generalization} is directly induced by  Theorem~\ref{thm:stability}. As the uniform convergence is stronger than generalization, directly applying Theorem~\ref{thm:stability} gives a slightly loose generalization bound. But our main contribution is to provide uniform convergence guarantees for the empirical loss of networks as well as uniform convergence of gradient and stationary points to their population counterparts, instead of pursuing tighter generalization bound. Uniform convergence considers the worst  learned model that concerns  deep learning practitioners more. The uniform convergence of the gradient and stationary points (see blow) has not been ever considered before. Moreover, the generalization result in~\cite{neyshabur2015norm} is not applicable to deep neural network models with sigmoid activation functions as analyzed in Sec.~\ref{sigmoid_generalization}.

\subsection{Uniform Convergence of Gradient}
Here we analyze the convergence of gradients of empirical and population risks for deep linear neural networks. Results on gradient convergence are useful for characterizing their landscapes. Our results are stated blow.

\begin{thm} \label{thm:uniformconvergence1}
Suppose Assumption~\ref{assumption12} on the input data $\xm$ holds and the activation functions in deep neural network are linear.  Then the empirical gradient uniformly converges to the population gradient in Euclidean norm. Specifically, if $n\geq c_{g'}\max(l^2r^2 r_x^4/(\dm_0d^2\varepsilon^2\tau^4 \log(l)),d\log(l))$ where $c_{g'}$ is a universal constant, there exist a universal constant $c_g$ such that
\begin{equation*}
\sup_{\wm\in\Omega}\left\| \nabla\Jhn(\wm)\!-\!\nabla\Jm(\wm)\right\|_2\!\leq\! c_g \tau \omega_g \sqrt{l\max_j(\dm_j\dm_{j-1})} \sqrt{ \frac{d\log(nl)\!+\!\log(12/\varepsilon)}{n}}
\end{equation*}
holds with probability at least $1-\varepsilon$, where $\omega_g=\max\left(\tau \sqrt{\dm_0} r^{2l-1}, \sqrt{\dm_0}  r^{2l-1}, r^{l-1} \right)$.
\end{thm}
By Theorem~\ref{thm:stability}, we can know that the convergence rate of empirical gradient to population gradient is $\mathcal{O}(1/\sqrt{n})$ (up to a  $\log$ factor). Here we can observe a factor $\sqrt{\max_j(\dm_j\dm_{j-1})}$ in the convergence rate, which suggest to avoid deep neural network  architecture of  unbalanced layer sizes (where some layers are extremely ``wide''). This result also matches the trend in deep learning applications for building deep but thin networks~\cite{he2016deep,szegedy2015going}.

Theorem~\ref{thm:uniformconvergence1} also conveys the similar properties of a point in empirical and population risk optimization  when sample number $n$ is large. For example, by Theorem~\ref{thm:uniformconvergence1}, if a point $\tilde{\wm}$ is an $\epsilon/2$-stationary point of $\Jm(\wm)$ and $n\geq c_\epsilon (\tau \omega_g/\epsilon)^2 l\max_j(\dm_j\dm_{j-1})d\log(l)$ where $c_\epsilon$ is a constant,  $\tilde{\wm}$ is also an $\epsilon$-stationary point of $\Jhn(\wm)$ with probability $1-\varepsilon$ and vice versa. Here by $\epsilon$-stationary point for a function $\Fm$, we mean a point $\wm$ satisfying $\|\nabla_{\wm} \Fm\|_2\leq \epsilon$. Understanding such properties is useful, since in practice one usually computes an $\epsilon$-stationary point of $\Jhn(\wm)$. These results  guarantee the computed point is at most a $2\epsilon$-stationary point of $\Jm(\wm)$ and is thus close to the  optima.

\subsection{Uniform Convergence of Stationary Points}
Here we analyze the properties of the stationary points when optimizing empirical risk for deep learning algorithms. For explanation simplicity, we consider the non-degenerate stationary points which are geometrically isolated and thus are unique in local regions.
\begin{defn}{\textbf{(Non-degenerate stationary points)}}~\cite{gromoll1969differentiable}
If a stationary point $\wm$ is said to be a non-degenerate stationary point of $\Jm(\wm)$, then it satisfies
\begin{align*}
\inf_i \left|\lambda_i \left(\nabla^2 \Jm(\wm)\right)\right| \geq \zeta,
\end{align*}
where $\lambda_i \left(\nabla^2 \Jm(\wm)\right)$ denotes the $i$-th eigenvalue of the Hessian $\nabla^2 \Jm(\wm)$ and $\zeta$ is a positive constant.
\end{defn}

\begin{defn}{\textbf{(Index of non-degenerate stationary points)}}~\cite{dubrovin2012modern}
The index of a symmetric non-degenerate matrix is the number of its negative eigenvalues,
and the index of a non-degenerate stationary point $\wm$ of a smooth function $\Fm$ is simply the index of its Hessian
$\nabla^2 \Fm(\wm)$.
\end{defn}
Non-degenerate stationary points include local minimum/maximum and non-degenerate saddle points, while degenerate stationary points refer to degenerate saddle points. Suppose that $\Jm(\wm)$ has $m$ non-degenerate stationary points that are denoted as $\{\wmii{1},$ $\wmii{2},\cdots,\wmii{m}\}$. Now we are ready to present our results on the behavior of stationary points.
\begin{thm} \label{thm:localminimal}
Suppose Assumption~\ref{assumption12} on the input data $\xm$ holds and the activation functions in deep neural network are linear. Then if
$n\geq c_h\max(l^2r^2 r_x^4/(\dm_0d^2\varepsilon^2\tau^4 \log(l)),d\log(l)/\zeta^2)$
where $c_h$ is a  constant, for $k\in\{1,\cdots,m\}$, there exists a non-degenerate stationary point $\wmin{k}$ of $\Jhn(\wm)$ which corresponds to the non-degenerate stationary point $\wmii{k}$ of $\Jm(\wm)$ with probability at least $1-\varepsilon$. In addition,  $\wmin{k}$ and $\wmii{k}$ have the same non-degenerate index and they satisfy
\begin{equation*}
\left\|\wmin{k}\!-\!\wmii{k}\right\|_2\leq \frac{2c_g \tau \omega_g}{\zeta} \sqrt{l\max_j(\dm_j\dm_{j-1})} \sqrt{ \frac{d\log(nl)+\log(12/\varepsilon)}{n}},\quad  (k=1,\cdots,m)
\end{equation*}
with probability at least $1-\varepsilon$. Here the parameter $\omega_g$ and the constant $c_g$ are given  in Theorem~\ref{thm:uniformconvergence1}.
\end{thm}
Theorem~\ref{thm:localminimal} guarantees that the non-degenerate stationary points of the empirical risk $\Jhn(\wm)$ one-to-one correspond to the non-degenerate stationary points of the popular risk $\Jm(\wm)$. In addition,  the corresponding pairs have the same non-degenerate index, which means their corresponding Hessian matrices have the same properties, such as the same number of negative eigenvalues. Thus when $n$ is sufficiently large, the properties of stationary points of  $\Jhn{(\wm)}$ are similar to the points of the population risk $\Jm{(\wm)}$ in the sense that they have exactly matching local minima/maxima and saddle points. By comparing Theorems~\ref{thm:uniformconvergence1} and~\ref{thm:localminimal}, we find that the uniform convergence rate of non-degenerate stationary points has an extra factor $1/\zeta$. This is because bounding stationary points needs to access not only the gradient itself but also the Hessian matrix.

Kawaguchi~\cite{Kawaguchi2016} points out the existence of degenerate stationary points in deep linear neural networks. But since degenerate stationary points are not isolated, such as flat regions, it is hard to establish the unique correspondence in these points. Fortunately, by Theorem~\ref{thm:uniformconvergence1}, the gradient at these points of $\Jhn{(\wm)}$ and $\Jm{(\wm)}$ are close. This means that if a point is a degenerate stationary point of $\Jm{(\wm)}$, then its gradient in $\Jhn{(\wm)}$ is also close to 0, \emph{i.e.}, it is also a stationary point for the empirical risk $\Jhn{(\wm)}$. Notice, Kawaguchi~\cite{Kawaguchi2016} analyzed the loss surface of deep linear networks but they explored the existence of saddle points and the relations between global minimum and local minimum, which differs from our uniform convergence analysis work.

\section{Results for Deep Nonlinear Neural Networks}
In the above section, we present analysis on the empirical risk optimization landscape for deep linear neural network models. In this section, we proceed to analyze deep nonlinear neural networks, which adopts the sigmoid activation  function  and is more popular in practice.  Notice, our analysis techniques  are also applicable for other
third-order differentiable  functions, \textit{e.g.}, tanh function with different convergence rate. Here we assume the input data are \textit{i.i.d.}\ Gaussian variables.
\begin{assum}\label{assumption12sg}
The input datum  $\xm$ is a vector of \textit{i.i.d.} Gaussian variables from $\mathcal{N}(0, \tau^2)$.
\end{assum} Since for any input, the sigmoid function always maps it  to the range $[0,1]$. Thus, we do not require the input $\xm$  to have bounded magnitude. Such  assumptions are common. For instance, Tian~\cite{tian2017layered} and Soudry \textit{et al.}~\cite{soudry2017exponentially} all assume the entries in the input vector  are from Gaussian distribution. We also assume $\wm\in\Omega$ which is also used in~\cite{xu2012robustness} for deep learning robust analysis. Similar to the analysis of deep linear neural networks, here we also analyze the empirical risk and its gradient and stationary points for deep nonlinear neural network. 
\subsection{Uniform Convergence, Stability and Generalization of Empirical Risk}\label{sigmoid_generalization}
Here we first give the uniform convergence analysis of empirical risk and then analyze its stability and generalization.
\begin{thm} \label{thm:stability_sig}
Assume the input sample $\xm$ obeys Assumption~\ref{assumption12sg} and the activation functions in deep neural network are the sigmoid functions. Then if $n\geq 18r^2/(d\tau^2\varepsilon^2\log(l))$, there exists a universal constant  $c_y$  such that
\vspace{-0.2em}
\begin{equation}\label{unifosfdarm_loss}
\sup_{\wm\in\Omega} \left|  \Jhn(\wm)- \Jm(\wm)\right|\leq \epsilon_n\triangleq \tau\sqrt{\frac{9}{8}c_yc_d\left(1 + c_r (l-1)\right)} \sqrt{\frac{d\log(nl)+\log(4/\varepsilon)}{n}}
\vspace{-0.2em}
\end{equation}
holds with probability at least $1\!-\!\varepsilon$, where $c_{d}\!=\!\max_j \dm_j\, (0\leq j\leq l)$,  $c_r\!=\!\max\left(r^2/16,\left(r^2/16\right)^{l-1}\right)$.
\end{thm}
From Theorem~\ref{thm:stability_sig}, we  obtain that under the distribution $\D$, the empirical risk  of a deep nonlinear neural network converges at the rate of  $\mathcal{O}(1/\sqrt{n})$ (up to a $\log$ factor). Similar to the deep linear neural network, the layer number $l$, the network size $d$ and the magnitude  of weights are also important for  the convergence rate. Also, since there is a factor $\max_j \dm_j$ in the convergence rate, it is better to avoid choices of a layer of extremely large width, since  a network with extremely wide layers have  high sample complexity. Interestingly, when analyzing the representation ability of deep learning, Eldan \textit{et al.}~\cite{eldan2016power} also suggest non-extreme-wide layers, though the conclusions are derived from different perspectives. By comparing Theorems~\ref{thm:stability} and~\ref{thm:stability_sig}, one can observe that there is a factor $(1/16)^{l-1}$ in the convergence rate in Theorem~\ref{thm:stability_sig}. This is because the convergence rate accesses the Lipschitz constant and when we bound it, sigmoid activation function brings in the factor $1/16$ for each layer.

We  then establish the stability property  and the  generalization error of the empirical risk for nonlinear neural networks. By Theorem~\ref{thm:stability_sig}, we can obtain the following results.
\begin{cor}\label{stabiliaftysig}
Assume the input sample  $\xm$ obeys Assumption~\ref{assumption12sg} and the activation functions in deep neural network are sigmoid functions.  Then  with probability at least $1-\varepsilon$, we have :
\begin{equation*}
\Bigg|\EE_{\Ss\sim\D,\Am,(\xmi{1}',\cdots,\xmi{n}'\!)\sim\D}\frac{1}{n}\! \sum_{j=1}^{n}\!\!\left(\!f_j(\wms{j},\!\xmi{j}'\!)\!-\!\!f_j(\wmn,\xmi{j}'\!)\! \right)\!\!\Bigg|\!\!\leq\! \epsilon_n,\ \ \Bigg|\EE_{\Ss\sim\D,\Am}\!\left(\!\Jm(\wmn\!)\!- \!\Jhn(\wmn\!)\!\right)\!\!\Bigg|\!\! \leq\! \epsilon_n,
\end{equation*}
where $\epsilon_n$ is defined in Eqn.~\eqref{unifosfdarm_loss}. The notations $\xmi{i}'$ and $f_j(\wms{j},\!\xmi{j}')$ here are the same in Definition~\ref{defi1}.
\end{cor}
By Corollary~\ref{stabiliaftysig}, we know that both the stability convergence rate and the convergence rate of generalization error are $\mathcal{O}(1/\sqrt{n})$. This result accords with  Theorems~8 and 9 in~\cite{shalev2010learnability} which implies $\mathcal{O}(1/\sqrt{n})$ is the bottleneck of the stability and  generalization convergence rate for generic learning algorithms. From this result, we have that if $n$ is sufficiently large, empirical risk can be expected to be very stable.  This also dispels misgivings of the random selection of training samples in practice. 

\subsection{Uniform Convergence of Gradient and Stationary Points}
Here we  analyze convergence property of   gradients of empirical risk for deep nonlinear neural networks.
\begin{thm} \label{thm:uniformconvergence1_sig}
Assume the input sample  $\xm$ obeys Assumption~\ref{assumption12sg} and the activation functions in deep neural network are sigmoid functions. Then the sample gradient uniformly converges to the population gradient in Euclidean norm. Specifically, if $n\geq c_{y'}c_d lr^2/(d\tau^2\varepsilon^2\log(l))$ where $c_{y'}$ is a constant,
\begin{equation*}
\sup_{\wm\in\Omega} \!\left\| \nabla\Jhn(\wm)\!-\!\nabla\Jm(\wm)\right\|_2\!\!\leq\! \tau\!\sqrt{\!\frac{512}{729}c_yc_r(l\!+\!2)\left(dc_r\!+\!lc_d\!+ \!(l\!-\!1) lc_dc_r\!\right)} \sqrt{\!\frac{d\log(nl)\!+\!\log(4/\varepsilon)}{n}}
\end{equation*}
holds with probability at least $1-\varepsilon$,
where $c_y$, $c_{d}$ and $c_{r}$ are the same parameters in Theorem~\ref{thm:stability_sig}.
\end{thm}
 Theorem~\ref{thm:uniformconvergence1_sig} gives similar results as Theorem~\ref{thm:stability_sig}, including the $\mathcal{O}(1/\sqrt{n})$ uniform convergence rate and suggestion on non-extreme-wide layers. But by comparing Theorems~\ref{thm:uniformconvergence1_sig} and~\ref{thm:stability_sig}, one can observe that the depth $l$ and the magnitude of the weights (reflected by the factor $c_r$) have more significant effect on the convergence rate. This is because for highly nonconvex problems, \textit{e.g.}, the loss function of deep neural network, bounding higher order information is more technically challenging. By comparing Theorems~\ref{thm:uniformconvergence1_sig} and~\ref{thm:uniformconvergence1}, the convergence rate in Theorem~\ref{thm:uniformconvergence1_sig} depends on $l$, $r$ and $d$ more heavily since nonlinear networks is more complex and its convergence rate is thus more challenging to bound.

Now we analyze the non-degenerate stationary points of the empirical risk  for deep nonlinear neural networks. Here we also assume that population risk has $m$  non-degenerate stationary points denoted by $\{\wmii{1},\wmii{2},\cdots,\wmii{m}\}$.
\begin{thm} \label{thm:localminimal_sig}
Assume the input sample  $\xm$ obeys Assumption~\ref{assumption12sg} and the activation functions in deep neural network are sigmoid functions.  Then if $n\geq c_s\max(c_d lr^2/(d\tau^2\varepsilon^2\log(l)), d\log(l)/\zeta^2)$ where $c_s$ is a constant, for $k\in\{1,\cdots,m\}$, there exists a non-degenerate stationary point $\wmin{k}$ of $\Jhn(\wm)$ which corresponds to the non-degenerate stationary point $\wmii{k}$ of $\Jm(\wm)$ with probability at least $1-\varepsilon$. Moreover, $\wmin{k}$ and $\wmii{k}$ have the same non-degenerate index and they obey
\begin{align*}
\left\|\wmin{k}\!\!-\!\wmii{k}\right\|_2\!\leq\!\! \frac{2\tau}{\zeta} \sqrt{\!\frac{512}{729}c_yc_r(l\!+\!2)\left(dc_r\!+\!lc_d\!+\!(l\!-\!1) lc_dc_r\!\right)}\sqrt{\!\frac{d\log(nl)\! +\!\log(4/\varepsilon)}{n}},\, (k\!=\!1,\cdots\!,m)
\end{align*}
with probability at least $1-\varepsilon$,
where $c_y$, $c_{d}$ and $c_{r}$ are the same parameters in Theorem~\ref{thm:stability_sig}.
\end{thm}
According to Theorem~\ref{thm:localminimal_sig}, there is one-to-one correspondence relationship between the non-degenerate stationary points of $\Jhn(\wm)$ and $\Jm(\wm)$. Also the corresponding pairs have the same non-degenerate index, which implies they have exactly matching local minima/maxima and saddle points. When $n$ is sufficiently large, the non-degenerate stationary point $\wmin{k}$ in $\Jhn(\wm)$ is very close to its corresponding non-degenerate stationary point $\wmii{k}$ in $\Jm(\wm)$. As for the degenerate stationary points, Theorem~\ref{thm:uniformconvergence1_sig} guarantees the gradient at these points of $\Jm(\wm)$ and $\Jhn(\wm)$ are very close to each other.

\section{Proof Roadmap}
Here we briefly introduce our proof roadmap. Due to space limitation, all the proofs of Theorems~\ref{thm:stability}~$\sim$~\ref{thm:localminimal_sig} and Corollaries~\ref{stability and generalization} and~\ref{stabiliaftysig} as well as technical  lemmas are deferred to  the supplementary material.

The proofs of Theorems~\ref{thm:stability} and~\ref{thm:stability_sig} are similar but essentially differ in some techniques for bounding probability due to their different assumptions. 
For explanation simplicity, we define four events:
$\Em\!=\!\{\sup_{\wm\in\Omega}|\Jhn(\wm)\!-\!\Jm(\wm)|\!>\!t\}$, $\Em_1\!=\!\left\{\sup_{\wm\in\Omega} \lvs \frac{1}{n}\!\sum_{i=1}^n \!\! \left(f(\wm,\xmi{i})\!-\!f(\wm_{\kt},\xmi{i})\right)\rvs\!>\!t/3\right\}$, $\Em_2\!=\!\{\sup_{\wm_{\kt}^i\! \in\! \mathcal{N}_{j},\,i\in[l]}\lvs\frac{1}{n}\!\sum_{i=1}^n \!\! f(\wm_{\kt},\xmi{i})\!-\!\EE f(\wm_{\kt},\xm)\rvs\!>\!t/3\}$, and $\Em_3\!=\!\left\{\sup_{\wm\in\Omega} \lvs\EE f(\wm_{\kt},\xm)\right.\right.$ $\left.\left.-\EE f(\wm,\xm)\rvs\!>\!t/3\right\}$, where $\wm_{\kt}=[\wm_{\kt}^1;\wm_{\kt}^2;\cdots;\wm_{\kt}^l]$ is constructed by selecting $\wm_{\kt}^i\in\Rs{\dm_{i}\dm_{i-1}}$ from $\epsilon/l$-net $\mathcal{N}_{j}$ such that $\|\wm - \wm_{\kt}\|_2\leq \epsilon$. Notice, in Theorems~\ref{thm:stability} and~\ref{thm:stability_sig}, $t$ are respectively set to $\epsilon_l$ and $\epsilon_n$ (see Eqn.~\eqref{uniform_loss} and \eqref{unifosfdarm_loss}). Then we have $\Pro(\Em)\leq \Pro(\Em_1)+\Pro(\Em_2)+\Pro(\Em_3)$. So we only need to separately bound $\Pro(\Em_1)$, $\Pro(\Em_2)$ and $\Pro(\Em_3)$. For $\Pro(\Em_1)$ and $\Pro(\Em_3)$, we use the Lipschitz constant of the loss function and the properties of $\epsilon$-net to prove $\Pro(\Em_1)\leq \varepsilon/2$ and $\Pro(\Em_3)=0$, while bounding $\Pro(\Em_2)$ need more efforts. Here based on the assumptions, we prove that $\Pro(\Em_2)$ has sub-exponential tail associated to the sample number $n$ and the networks parameters, and it satisfies $\Pro(\Em_2)\leq \varepsilon/2$ with proper conditions. Finally, combining the bounds of the three terms, we obtain the desired results. Then we utilize the uniform convergence of $\Jhn(\wm)$ to prove the stability and generalization bounds of $\Jhn(\wm)$ (\textit{i.e.}~Corollaries~\ref{stability and generalization} and~\ref{stabiliaftysig}).

We adopt similar strategy to prove Theorems~\ref{thm:uniformconvergence1} and~\ref{thm:uniformconvergence1_sig}. Specifically, we divide the event $\sup_{\wm\in\Omega} \! \| \nabla\Jhn(\wm)\!-\!\nabla\Jm(\wm) \|_2>t$ into $\Em_1$, $\Em_2$ and $\Em_3$ which have the same forms as their counterparts in the proofs of Theorem~\ref{thm:stability} with replacing loss function by its gradient. But to prove $\Pro(\Em_1)\leq \varepsilon/2$ and $\Pro(\Em_3)=0$, we need to access the Lipschitz constant of the gradient which is more challenging to bound, especially for deep neural networks. The remaining is to prove $\Pro(\Em_2)$. We also prove that it has sub-exponential tail associated to the sample number $n$ and the networks parameters and it obeys $\Pro(\Em_2)\leq \varepsilon/2$ with proper conditions.

To prove Theorems~\ref{thm:localminimal} and~\ref{thm:localminimal_sig}, we first prove the uniform convergence of the empirical Hessian to its population Hessian. Then, we define such a set
$D = \{\wm\in \Omega:\; \lv \nabla \Jm(\wm) \rv_2 < \epsilon\ \text{and}$ $  \inf_i \left|\lambda_i\left(\nabla^2 \Jm(\wm)\right)\right| \geq \zeta\}$. In this way, $D$ can be decomposed into countably components,
with each component containing either exactly one non-degenerate stationary point, or no non-degenerate stationary point. For each component, uniform convergence of gradient and results in differential topology guarantee that if $\Jm(\wm)$ has no stationary points, then $\Jhn(\wm)$ also has no stationary points and vise versa. Similarly, for each component, uniform convergence of Hessian and results in differential topology guarantee that if $\Jm(\wm)$ has a unique non-degenerate stationary point, $\Jhn(\wm)$  has also a unique non-degenerate stationary point with the same index. After establishing exact correspondence between the non-degenerate stationary points of empirical risk and population risk, we use the uniform convergence of gradient and Hessian to bound the distance between the corresponding pairs.

\section{Conclusion}
In this work, we provided theoretical analysis on the landscape of empirical risk optimization for deep linear/nonlinear neural networks, including the uniform convergence, stability, and generalization of the empirical risk itself as well as the properties of its gradient and  stationary points. We proved their convergence rate to their population counterparts of $\mathcal{O}(1/\sqrt{n})$.  These results also reveal that the depth $l$, the network size $d$ and the width of a network are critical for the convergence rates. We also proved that the weight parameter magnitude also plays an important role in  the convergence rate. Indeed, small magnitude of the weights are suggested. All the results match the widely used network architectures in practice.


\bibliographystyle{unsrt}
\bibliography{referen}

\newpage
\appendix

\section{Structure of This Document}
This document gives some other necessary notations and preliminaries for our analysis in Sec.~\ref{notations}. Then we prove Theorems~\ref{thm:stability} $\sim$ \ref{thm:localminimal} and Corollary~\ref{stability and generalization} for deep linear neural networks in Sec.~\ref{Linearactivati}. Then we present the proofs of Theorems~\ref{thm:stability_sig} $\sim$ \ref{thm:localminimal_sig} and Corollary~\ref{stabiliaftysig} for deep nonlinear neural networks in Sec.~\ref{deepnonlinear}.

In both Sec.~\ref{Linearactivati} and~\ref{deepnonlinear}, we first present the technical lemmas for proving our final results and subsequently present the proofs of these lemmas. Then we utilize these technical lemmas to prove our desired results. Finally, we give the proofs of other auxiliary lemmas.

\section{Notations and Preliminary Tools}\label{notations}

Beyond the notations introduced in the manuscript, we need some other notations used in this document. Then we introduce several lemmas that will be used later.

\subsection{Notations}
Throughout this document, we use $\lr\cdot, \cdot\rl$ to denote the inner product. $\Am\otimes\Cm$ denotes the Kronecker product between $\Am$ and $\Cm$. Note that $\Am$ and $\Cm$ in $\Am\otimes\Cm$ can be matrices or vectors. For a matrix $\Am\in\Rss{n_1}{n_2}$, we use $\|\Am\|_F=\sqrt{\sum_{i,j}\Am_{ij}^2}$ to denote its Frobenius norm, where $\Am_{ij}$ is the $(i,j)$-th entry of $\Am$. We use $\|\Am\|_{\op}=\max_i |\lambda_i(\Am)|$ to denote the operation norm of a matrix $\Am\in\Rss{n_1}{n_1}$, where $\lambda_i(\Am)$ denotes the $i$-th eigenvalue of the matrix $\Am$. For a 3-way tensor $\A\in\Rsss{n_1}{n_2}{n_3}$, its operation norm is computed as
 $$\|\A\|_{\op} = \sup_{\|\lam\|_2\leq 1}\lr\lam^{\otimes^3}, \A\rl=\sum_{i,j,k}\A_{ijk}\lam_i\lam_j\lam_k,$$
 where $\A_{ijk}$ denotes the $(i,j,k)$-th entry of $\A$.
Also we denote the vectorization of $\Wmi{j}$ (the weight matrix of the $j$-th layer) as $$\wmi{j}=\vect{\Wmi{j}}\in\Rs{\dm_j\dm_{j-1}}.$$  We denote $\Im_k$ as the identity matrix of size $k\times k$.

\subsection{Technical Lemmas}
We first introduce Lemmas~\ref{lem:2norm} and~\ref{lem:opnorm} which are  respectively used for bounding the $\ell_2$-norm of a vector and the operation norm of a matrix. Then we introduce Lemmas~\ref{lemma:Stability2} and ~\ref{lemma:Decomposition} which discuss the topology of functions. In Lemma~\ref{stab}, we give the relationship between the stability and generalization of empirical risk.
\begin{lem}~\cite{VRMT}\label{lem:2norm}
For any vector $\xm\in\Rs{d}$, its $\ell_2$-norm can be bounded as
\begin{equation*}
\|\xm\|_2 \leq \frac{1}{1-\epsilon} \sup_{\lam \in \lam_\epsilon} \lr \lam,\xm\rl.
\end{equation*}
where $\lam_\epsilon=\{\lam_1,\dots,\lam_{\kt}\}$ be an $\epsilon$-covering net of $\Bi{d}{1}$.
\end{lem}

\begin{lem}~\cite{VRMT}\label{lem:opnorm}
For any symmetric matrix $\Xm\in\Rss{d}{d}$, its operator norm can be bounded  as
\begin{equation*}
\|\Xm\|_{\op} \leq \frac{1}{1-2\epsilon} \sup_{\lam \in \lam_\epsilon} \left|\lr \lam,\Xm\lam\rl\right|.
\end{equation*}
where $\lam_\epsilon=\{\lam_1,\dots,\lam_{\kt}\}$ be an $\epsilon$-covering net of $\Bi{d}{1}$.
\end{lem}

\begin{lem}\cite{mei2016landscape}\label{lemma:Stability2}
Let $D\subseteq \Rs{d}$ be a compact set with a $C^2$ boundary $\partial D$, and $f,g:A\to\mathbb{R}$ be $C^2$ functions
defined on an open set $A$, with $D\subseteq A$. Assume that for all $\wm \in \partial D$ and all $t\in [0,1]$,
$t\nabla f(\wm)+(1-t)\nabla g(\wm)\neq \bm{0}$. Finally, assume that the Hessian $\nabla^2f(\wm)$ is non-degenerate and
has index equal to $r$ for all $\wm\in D$. Then the following properties hold:
\begin{myitemize}
\item[(1)] If $g$ has no critical point in $D$, then $f$ has no critical point in $D$.
\item[(2)] If $g$ has a  unique critical point $\wm$ in $D$ that is non-degenerate with an index of $r$, then $f$ also has a unique critical point $\wm'$ in $D$ with the index equal to $r$.
\end{myitemize}
\end{lem}

\begin{lem}\cite{mei2016landscape}\label{lemma:Decomposition}
Suppose that $F(\wm): \Theta \to \mathbb{R}$ is a $C^2$ function where $\wm\in\Theta$. Assume that $\{\wmii{1},\dots,$ $ \wmii{m}\}$ is its non-degenerate critical points and let
$D = \{\wm\in \Theta:\; \lv \nabla F(\wm) \rv_2 < \epsilon\ \text{and}$ $  \inf_i \left|\lambda_i\left(\nabla^2 F(\wm)\right)\right| \geq \zeta\}$. Then $D$ can be decomposed into (at most) countably  components,
with each component containing either exactly one critical point, or no critical point. Concretely, there exist disjoint open sets $\{D_k\}_{k\in \mathbb{N}}$, with $D_k$ possibly empty for $k\geq m+1$, such that
\begin{align*}
D = \cup_{k=1}^{\infty} D_k\, .
\end{align*}
Furthermore, $\wmii{k}\in D_k$ for $1\leq k\leq m$ and each $D_i$, $k\geq m+1$ contains no stationary points.
\end{lem}

\begin{lem}\label{stab}~\cite{shalev2014understanding,gonen2017fast}
Assume that $\D$ is a sample distribution and randomized algorithm $\Am$ is employed for optimization. Suppose that $(\xmi{1}',\cdots,\xmi{n}')\sim\D$ and $\wmn=\argmin_{\wm} \Jhn(\wm)$.
For every $j\in\{1,\cdots,n\}$, suppose $\wms{j}=\argmin_{\wm} \frac{1}{n-1} \sum_{i\neq j} f_i(\wm,\xmi{i})$. For arbitrary distribution $\D$, we have
\begin{equation*}
\left|\EE_{\Ss\sim\D,\,\Am\,(\xmi{1}',\cdots,\xmi{n}')\sim\D}\frac{1}{n}\! \sum_{j=1}^{n}\!\!\left(\!f_j(\wms{j},\!\xmi{j}')\!-\!f_j(\wmn,\xmi{j}') \right)\!\right| =\Bigg|\EE_{\Ss\sim\D,\,\Am}\left(\Jm(\wmn)\!-\!\Jhn(\wmn)\right)\!\Bigg|.
\end{equation*}
\end{lem}

\section{Proofs for Deep Linear Neural Networks}\label{Linearactivati}
In this section, we first present the technical lemmas in Sec.~\ref{keylemma1} and then we give the proofs of these lemmas in Sec.~\ref{keyproof}.  Next, we utilize these lemmas to prove the results in~Theorems~\ref{thm:stability} $\sim$
\ref{thm:localminimal} and Corollary~\ref{stability and generalization} in Sec.~\ref{mainproof}. Finally, we give the proofs of other lemmas in Sec.~\ref{otherlemmas}.
\subsection{Technical Lemmas}\label{keylemma1}
Here we present the technical lemmas for proving our desired results. For brevity, we also define $\Bmi{j:s}$ as follows:
\begin{equation}\label{grdgsda345dient1}
\begin{split}
&\Bmi{s:t}\triangleq\Wmi{s}\Wmi{s-1}\cdots\Wmi{t} \in\Rss{\dm_s}{\dm_{t-1}},\ (s\geq t);\quad  \Bmi{s:t}\triangleq\Im,\ (s<t).
\end{split}
\end{equation}

\begin{lem}\label{uh2397}
Assume that the activation functions in the deep neural network $f(\wm,\xm)$ are linear functions. Then the gradient of $f(\wm,\xm)$ with respect to $\wmi{j}$ can be written as
\begin{equation*}\label{gradient1}
\nabla_{\wmi{j}}f(\wm,\xm)=\left((\Bmi{j-1:1}\xm)\otimes\Bmi{l:j+1}^T\right) \emm,\ (j=1,\cdots,l),
\end{equation*}
where $\otimes$ denotes the Kronecke product. Then we can compute the Hessian matrix as follows:
\begin{equation*}\label{gradient1}
\nabla^2f(\wm,\xm)=
\begin{bmatrix}
\nabla_{\wmi{1}}\left(\nabla_{\wmi{1}}f(\wm,\xm)\right)  &\cdots&\nabla_{\wmi{1}}\left(\nabla_{\wmi{l}}f(\wm,\xm)\right)  \\
\nabla_{\wmi{2}}\left(\nabla_{\wmi{1}}f(\wm,\xm)\right)  &\cdots&\nabla_{\wmi{2}}\left(\nabla_{\wmi{l}}f(\wm,\xm)\right)  \\
\vdots  &\ddots&\vdots   \\
\nabla_{\wmi{l}}\left(\nabla_{\wmi{1}}f(\wm,\xm)\right)  &\cdots&\nabla_{\wmi{l}}\left(\nabla_{\wmi{l}}f(\wm,\xm)\right)  \\
\end{bmatrix},
\end{equation*}
where $\Qm_{st}\triangleq\nabla_{\wmi{s}}\left(\nabla_{\wmi{t}}f(\wm,\xm)\right)$ is defined as
\begin{equation*}\label{gradient1}
\Qm_{st}\!=\!
\begin{cases}
\left(\Bmi{t-1:s+1}^T\right)\!\otimes\!\left(\Bmi{s-1:1}\xm\emm^T \Bmi{l:t+1}^T\right)\!+\!\left(\Bmi{s-1:1}\xm \xm^T \Bmi{t-1:1}^T\right)\!\otimes\!\left(\Bmi{l:s+1}^T \Bmi{l:t+1}\right),&\!\!\! \mbox{if } s\!<\!t,\\
\left(\Bmi{s-1:1}\xm\xm^T\Bmi{s-1:1}\right)\otimes\left({\Bmi{l:s+1}}^T\Bmi{l:s+1}\right), & \!\!\!\mbox{if } s\!=\!t, \\
\left(\Bmi{l:s+1}^T\emm\xm^T\Bmi{t-1:1}^T\right)\!\otimes\! \Bmi{s-1:t+1}\!+\!\left(\Bmi{s-1:1}\xm \xm^T \Bmi{t-1:1}^T\right)\!\otimes\!\left(\Bmi{l:s+1}^T \Bmi{l:t+1}\right),&\!\!\!\mbox{if } s\!>\!t. \\
\end{cases}
\end{equation*}
\end{lem}

\begin{lem}\label{thmsafobjeasdfctive}
Suppose Assumption~\ref{assumption12} on the input data $\xm$ holds and the activation functions in deep neural network are linear functions. Then for any $t>0$, the objective $f(\wm,\xm)$ obeys
\begin{align*}
\Pro\!\left(\!\frac{1}{n}\!\sum_{i=1}^n\!\left( f(\wm,\xmi{i})\!-\!\EE(f(\wm,\xmi{i}))\right)\!>\!t\!\right)\leq 2\exp\!\left(\!-c_{f'}n\min\!\left(\!\frac{t^2}{\omega_{f}^2 \max\left(\dm_l\omega_{f}^2\tau^4,\tau^2\right)}, \frac{t}{\omega_{f}^2\tau^2}\!\right)\!\right),
\end{align*}
where $c_{f'}$ is a positive constant and $\omega_{f}=r^{l}$.
\end{lem}

\begin{lem}\label{thm:gradient12}
Suppose Assumption~\ref{assumption12} on the input data $\xm$ holds and the activation functions in deep neural network are linear functions. Then for any $t>0$ and arbitrary unit vector $\lam\in\SS^{d-1}$, the gradient $\nabla f(\wm,\xm)$ obeys
\begin{equation*}\label{gradient14211}
\begin{split}
&\Pro\!\left(\!\frac{1}{n}\sum_{i=1}^n\left( \lr\lam,\nabla_{\wm}f(\wm,\xmi{i})-\!\EE\nabla_{\wm}f(\wm,\xmi{i})\rl\right) \!>\!t\right)\\
&\qquad \quad\qquad\qquad\qquad\leq 3\exp\left(-c_{g'}n\min\left(\frac{t^2}{ l \max\left(\omega_g\tau^2,\omega_g\tau^4,\omega_{g'}\tau^2\right)},
\frac{t}{\sqrt{l\omega_g}\max\left(\tau,\tau^2\right)}\right)\right),
\end{split}
\end{equation*}
where $c_{g'}$ is a constant; $\omega_g=\dm_0 r^{2(2l-1)}\max_j(\dm_j\dm_{j-1})$ and $\omega_{g'}= r^{2(l-1)}\max_j(\dm_j\dm_{j-1})$.
\end{lem}

\begin{lem}\label{thm: hessian}
Suppose Assumption~\ref{assumption12} on the input data $\xm$ holds and the activation functions in deep neural network are  linear functions. Then for any $t>0$ and arbitrary unit vector $\lam\in\SS^{d-1}$, the Hessian $\nabla^2 f(\wm,\xm)$ obeys
\begin{equation*}\label{gradient14211}
\begin{split}
&\Pro\left(\frac{1}{n}\sum_{i=1}^n \left(\lr \lam,(\nabla_{\wm}^2f(\wm,\xmi{i})-\EE \nabla_{\wm}^2f(\wm,\xmi{i}))\lam\rl\right)>t\right)\\
&  \qquad\quad\qquad \qquad\qquad\leq 5\exp\left(-c_{h'}n\min\left(\frac{t^2}{ \tau^2l^2\max\left(\omega_g ,\omega_g \tau^2, \omega_h \right)},
\frac{t}{\sqrt{\omega_g} l\max\left(\tau,\tau^2\right)}\right)\right),
\end{split}
\end{equation*}
where $\omega_g=\left(\max_j(\dm_j\dm_{j-1})\right)^2 r^{4(l-1)}$ and $\omega_h=\left(\max_j(\dm_j\dm_{j-1})\right)^2 r^{2(l-2)}$.
\end{lem}


\begin{lem}\label{thmgradisagf5675ent12}
 Suppose the activation functions in deep neural network are linear functions. Then for any $\wm \in\Bi{d}{r}$ and $\xm\in\Bi{\dm_0}{r_x}$, we have
\begin{align*}
\left\|\nabla_{\wm} f(\wm,\x)\right\|_2 \leq \sqrt{\alpha_g},\quad \text{where}\quad \alpha_g= c_tlr_x^4 r^{4l-2}.
\end{align*}
in which $c_t$ is a constant.
Further, for any $\wm \in\Bi{d}{r}$ and $\xm\in\Bi{\dm_0}{r_x}$, we also have
\begin{align*}
\left\|\nabla^2 f(\wm,\x)\right\|_{\op}\leq \left\|\nabla^2 f(\wm,\x)\right\|_{F}\leq l\sqrt{\alpha_l},\quad \text{where}\quad \alpha_l\triangleq c_{t'}r_x^4 r^{4l-2}.
\end{align*}
in which $c_{t'}$ is a constant.
With the same condition, we can bound the operation norm of $\nabla^3 f(\wm,\x)$. That is, there exists a universal constant $\alpha_p$ such that $\left\|\nabla^3 f(\wm,\x)\right\|_{\op} \leq \alpha_p$.
\end{lem}

\begin{lem} \label{thm:uniformconvergence321}
Suppose Assumption~\ref{assumption12} on the input data $\xm$ holds and the activation functions in deep neural network are linear functions.  Then there exist two universal constant $c_g$ and $c_h$ such that the sample Hessian converges uniformly to the population Hessian in operator norm. Specifically, there exit two universal constants $c_{h''}$ and $c_h$  such that if $n\geq c_{h''}\max(\frac{\alpha_p^2r^2}{\tau^2 l^2\omega_h^2\varepsilon^2(\max_j(\dm_j\dm_{j-1}))^2d\log(l)}, d\log(l))$, then
\begin{equation*}
\sup_{\wm\in\Omega}\left\| \nabla^2\Jhn(\wm)\!-\!\nabla^2\Jm(\wm)\right\|_{\op}\!\leq\! c_h \tau l \omega_h\max_j(\dm_j\dm_{j-1})  \sqrt{\!\frac{d\log(nl)\!+\!\log(20/\varepsilon)}{n}}
\end{equation*}
holds with probability at least $1-\varepsilon$, where $\omega_h=\max\!\left(\tau r^{2(l-1)},r^{2(l-1)},r^{l-2}\right)$.
\end{lem}

\subsection{Proofs of Technical Lemmas}\label{keyproof}
To prove the above lemmas, we first introduce some useful results.
\begin{lem}~\cite{rudelson2013hanson}\label{sumsafsad678exp}
Assume that $\xm=(\xm_1;\xm_2;\cdots;\xm_k)\in\Rs{k}$ is a random vector with independent components $x_i$ which have zero mean and are independent $\tau_i^2$-sub-Gaussian variables. Here $\max_i\tau_i^2\leq \tau^2$. Let $\Am$ be an $k\times k$ matrix.  Then we have
\begin{align*}
\EE  \exp \left(\lambda\left(\sum_{i,j:i\neq j}\Am_{ij} x_ix_j-\EE(\sum_{i,j:i\neq j}\Am_{ij} x_ix_j)\right)\right)
\leq\exp\left(2\tau^2\lambda^2\|\Am\|_F^2\right), \ |\lambda| \leq 1/ (2\tau\|\Am\|_2).
\end{align*}
\end{lem}

\begin{lem}\label{sumsafa34sad678safexp}
Assume that $\xm=(\xm_1;\xm_2;\cdots;\xm_k)\in\Rs{k}$ is a random vector with independent components $x_i$ which have zero mean and are independent $\tau_i^2$-sub-Gaussian variables. Here $\max_i\tau_i^2\leq \tau^2$. Let $\am$ be an $n$-dimensional vector.  Then we have
\begin{align*}
\EE  \exp \left(\lambda\left(\sum_{i=1}^k\am_{i} \xm_i^2-\EE\left(\sum_{i=1}^k\am_{i} \xm_i^2\right)\right)\right)
\leq \EE \exp\left(128\lambda^2\tau^4\left(\sum_{i=1}^{k}\am_{i}^2 \right)\right),\quad |\lambda|\leq \frac{1}{\tau^2\max_i \am_i}.
\end{align*}
\end{lem}

\begin{lem}\label{asfasasfshdfhfdet}
For $\Bm_{j:t}$ defined in Eqn.~\eqref{grdgsda345dient1}, we have the following properties:
\begin{equation*}\label{gradient1}
\begin{split}
\|\Bmi{s:t}\|_{\op}\leq\left\|\Bm_{s:t}\right\|_F\leq \omega_{r}\quad \text{and} \quad \|\Bmi{l:1}\|_{\op}\leq \left\|\Bm_{l:1}\right\|_F\leq \omega_{f},
\end{split}
\end{equation*}
where $\omega_{r}= r^{s-t+1}\leq \max\left(r,r^{l}\right)$ and $\omega_{f}=r^{l}$.
\end{lem}
Lemma~\ref{sumsafa34sad678safexp} is useful for bounding probability. The two inequalities in Lemma~\ref{asfasasfshdfhfdet} can be obtained by using $\|\wmi{j}\|_2\leq r\, (\forall j=1,\cdots,l)$. We defer the proofs of Lemmas~\ref{sumsafa34sad678safexp} and~\ref{asfasasfshdfhfdet} to Sec.~\ref{12rs43}.

\subsubsection{Proof of Lemma~\ref{uh2397}}
\begin{proof}[\hypertarget{lemmacompu}{Proof}]
When the activation functions are linear functions, we can easily compute the gradient of $f(\wm,\xm)$ with respect to $\wmi{j}$:
\begin{equation*}\label{gradient1}
\nabla_{\wmi{j}}f(\wm,\xm)=\left((\Bmi{j-1:1}\xm)\otimes\Bmi{l:j+1}^T\right) \emm,\ (j=1,\cdots,l),
\end{equation*}
where $\otimes$ denotes the Kronecker product. Now we consider the computation of the Hessian matrix. For brevity, let $\Qm_s=\left((\Bmi{s-1:1}\xm)\otimes\Bmi{l:s+1}^T\right)$. Then we can compute $\nabla_{\wmi{s}}^2 f(\wm,\xm)$ as follows:
\begin{equation*}\label{gradient1}
\begin{split}
\nabla_{\wmi{s}}^2f(\wm,\xm)=&  \frac{\partial^2f(\wm,\xm)}{\partial\wmi{s}^T\partial\wmi{s}} =\frac{\partial^2f(\wm,\xm)}{\partial\wmi{s}^T\partial\wmi{s}}
= \frac{\partial (\Qm_s\emm)}{\partial\wmi{s}^T}
= \frac{\partial \vect{\Qm_s\emm}}{\partial\wmi{s}^T} \\
= &\frac{\partial \vect{\Qm_s\Bmi{l:s+1}\Wmi{t}\Bmi{s-1:1}\xm}}{\partial\wmi{s}^T}\\
= &\frac{\partial \left((\Bmi{s-1:1}\xm)^T\otimes (\Qm_s\Bmi{l:s+1})\right) \vect{\Wmi{s}}}{\partial\wmi{s}^T}\\
= &(\Bmi{s-1:1}\xm)^T\otimes \left(\left((\Bmi{s-1:1}\xm)\otimes\Bmi{l:s+1}^T\right) \Bmi{l:s+1}\right)\\
\overset{\text{\ding{172}}}{=}&(\Bmi{s-1:1}\xm)^T\otimes \left((\Bmi{s-1:1}\xm)\otimes\left(\Bmi{l:s+1}^T \Bmi{l:s+1}\right)\right)\\
\overset{\text{\ding{173}}}{=}&\left((\Bmi{s-1:1}\xm)^T\otimes (\Bmi{s-1:1}\xm)\right)\otimes\left(\Bmi{l:s+1}^T \Bmi{l:s+1}\right)\\
\overset{\text{\ding{174}}}{=}&\left((\Bmi{s-1:1}\xm) (\Bmi{s-1:1}\xm)^T\right)\otimes\left(\Bmi{l:s+1}^T \Bmi{l:s+1}\right),
\end{split}
\end{equation*}
where $\text{\ding{172}}$ holds since $\Bmi{j-1:1}\xm$ is a vector and for any vector $\xm$, we have $(\xm\otimes \Am)\Bm=\xm\otimes(\Am\Bm)$. \ding{173} holds because for any four matrices $\Zm_1\sim\Zm_3$ of proper sizes, we have $(\Zm_1\otimes\Zm_2) \otimes\Zm_3= \Zm_1\otimes(\Zm_2 \otimes\Zm_3)$. \ding{174} holds because for any two matrices $\zm_1, \zm_2$ of proper sizes, we have $\zm_1\zm_2^T=\zm_1\otimes\zm_2^T=\zm_2^T\otimes\zm_1$.

Then, we consider the case $s>t$:
\begin{equation*}\label{gradient1}
\begin{split}
\nabla_{\wmi{t}}\left(\nabla_{\wmi{s}}f(\wm,\xm)\right)=&  \frac{\partial^2f(\wm,\xm)}{\partial\wmi{t}^T\partial\wmi{s}} =\frac{\partial^2f(\wm,\xm)}{\partial\wmi{t}^T\partial\wmi{s}}
= \frac{\partial (\Qm_s\emm)}{\partial\wmi{t}^T}
= \frac{\partial \vect{\Qm_s\emm}}{\partial\wmi{t}^T} \\
= &\frac{\partial \vect{\Qm_s\Bmi{l:t+1}\Wmi{t}\Bmi{t-1:1}\xm}}{\partial\wmi{t}^T}+ \frac{\partial \vect{\left((\Bmi{s-1:1}\xm)\otimes\Bmi{l:s+1}^T\right)\emm}}{ \partial\wmi{t}^T}.
\end{split}
\end{equation*}
Notice, here we just think that $\Qm_s$ in the $\frac{\partial \vect{\Qm_s\Bmi{l:t+1}\Wmi{t}\Bmi{t-1:1}\xm}}{\partial\wmi{t}^T}$ is a constant matrix and is not related to $\Wmi{t}$. Similarly, we also take $\emm$ in $\frac{\partial \vect{\left((\Bmi{s-1:1}\xm)\otimes\Bmi{l:s+1}^T\right)\emm}}{ \partial\wmi{t}^T}$ as a constant vector.
Since we have
\begin{equation*}\label{gradient1}
\begin{split}
\frac{\partial \vect{\Qm_s\Bmi{l:t+1}\Wmi{t}\Bmi{t-1:1}\xm}}{\partial\wmi{t}^T}= \left(\Bmi{s-1:1}\xm \xm^T \Bmi{t-1:1}^T\right)\otimes\left(\Bmi{l:s+1}^T \Bmi{l:t+1}\right),
\end{split}
\end{equation*}
we only need to consider
\begin{equation*}\label{gradient1}
\begin{split}
\frac{\partial \vect{\left((\Bmi{s-1:1}\xm)\otimes\Bmi{l:s+1}^T\right)\emm}}{\partial\wmi{t}^T}
= & \frac{\partial \vect{(\Bmi{s-1:1}\xm)\otimes\left(\Bmi{l:s+1}^T\emm\right)}}{\partial\wmi{t}^T}\\
\\
= & \frac{\partial \vect{(\Bmi{s-1:1}\xm) \left(\Bmi{l:s+1}^T\emm\right)^T}}{\partial\wmi{t}^T}\\
\\
= & \frac{\partial \vect{\Bmi{s-1:t+1}\Wmi{t}\left(\Bmi{t-1:1}\xm  \emm^T\Bmi{l:s+1}\right)}}{\partial\wm_{t}^T}\\
= & \frac{\partial \left(\Bmi{t-1:1}\xm  \emm^T\Bmi{l:s+1}\right)^T\otimes \Bmi{s-1:t+1} \vect{\Wmi{t}}}{\partial\wm_{t}^T}\\
= & \left(\Bmi{t-1:1}\xm  \emm^T\Bmi{l:s+1}\right)^T\otimes \Bmi{s-1:t+1}.
\end{split}
\end{equation*}
Therefore, for $s>t$, by combining the above two terms, we can obtain
\begin{equation*}\label{gradient1}
\begin{split}
\nabla_{\wmi{t}}\left(\nabla_{\wmi{s}}f(\wm,\xm)\right)= \left(\Bmi{l:s+1}^T\emm\xm^T\Bmi{t-1:1}^T\right)\!\otimes\! \Bmi{s-1:t+1}\!+\!\left(\Bmi{s-1:1}\xm \xm^T \Bmi{t-1:1}^T\right)\!\otimes\!\left(\Bmi{l:s+1}^T \Bmi{l:t+1}\right).
\end{split}
\end{equation*}
Then, by similar method, we can compute the Hessian for the case $s<t$ as follows:
\begin{equation*}\label{gradient1}
\begin{split}
\nabla_{\wmi{t}}\!\left(\nabla_{\wmi{s}}\!f(\wm,\xm)\right)\!=\! \left(\Bmi{t-1:s+1}^T\right)\!\otimes\!\left(\Bmi{s-1:1}\xm\emm^T \Bmi{l:t+1}^T\right)\!+\!\left(\Bmi{s-1:1}\xm \xm^T \Bmi{t-1:1}^T\right)\!\otimes\!\left(\Bmi{l:s+1}^T \Bmi{l:t+1}\right).
\end{split}
\end{equation*}
The proof is completed.
\end{proof}

\subsubsection{Proof of Lemma~\ref{thmsafobjeasdfctive}}
\begin{proof}[\hypertarget{lemmaloss}{Proof}]
We first prove that $\vmi{l}$, which is defined in Eqn.~\eqref{vl_definition}, is sub-Gaussian.
\begin{equation}\label{vl_definition}
\vmi{l}=\Wm^{(l)}\cdots\Wm^{(1)}\xm=\Bm_{l:1}\xm.
\end{equation}
Then by the convexity in $\lambda$ of $\exp(\lambda t)$ and Lemma~\ref{asfasasfshdfhfdet}, we can obtain
\begin{equation}\label{EFAFQ}
\begin{split}
\EE \left(\exp \left(\lr\lam,\vmi{l}-\EE(\vmi{l})\rl\right)\right)=&\EE \left(\exp \left(\lr\lam,\Bm_{l:1}\xm-\EE \Bm_{l:1}\xm\rl\right)\right)\\
\leq &\EE \left(\exp \left(\lr\Bmi{l:1}^T\lam,\xm\rl\right)\right)\\
\leq  &\exp \left(\frac{\|\Bmi{l:1}^T\lam\|_2^2\tau^2}{2}\right)\\
\overset{\text{\ding{172}}}{\leq} &\exp \left(\frac{\omega_{f}^2\tau^2\|\lam\|_2^2}{2}\right),
\end{split}
\end{equation}
where \ding{172} uses the conclusion that $\|\Bmi{l:1}\|_{\op}\leq \|\Bmi{l:1}\|_F\leq \omega_f$ in Lemma~\ref{asfasasfshdfhfdet}.
This means that $\vmi{l}$ is centered and is $\omega_{f}^2\tau^2$-sub-Gaussian. Accordingly, we can obtain that the $k$-th entry of $\vmi{l}$ is also $z_k\tau^2$-sub-Gaussian, where $z_k$ is a universal positive constant. Note that $\max_k z_k\leq\omega_{f}^2$.
Let $\vmi{l}_i$ denotes the output of the $i$-th sample $\xmi{i}$. By Lemma~\ref{sumsafa34sad678safexp}, we have that for $s>0$,
\begin{align*}
\Pro\left(\frac{1}{n}\sum_{i=1}^n \left(\|\vmi{l}_i\|_{2}^2-\EE\|\vmi{l}_i\|_{2}^2\right) >\frac{t}{2}\right)&=\Pro\left(s\sum_{i=1}^n \left(\|\vmi{l}_i\|_{2}^2-\EE\|\vmi{l}_i\|_{2}^2\right) >\frac{nst}{2}\right)\\
&\led{172}\exp\left(-\frac{snt}{2}\right) \EE\left(s\sum_{i=1}^n \left(\|\vmi{l}\|_{2}^2-\EE\|\vmi{l}\|_{2}^2\right)\right)\\
&\led{173}\exp\left(-\frac{snt}{2}\right) \prod_{i=1}^n\EE\left(s \left(\|\vmi{l}\|_{2}^2-\EE\|\vmi{l}\|_{2}^2\right)\right)\\
&\led{174}\exp\left(-\frac{snt}{2}\right) \prod_{i=1}^n\exp\left(128\dm_ls^2\omega_{f}^4\tau^4 \right)\quad |s|\leq \frac{1}{\omega_{f}^2\tau^2}\\
&\led{175} \exp\left(-c'n\min\left(\frac{t^2}{\dm_l\omega_{f}^4 \tau^4}, \frac{t}{\omega_{f}^2\tau^2}\right)\right).
\end{align*}
Note that \ding{172} holds because of Chebyshev's inequality. \ding{173} holds since $\xmi{i}$ are independent. \ding{174} is established by applying  Lemma~\ref{sumsafa34sad678safexp}. We have \ding{175} by optimizing $s$. 
Since $\vmi{l}$ is sub-Gaussian, we have
\begin{align*}
\Pro\left(\frac{1}{n}\sum_{i=1}^n\left(\ym^T\vmi{l}_i-\EE \ym^T\vmi{l}_i\right)>\frac{t}{2}\right) \leq &\Pro\left(s\sum_{i=1}^n\left(\ym^T\vmi{l}_i-\EE \ym^T\vmi{l}_i\right)>\frac{nst}{2}\right)\\
\leq &\exp\left(-\frac{nst}{2}\right) \EE\exp\left(s\sum_{i=1}^n\left(\ym^T\vmi{l}_i-\EE \ym^T\vmi{l}_i\right) \right)\\
\leq &\exp\left(-\frac{nst}{2}\right) \prod_{i=1}^{n}\EE\exp\left(s\left(\ym^T\vmi{l}_i-\EE \ym^T\vmi{l}_i\right) \right)\\
\led{172} &\exp\left(-\frac{nst}{2}\right) \prod_{i=1}^{n}\exp\left(\frac{\omega_{f}^2\tau^2s^2\|\ym\|_2^2}{2}\right)\\
\overset{\text{\ding{173}}}{\leq}& \exp\left(-\frac{nt^2}{8 \omega_{f}^2\tau^2\|\ym\|_2^2}\right),
\end{align*}
where \ding{172} holds because of Eqn.~\eqref{EFAFQ} and we have \ding{173} since we optimize $s$.

Since the loss function $f(\wm,\xm)$ is defined as $f(\wm,\xm)=\|\vmi{l}-\ym\|_{2}^2$, we have
\begin{align*}
f(\wm,\xm)-\EE(f(\wm,\xm))\!=\!\|\vmi{l}-\ym\|_{2}^2\!-\! \EE(\|\vmi{l}\!-\!\ym\|_{2}^2)
\!=\!\left(\|\vmi{l}\|_{2}^2\!-\!\EE\|\vmi{l}\|_{2}^2\right) \!+\!\left(\ym^T\vmi{l}\!-\!\EE \ym^T\vmi{l}\right).
\end{align*}
Therefore, we have
\begin{align*}
&\Pro\left(\frac{1}{n}\sum_{i=1}^n\left( f(\wm,\xmi{i})-\EE(f(\wm,\xmi{i}))\right)>t\right)\\
\leq&  \Pro\left(\frac{1}{n}\sum_{i=1}^n\left(\|\vmi{l}_i\|_{2}^2 -\EE\|\vmi{l}_i\|_{2}^2\right)>\frac{t}{2}\right) +\Pro\left(\frac{1}{n}\sum_{i=1}^n\left(\ym^T\vmi{l}_i-\EE \ym^T\vmi{l}_i\right)>\frac{t}{2}\right) \\
\leq & 2\exp\left(-c_{f'}n\min\left(\frac{t^2}{\dm_l\omega_{f}^4 \tau^4}, \frac{t^2}{\omega_{f}^2\tau^2}, \frac{t}{\omega_{f}^2\tau^2}\right)\right).
\end{align*}
where $c_{f'}$ is a constant. Note that $\|\ym\|_2^2$ is the label of $\xm$, then it can also be bounded. The proof is completed.
\end{proof}

\subsubsection{Proof of Lemma~\ref{thm:gradient12}}
\begin{proof}[\hypertarget{lemmagradient}{Proof}]
For brevity, let $\Qm_{j}$ denote $\nabla_{\wmi{j}}f(\wm,\xm)$. Then, by Lemma~\ref{uh2397} we have
\begin{equation}\label{gradgdsgsient211}
\nabla_{\wmi{j}}f(\wm)=\left((\Bmi{j-1:1}\xm)\otimes\Bmi{l:j+1}^T\right) \emm\overset{\text{\ding{172}}}{=}(\Bmi{j-1:1}\xm)\otimes(\Bmi{l:j+1}^T\emm) \overset{\text{\ding{173}}}{=} \left(\Bmi{j-1:1}\otimes\Bmi{l:j+1}^T\right) \left(\xm\otimes\emm\right),
\end{equation}
where $\text{\ding{172}}$ holds since $\Bmi{j-1:1}\xm$ is a vector, and \ding{173} holds because for any four matrices $\Zm_1\sim\Zm_4$ of proper sizes, we have $(\Zm_1\Zm_3)\otimes (\Zm_2\Zm_4)= (\Zm_1\otimes \Zm_2)(\Zm_3\otimes \Zm_4)$. Note that $\emm=\vmi{l}-\ym=\Bmi{l:1}\xm-\ym$. Then we know that the $i$-th entry $\Qm_{j}^{i}$ has the form $\Qm_{j}^{i}=\sum_{p,q} z_{pq}^{ij}\xm_p \xm_q+\sum_{p} y_{p}^{ij} \xm_p+r^{ij}$ (Step 1 blow will give the detailed analysis) where $\xm_p$ denotes the $p$-th entry in $\xm$. Note that $z_{pq}^{ij}, y_{p}^{ij}$ and $r^{ij}$ are constants and independent on $\xm$.

We divide $\lam\in\Rs{\sum_{j=1}^{l}\dm_j\dm_{j-1}}$ into $\lam=(\lam_1;\cdots;\lam_l)$ where $\lam_j\in\Rs{\dm_j\dm_{j-1}}$. Let $\lam_j^i$ denote the $i$-th entry in $\lam_j$. Accordingly, we have
\begin{equation*}\label{gradient14211}
\begin{split}
\Em\triangleq\lr\lam,\nabla_{\wm}f(\wm,\xm)-\EE\nabla_{\wm}f(\wm,\xm)\rl =\sum_{j=1}^{l} \left\langle\lam_j,\Qm_j-\EE \Qm_j   \right\rangle=\Em_1+\Em_2+\Em_3,
\end{split}
\end{equation*}
where $\Em_1, \Em_2$, and $\Em_3$ are defined as
\begin{align}
&\Em_{1}\!\!=\!\!\!\sum_{p,q:p\neq q}\!\!\left(
\sum_{j=1}^{l} \sum_{i=1}^{\dm_j\dm_{j-1}} \lam_j^i  z_{pq}^{ij}\right) \left(\xm_p \xm_q- \EE \xm_p \xm_q\right),\ \Em_{2}=\!\sum_{p}\!\!\left(
\sum_{j=1}^{l} \sum_{i=1}^{\dm_j\dm_{j-1}} \lam_j^i  z_{pp}^{ij}\right)\left(\xm_p^2 - \EE \xm_p^2\right),\notag\\
& \Em_{3}=\sum_{p}\left( \sum_{j=1}^{l} \sum_{i=1}^{\dm_j\dm_{j-1}}\lam_j^i y_{p}^{ij}\right)\left( \xm_p-\EE \xm_p \right).\label{asfsdf6wer58}
\end{align}

Thus, we can further separate the event as:
\begin{align*}
&\Pro\left(\Em>t\right)\leq  \Pro\!\left(\frac{1}{n}\sum_{k=1}^n\Em_{1}^k\!>\!\frac{t}{3}\right)\!+\!\Pro\! \left(\frac{1}{n}\sum_{k=1}^n\Em_{2}^k\!>\!\frac{t}{3}\!\right) \!+\!\Pro\left(\frac{1}{n}\sum_{k=1}^n\Em_{3}^k\!> \!\frac{t}{3}\right).
\end{align*}

Thus, to prove our conclusion, we can respectively establish the upper bounds of the three events.
To the end, for each input sample $\xmi{i}$, we divide its corresponding $\Qm_j-\EE \Qm_j$ into $\Em_{1}$, $\Em_{2}$ and $\Em_{3}$. Then we bound the three events separately. Before that, we first give several equalities. Since $ \Bmi{j:s}= \Wmi{j}\Wmi{j-1}\cdots\Wmi{s}\ (j\geq s)$, by Lemma~\ref{asfasasfshdfhfdet} we have
\begin{equation}\label{gr1adafsier3ent1}
\begin{split}
\|\Bmi{j:s}\|_F^2\leq r^{2(j-s+1)} \ \ \text{and}\ \ \left\|\Bmi{l:t+1}\right\|_F^2\|\Bmi{t-1:s+1}\|_F^2\left\|\Bmi{s-1:1}\right\|_F^2 \leq r^{2(l-2)},
\end{split}
\end{equation}
These two inequalities can be obtained by using $\|\Wmi{i}\|_F^2=\|\wmi{i}\|_2^2\leq r^2$.

{\bf\noindent{Step 1. Divide $\Qm_{j}-\EE \Qm_{j}$}}:
Note that $\emm=\vmi{l}-\ym=\Bmi{l:1}\xm-\ym$. Let $\Hm_j=\Bmi{j-1:1}\otimes\Bmi{l:j+1}^T$. Then we can further write Eqn.~\eqref{gradgdsgsient211} as
\begin{equation}\label{gradgdsgssfsafient211}
\Qm_j=\nabla_{\wmi{j}}f(\wm)=\Hm_j\left(\xm\otimes(\Bmi{l:1}\xm) -\xm\otimes\ym\right)= \Hm_j\left(\left(\Im_{\dm_0}\otimes \Bmi{l:1}\right)\left(\xm\otimes\xm\right)-\xm\otimes\ym\right),
\end{equation}
where $\Im_{\dm_0}\in\Rss{\dm_0}{\dm_0}$ is the identity matrix. According to Eqn.~\eqref{gradgdsgssfsafient211}, we can write the $i$-th entry of $\Qm_{j}$ as the form $\Qm_{j}^i=\sum_{p,q} z_{pq}^{ij}\xm_p \xm_q+\sum_{p} y_{p}^{ij} \xm_p+r^{ij}$ where $\xm_p$ denotes the $p$-th entry in $\xm$.  Let $\Zm_j=\Hm_j \left(\Im_{\dm_0}\otimes \Bmi{l:1}\right)\in\Rss{\dm_j\dm_{j-1}}{\dm_0^2}$. Then, we know that the $i$-th entry $Q_j^i=\Zm(i,:)\xm'$, where $\xm'=\xm\otimes\xm= [\xm_1\xm;\xm_2\xm;\cdots,\xm_{\dm_0}\xm]\in\Rs{\dm_0^2}$. In this way, we have $z_{pq}^{ij}=\Zm_j(i,(p-1)\dm_0+q)$ which further implies
\begin{equation}\label{gsrfg11}
\sum_{p,q} (z_{pq}^{ij})^2=\|\Zm_j(i,:)\|_2^2\leq\|\Zm_j\|_F^2\leq  \|\Bmi{j-1:1}\|_{F}^2\|\Bmi{l:j+1}\|_{F}^2\|\Im_{\dm_0}\|_{F}^2 \|\Bmi{l:1}\|_{F}^2
 \leq  z_y,
\end{equation}
where $z_y$ is defined as
\begin{equation*}\label{gsafradieant14211}
z_y\triangleq \dm_0 r^{2(l-1)}r^{2l}=\dm_0 r^{2(2l-1)}.
\end{equation*}
Note that Eqn.~\eqref{gsrfg11} uses the conclusion in Eqn.~\eqref{gr1adafsier3ent1}. We divide the $i$-th row $\Hm_j(i,:)$ into $\Hm_j(i,:)=[\Hm_{ji}^1,\Hm_{ji}^2,\cdots,\Hm_{ji}^{\dm_0}]$ where $\Hm_{ji}^p\in\Rs{1\times\dm_l}$. Then we have $y_p^{ij}=\ym^T\Hm_{ji}^p$. This yields
\begin{equation}\label{gsafrasfadisdsafdent14211}
\sum_{p} (y_{p}^{ij})^2=\sum_{p} (\ym^T\Hm_{ji}^p)^2\leq \sum_{p} \| \ym\|_2^2\|\Hm_{ji}^p \|_2^2= \| \ym\|_2^2\|\Hm_{j}(i,:) \|_2^2\leq \| \ym\|_2^2\|\Hm_{j} \|_F^2 \leq h_y,
\end{equation}
where $h_y$ is defined as
\begin{equation*}\label{gsafrasfadisdent14211}
 h_y\triangleq\|\ym\|_2^2 r^{2(l-1)}.
\end{equation*}

Let $\lam_j^i$ denote the $i$-th entry of $\lam_j$. Then, by Eqn.~\eqref{asfsdf6wer58}, we can obtain
\begin{equation*}\label{gradient14211}
\begin{split}
\sum_{j}\lr \lam_j,(\Qm_{j}-\EE(\Qm_{j}))\rl
&=\!\!\!\sum_{p,q:p\neq q}\!\!\! a_{pq}\left(\xm_p \xm_q- \EE \xm_p \xm_q\right)\!+\! \sum_{p} a_{pp}\left(\xm_p^2- \EE \xm_p^2\right)\!+\!\sum_{p} b_{p}\left(\xm_p- \EE \xm_p\right)\\
&=\Em_{1}+\Em_{2}+\Em_{3},
\end{split}
\end{equation*}
where $a_{pq}$ and $b_p$ are defined as
\begin{align*}
a_{pq}=
\sum_{j=1}^{l} \sum_{i=1}^{\dm_j\dm_{j-1}} \lam_j^i  z_{pq}^{ij}  \quad \text{and}\quad  b_{p}=\sum_{j=1}^{l} \sum_{i=1}^{\dm_j\dm_{j-1}}\lam_j^i y_{p}^{ij}.\label{asf658}
\end{align*}
Before we bound $a_{pq}$ and $b_{pq}$, we first give
\begin{equation}\label{ghtasfdasy34}
\sum_{i=1}^{\dm_j\dm_{j-1}}(z_{pq}^{ij})^2\leq \sum_{i=1}^{\dm_j\dm_{j-1}}\sum_{p,q}(z_{pq}^{ij})^2\led{172} \sum_{i=1}^{\dm_j\dm_{j-1}}z_y\leq z_y \max_j(\dm_j\dm_{j-1})\triangleq \omega.
\end{equation}
Notice \ding{172} uses Eqn.~\eqref{gsrfg11}.
Then we can utilize Eqn.~\eqref{ghtasfdasy34} and $\sum_{j=1}^{l}\left(\sum_{i=1}^{\dm_j\dm_{j-1}} (\lam_j^i)^2\right)=1$ to bound $a_{pq}$  as follows:
\begin{equation*}\label{gradsdfa78ient14211}
\begin{split}
a_{pq}^2\leq l\left(\sum_{j=1}^{l}\left(  \sum_{i=1}^{\dm_j\dm_{j-1}}\lam_j^i z_{pq}^{ij}\right)^2\right)\leq l\sum_{j=1}^{l}\left(\sum_{i=1}^{\dm_j\dm_{j-1}}(\lam_j^i)^2\right)\left(  \sum_{i=1}^{\dm_j\dm_{j-1}}(z_{pq}^{ij})^2\right)\leq l\omega.
\end{split}
\end{equation*}
which further gives
\begin{equation*}\label{gradsdfa78ient14211}
\begin{split}
\sum_{p,q} a_{pq}^2\leq l\sum_{j=1}^{l}\left(\sum_{i=1}^{\dm_j\dm_{j-1}} (\lam_j^i)^2\right)\left(  \sum_{i=1}^{\dm_j\dm_{j-1}}\sum_{p,q}(z_{pq}^{ij})^2\right)\led{172} l\omega.
\end{split}
\end{equation*}
where \ding{172} uses Eqn.~\eqref{ghtasfdasy34}. Similarly, we can bound $b_p$ as
\begin{equation*}\label{gradsdfa78ient14211}
\begin{split}
b_p^2\leq l\sum_{j=1}^{l}\left( \sum_{i=1}^{\dm_j\dm_{j-1}}\lam_j^i y_{p}^{ij}\right)^2\leq l\sum_{j=1}^{l}\left( \sum_{i=1}^{\dm_j\dm_{j-1}}(\lam_j^i)^2 \right)\left( \sum_{i=1}^{\dm_j\dm_{j-1}}(y_{p}^{ij})^2\right)\leq l\omega',
\end{split}
\end{equation*}
where $\omega'=h_y\max_j(\dm_j\dm_{j-1})$.
Accordingly, we can have
\begin{equation*}\label{gradsdfsfa78ient14211}
\begin{split}
\sum_{p} b_p^2\leq   l\sum_{j=1}^{l}\left( \sum_{i=1}^{\dm_j\dm_{j-1}}(\lam_j^i)^2 \right)\left( \sum_{i=1}^{\dm_j\dm_{j-1}}\sum_{p}(y_{p}^{ij})^2\right) \overset{\text{\ding{172}}}{\leq}  l \omega',
\end{split}
\end{equation*}
where \ding{172} uses \eqref{gsafrasfadisdsafdent14211}.

{\bf\noindent{Step 2. Bound  $\Pro(\Em_{1}>t/3)$, $\Pro(\Em_{2}>t/3)$ and $\Pro(\Em_{3}>t/3)$}}: Let $\Em_{h1}^k$ denotes the $\Em_{h1}$ which corresponds to
the $k$-th sample $\xmi{k}$. Therefore, we can bound
\begin{equation*}\label{gradient14211}
\begin{split}
\Pro\left(\frac{1}{n}\sum_{k=1}^n\Em_{1}^k> \frac{t}{3}\right)= &\Pro\left(s\sum_{k=1}^n\left( \sum_{p,q:p\neq q} a_{pq}^k\left(\xm_p^k \xm_q^k- \EE \xm_p^k \xm_q^k\right) \right)>\frac{snt}{3}\right)\\
\led{172} &\exp\left(-\frac{nst}{3}\right) \EE\exp\left(s \sum_{k=1}^n\left( \sum_{p,q:p\neq q} a_{pq}^k\left(\xm_p^k \xm_q^k- \EE \xm_p^k \xm_q^k\right) \right)\right)\\
\led{173} &\exp\left(-\frac{nst}{3}\right)\prod_{k=1}^{n} \EE\exp\left(s \left( \sum_{p,q:p\neq q} a_{pq}^k\left(\xm_p^k \xm_q^k- \EE \xm_p^k \xm_q^k\right) \right)\right)\\
\led{174} &\exp\left(-\frac{nst}{3}\right)\prod_{k=1}^{n} \exp\left(2\tau^2s^2 \sum_{p,q:p\neq q} (a_{pq}^k)^2\right) \quad |s| \leq \frac{1}{2\tau\sqrt{l\omega}}\\
\leq &\exp\left(-\frac{nst}{3}\right)\prod_{j=1}^{n} \exp\left(2\tau^2s^2  l\omega\right)\\
\led{175} &\exp\left(-c'n\min\left(\frac{t^2}{\omega l\tau^2},
\frac{t}{\sqrt{l\omega}\tau}\right)\right),
\end{split}
\end{equation*}
where \ding{172} holds because of Chebyshev's inequality. \ding{173} holds since $\xmi{i}$ are independent. \ding{174} is established by applying Lemma~\ref{sumsafsad678exp}. We have \ding{175} by optimizing $s$. Similarly, by Lemma~\ref{sumsafa34sad678safexp} we can bound $\Pro\left(\frac{1}{n}\sum_{k=1}^n\Em_{2}^k> \frac{t}{3}\right)$ as follows:
\begin{equation*}\label{gradient14211}
\begin{split}
\Pro\left(\frac{1}{n}\sum_{k=1}^n\Em_{2}^k> \frac{t}{3}\right)
\leq &\exp\left(-\frac{nst}{3}\right)\prod_{k=1}^{n} \EE\exp\left(s \left( \sum_{p} a_{pp}^k\left((\xm_p^k)^2- \EE (\xm_p^k)^2\right) \right)\right)\\
\leq &\exp\left(-\frac{nst}{3}\right)\prod_{k=1}^{n} \exp\left(128\tau^4s^2 l\omega\right) \quad |s| \leq \frac{1}{\tau^2\sqrt{l\omega}}\\
\leq &\exp\left(-c''n\min\left(\frac{t^2}{\omega l\tau^4},
\frac{t}{\sqrt{l\omega}\tau^2}\right)\right).
\end{split}
\end{equation*}
Finally, since $\xmi{i}$ are independent sub-Gaussian, we can use Hoeffding inequality and obtain
\begin{equation*}\label{gradient14211}
\begin{split}
\Pro\left(\frac{1}{n}\sum_{k=1}^n\Em_{3}^k> \frac{t}{3}\right)\leq
\Pro\left(\frac{1}{n}\sum_{k=1}^n\left( \sum_{p} b_{p}^k\left(\xm_p^k- \EE \xm_p^k\right) \right)> \frac{t}{3}\right)
 \exp\left(-\frac{c'''nt^2}{\omega'l \tau^2}\right).
\end{split}
\end{equation*}

{\bf\noindent{Step 3. Bound $\Pro\!\left(\Em\!>\!t\right)$}}:   By comparing the values of $\omega$ and $\omega'$, we can obtain
\begin{align*}
\Pro\left(\Em>t\right)\leq&  \Pro\!\left(\frac{1}{n}\sum_{k=1}^n\Em_{1}^j\!>\!\frac{t}{3}\right)\!+\!\Pro\! \left(\frac{1}{n}\sum_{k=1}^n\Em_{2}^j\!>\!\frac{t}{3}\!\right) \!+\!\Pro\left(\frac{1}{n}\sum_{k=1}^n\Em_{3}^j\!> \!\frac{t}{3}\right)\\
\leq & 3\exp\left(-c_{g'}n\min\left(\frac{t^2}{ l \max\left(\omega_g\tau^2,\omega_g\tau^4,\omega_{g'}\tau^2\right)},
\frac{t}{\sqrt{l\omega_g}\max\left(\tau,\tau^2\right)}\right)\right),
\end{align*}
where $\omega_g=\dm_0 r^{2(2l-1)}\max_j(\dm_j\dm_{j-1})$ and $\omega_{g'}= r^{2(l-1)}\max_j(\dm_j\dm_{j-1})$. The proof is completed.
\end{proof}

\subsubsection{Proofs of Lemma~~\ref{thm: hessian}}
\begin{proof}[\hypertarget{lemmahessian}{Proof}]
For brevity, let $\Qm_{ts}$ denote $\nabla_{\wmi{t}}\left(\nabla_{\wmi{s}}f(\wm,\xm)\right)$. Then, by Lemma~\ref{uh2397} we have
\begin{equation*}\label{gradient1}
\Qm_{ts}\!=\!\!
\begin{cases}
\left(\Bmi{l:s+1}^T\emm\xm^T\Bmi{t-1:1}^T\right)\otimes \Bmi{s-1:t+1}+\left(\Bmi{s-1:1}\xm \xm^T \Bmi{t-1:1}^T\right)\otimes\left(\Bmi{l:s+1}^T \Bmi{l:t+1}\right),&\!\!\!\mbox{if } s>t,\\
\left(\Bmi{s-1:1}\xm\xm^T\Bmi{s-1:1}\right)\otimes\left({\Bmi{l:s+1}}^T\Bmi{l:s+1}\right), & \!\!\!\mbox{if } s=t, \\
\left(\Bmi{t-1:s+1}^T\right)\otimes\left(\Bmi{s-1:1}\xm\emm^T \Bmi{l:t+1}^T\right)+\left(\Bmi{s-1:1}\xm \xm^T \Bmi{t-1:1}^T\right)\otimes\left(\Bmi{l:s+1}^T \Bmi{l:t+1}\right),& \!\!\!\mbox{if } s<t.\\
\end{cases}
\end{equation*}
Then we know that the $(i,k)$-th entry $\Qm_{ts}^{ik}$ has the form $\Qm_{ts}^{ik}=\sum_{p,q} z_{pq}^{ik}\xm_p \xm_q+\sum_{p} y_{p}^{ik} \xm_p+r^{ik}$ (explained in the following Step 1.~I) where $\xm_p$ denotes the $p$-th entry in $\xm$. Note that $z_{pq}^{ik}, y_{p}^{ik}$ and $r^{ik}$ are constant and independent on $\xm$. For convenience, we let $\Qm_{ts}=\Hm_{ts}+\Gm_{ts}$, where $\Gm_{ts}=\left(\Bmi{s-1:1}\xm \xm^T \Bmi{t-1:1}^T\right)\otimes\left(\Bmi{l:s+1}^T \Bmi{l:t+1}\right)$ and $\Hm_{ts}$ is defined as
\begin{equation*}\label{gradient1}
\Hm_{ts}=
\begin{cases}
\left(\Bmi{l:s+1}^T\emm\xm^T\Bmi{t-1:1}^T\right)\otimes \Bmi{s-1:t+1},& \mbox{if } s>t, \\
\ \bm{0}, & \mbox{if } s=t, \\
\left(\Bmi{t-1:s+1}^T\right)\otimes\left(\Bmi{s-1:1}\xm\emm^T \Bmi{l:t+1}^T\right), & \mbox{if } s<t.
\end{cases}
\end{equation*}
Let
\begin{align*}
&\Em=\frac{1}{n}\sum_{j=1}^n\lr \lam,\left(\nabla_{\wm}^2f(\wm,\xm)-\EE \nabla_{\wm}^2f(\wm,\xm)\right)\lam\rl,\ \Em_h=\frac{1}{n}\sum_{j=1}^n \!\sum_{t,s}\! \lr \lam_t,\!\left(\Hm_{ts}\!-\!\EE (\Hm_{ts})\right)\lam_s\rl,\\
&\Em_g=\frac{1}{n}\sum_{j=1}^n \!\sum_{t,s}\! \lr \lam_t,\!\left(\Gm_{ts}\!-\!\EE (\Gm_{ts})\right)\lam_s\rl.
\end{align*}
Then we divide the event as two events:
\begin{align*}
&\Pro\left(\Em>t\right)=\Pro\left(\Em_h+\Em_g>t\right)\leq \Pro\left(\Em_h>t/2\right)+\Pro\left(\Em_g>t/2\right).
\end{align*}
Now we look each event separately. Similar to $\Qm_{ts}$, the $(i,k)$-th entry $\Hm_{ts}^{ik}$ has the form $\Hm_{ts}^{ik}=\sum_{p,q} z_{pq}^{ik}\xm_p \xm_q+\sum_{p} y_{p}^{ik} \xm_p+r^{ik}$. We divide the unit vector $\lam\in\Rs{d}$ as $\lam=(\lam_1;\cdots;\lam_l)$ where $\lam_j\in\Rs{\dm_j\dm_{j-1}}$. For input vector $\xm$, let $\sum_{t,s} \lr \lam_t,\left(\Hm_{ts}-\EE (\Hm_{ts})\right)\lam_s\rl=\Em_{h1}+\Em_{h2}+\Em_{h3}$, where
\begin{align}
&\Em_{h1}\!\!=\!\!\!\sum_{p,q:p\neq q}\!\!\left(
\sum_{t,s} \sum_{i,k}(\lam_t^i\lam_s^k) z_{pq}^{ik}\right)\!\! \left(\xm_p \xm_q\!- \!\EE \xm_p \xm_q\right),\ \Em_{h2}=\!\sum_{p}\!\left(
\sum_{t,s} \sum_{i,k}(\lam_t^i\lam_s^k) z_{pq}^{ik}\right)\!\! \left(\xm_p^2 \!- \!\EE \xm_p^2\right),\notag\\
& \Em_{h3}=\sum_{p} \left(\sum_{t,s}\sum_{i,k}(\lam_t^i\lam_s^k) y_{p}^{ik}\right) \left(\xm_p-\EE \xm_p \right),\label{asf658}
\end{align}
where $\xm_p$ denotes the $p$-th entry in $\xm$ and $\lam_j^i$ denotes the $i$-th entry of $\lam_j$. Let $\Em_{h_1}^j$, $\Em_{h_2}^j$, and $\Em_{h_3}^j$ denote the $\Em_{h_1}$, $\Em_{h_2}$, and $\Em_{h_3}^j$ of the $j$-th sample. Thus, considering $n$ samples, we can further separately divide the two events above as:
\begin{align*}
&\Pro\!\left(\!\Em_h\!>\!\frac{t}{2}\!\right)\!\leq \! \Pro\!\left(\frac{1}{n}\sum_{j=1}^n\Em_{h1}^j\!>\!\frac{t}{6}\right)\!+\!\Pro\! \left(\frac{1}{n}\sum_{j=1}^n\Em_{h2}^j\!>\!\frac{t}{6}\!\right) \!+\!\Pro\left(\frac{1}{n}\sum_{j=1}^n\Em_{h3}^j\!> \!\frac{t}{6}\right).
\end{align*}
Similarly, we can define $\Em_{g1}, \Em_{g2}$ and $\Em_{g3}$.
\begin{align*}
&\Pro\!\left(\!\Em_g\!>\!\frac{t}{2}\!\right)\!\leq \! \Pro\!\left(\frac{1}{n}\sum_{j=1}^n\Em_{g1}^j\!>\!\frac{t}{6}\right)\!+\!\Pro\! \left(\frac{1}{n}\sum_{j=1}^n\Em_{g2}^j\!>\!\frac{t}{6}\!\right) \!+\!\Pro\left(\frac{1}{n}\sum_{j=1}^n\Em_{g3}^j\!> \!\frac{t}{6}\right).
\end{align*}
Thus, to prove our conclusion, we can respectively establish the upper bounds of $\Pro\!\left(\!\Em_h\!>\!\frac{t}{2}\!\right)$ and $\Pro\!\left(\!\Em_g\!>\!\frac{t}{2}\!\right)$.

{\bf\noindent{Step 1: Bound $\Pro\!\left(\!\Em_h\!>\!\frac{t}{2}\!\right)$}}

To achieve our goal, for each input sample $\xmi{i}$, we divide its corresponding $\sum_{t,s}(\Hm_{ts}-\EE \Hm_{ts})$ as $\Em_{h1}$, $\Em_{h2}$ and $\Em_{h3}$. Then we bound the three events separately. Before that, we first give two equalities. Since $ \Bmi{j:s}= \Wmi{j}\Wmi{j-1}\cdots\Wmi{s}\ (j\geq s)$, by Lemma~\ref{asfasasfshdfhfdet} we have
\begin{equation}\label{gr1adafsient1}
\begin{split}
\|\Bmi{j:s}\|_F^2\leq r^{2(j-s+1)} \ \ \text{and}\ \ \left\|\Bmi{l:t+1}\right\|_F^2\|\Bmi{t-1:s+1}\|_F^2\left\|\Bmi{s-1:1}\right\|_F^2 \leq r^{2(l-2)},
\end{split}
\end{equation}
These two inequalities can be obtained by using $\|\Wmi{i}\|_F^2=\|\wmi{i}\|_2^2\leq r^2$.

{\bf\noindent{I. Divide $\Hm_{ts}-\EE \Hm_{ts}$}}: For $t\neq s$, we can write the $(i,k)$-th entry $\Hm_{ts}^{ik}$ as the form $\Hm_{ts}^{ik}=\sum_{p,q} z_{pq}^{ik}\xm_p \xm_q+\sum_{p} y_{p}^{ik} \xm_p+r^{ik}$. Now we try to bound $z_{pq}^{ik}$ and $y_{p}^{ik}$. We first consider the case $s<t$. Note that $\emm=\vmi{l}-\ym=\Bmi{l:1}\xm-\ym$.
Specifically, we have
\begin{equation}\label{gradsfa43ient1}
\Hm_{ts}
=\left(\Bmi{t-1:s+1}^T\right)\otimes \left(\Bmi{s-1:1}\xm\xm^T\Bmi{l:1}^T\Bmi{l:t+1}^T-\Bmi{s-1:1}\xm\ym^T \Bmi{l:t+1}^T\right).
\end{equation}
So the $(i',k')$-th entry in the matrix $\Bmi{s-1:1}\xm\xm^T\Bmi{l:1}^T\Bmi{l:t+1}^T$ is $[\Bmi{s-1:1}\xm\xm^T\Bmi{l:1}^T\Bmi{l:t+1}^T]_{i'k'}=(\Bmi{s-1:1})(i',:)\xm (\Bmi{l:1}\Bmi{l:t+1})(k',:)\xm=\xm^T((\Bmi{s-1:1})(i',:))^T (\Bmi{l:1}\Bmi{l:t+1})(k',:)\xm$, where $\Am(i',:)$ denotes the $i'$-th row of $\Am$. Let $i'_i=\mcode{mod}(i,\dm_s)$, $k'_k=\mcode{mod}(k,\dm_{t-1})$,  $i''_i=\lfloor i/\dm_s\rfloor$ and $k''_k=\lfloor k/\dm_{t-1}\rfloor$. In this case, the $(i,k)$-th entry $\Hm_{ts}^{ik}
=[\Bmi{t-1:s+1}]_{k''_ki''_i}\xm^T((\Bmi{s-1:1})(i'_i,:))^T (\Bmi{l:1}\Bmi{l:t+1})(k'_k,:)\xm+ [\Bmi{t-1:s+1}]_{k''_ki''_i}\ym^T (\Bmi{l:t+1})(k'_k,:)^T (\Bmi{s-1:1})(i'_i,:)\xm$. Therefore, we have
\begin{equation}\label{sagd5t6}
\begin{split}
\sum_{p,q} (z_{pq}^{ik})^2\!= \! [\Bmi{t-1:s+1}]_{k''_ki''_i}^2\left\| ((\Bmi{s-1:1})(i'_i,:))^T (\Bmi{l:1}\Bmi{l:t+1})(k'_k,:)\right\|_F^2
\!\overset{\text{\ding{172}}}{\leq} \! r^{4(l-1)}\triangleq z_y,
\end{split}
\end{equation}
where \ding{172} uses Eqn.~\eqref{gr1adafsient1} and these three inequalities: $[\Bmi{t-1:s+1}]_{k''_ki''_i}^2\leq \|\Bmi{t-1:s+1}\|_F^2$, $\left\|(\Bmi{s-1:1})(i'_i,:)\right\|_F^2\leq \left\|\Bmi{s-1:1}\right\|_F^2$, $\left\|(\Bmi{l:1}\Bmi{l:t+1})(k'_k,:)\right\|_F^2\leq \left\|\Bmi{l:1}\Bmi{l:t+1}\right\|_F^2$.

Similarly, we can bound
\begin{equation}\label{saf425}
\begin{split}
\sum_{p} (y_{p}^{ik})^2= [\Bmi{t-1:s+1}]_{k''_ki''_i}^2\left\| \ym^T (\Bmi{l:t+1})(k'_k,:)^T (\Bmi{s-1:1})(i'_i,:)\right\|_F^2
\overset{\text{\ding{172}}}{\leq} \left\|\ym\right\|_2^2 r^{2(l-2)}\triangleq h_y,
\end{split}
\end{equation}
where \ding{172} uses Eqn.~\eqref{gr1adafsient1} and $[\Bmi{t-1:s+1}]_{k''_ki''_i}^2\leq \|\Bmi{t-1:s+1}\|_F^2$.

Note that for the case $s\geq t$, Eqn.~\eqref{sagd5t6} and ~\eqref{saf425} also holds.
Let $\lam_j^i$ denote the $i$-th entry of $\lam_j$. Then, by Eqn.~\eqref{asf658}, we can obtain
\begin{equation*}\label{gradient14211}
\begin{split}
\sum_{t,s}\!\!\left(\lr \lam_t,(\Hm_{ts}\!-\!\EE(\Hm_{ts}))\lam_s\rl\right)
\!&=\!\!\!\sum_{p,q:p\neq q}\!\! a_{pq}\left(\xm_p \xm_q\!- \!\EE \xm_p \xm_q\right)\!+\! \sum_{p} a_{pp}\left(\xm_p^2\!- \!\EE \xm_p^2\right)\!+\!\sum_{p} b_{p}\left(\xm_p\!- \!\EE \xm_p\right)\\
&=\Em_{h1}+\Em_{h2}+\Em_{h3},
\end{split}
\end{equation*}
where $a_{pq}$ and $b_p$ are defined as
\begin{align*}
a_{pq}=
\sum_{t,s} \sum_{i,k}(\lam_t^i\lam_s^k) z_{pq}^{ik}\quad \text{and}\quad  b_{p}=\sum_{t,s}\sum_{i,k}(\lam_t^i\lam_s^k) y_{p}^{ik}.\label{asf658}
\end{align*}
Before we bound $a_{pq}$ and $b_{pq}$, we first give
\begin{equation}\label{ghty34}
\sum_{i,k}(z_{pq}^{ik})^2\leq \sum_{i,k}\sum_{p,q}(z_{pq}^{ik})^2\led{172} \sum_{i,k}z_y\leq z_y \left(\max_j(\dm_j\dm_{j-1})\right)^2\triangleq \omega.
\end{equation}
Note that \ding{172} uses Eqn.~\eqref{sagd5t6}  and $i\in\{1,\cdots,\dm_i\dm_{i-1}\},\ j\in\{1,\cdots,\dm_j\dm_{j-1}\}$.
Besides, we have $\sum_{t,s}\left( \sum_{i,k}(\lam_t^i\lam_s^k)^2 \right)=1$. Therefore we can bound $a_{pq}$ as follows:
\begin{equation*}\label{gradsdfa78ient14211}
\begin{split}
a_{pq}^2 \!\leq\! l^2\sum_{t,s}\!\!\left( \! \sum_{i,k}(\lam_t^i\lam_s^k) z_{pq}^{ik}\!\right)^2 \!\!\!\leq \! l^2\sum_{t,s}\!\!\left( \! \sum_{i,k}(\lam_t^i\lam_s^k)^2 \!\right)\!\!\left( \! \sum_{i,k}(z_{pq}^{ik})^2\!\right)\!\leq\! \omega l^2\sum_{t,s}\!\!\left(\! \sum_{i,k}(\lam_t^i\lam_s^k)^2 \!\right)\!\leq\! \omega l^2.
\end{split}
\end{equation*}
which further yields
\begin{equation*}\label{gradsdfa78ient14211}
\begin{split}
\sum_{p,q} a_{pq}^2 \leq l^2\sum_{t,s}\left( \sum_{i,k}(\lam_t^i\lam_s^k)^2 \right)\left( \sum_{i,k}\sum_{p,q}(z_{pq}^{ik})^2\right)\leq \omega l^2\sum_{t,s}\left( \sum_{i,k}(\lam_t^i\lam_s^k)^2 \right)\leq \omega l^2.
\end{split}
\end{equation*}

Similarly, we have
\begin{equation*}\label{gradsdfa78ient14211}
\begin{split}
b_p^2 \leq  l^2\sum_{t,s}\left(\sum_{i,k}(\lam_t^i\lam_s^k) y_{p}^{ik}\right)^2 \leq  l^2\sum_{t,s}\left(\sum_{i,k}(\lam_t^i\lam_s^k)^2\right)\left( \sum_{i,k} (y_{p}^{ik})^2\right)\led{172} \omega'l^2,
\end{split}
\end{equation*}
where $\omega'=h_y\left(\max_j(\dm_j\dm_{j-1})\right)^2$. Note that \ding{172} uses \eqref{saf425}.
Accordingly, we can have
\begin{equation*}\label{gradsdfa78ient14211}
\begin{split}
\sum_{p}b_p^2 \!\leq\! l^2\sum_{t,s}\left(\sum_{i,k}(\lam_t^i\lam_s^k)^2\right)\left( \sum_{i,k} \sum_{p}(y_{p}^{ik})^2\right) \!\leq\!  \omega'l^2.
\end{split}
\end{equation*}
{\bf\noindent{II. Bound  $\Pro(\Em_{h1}>t/6)$, $\Pro(\Em_{h2}>t/6)$ and $\Pro(\Em_{h3}>t/6)$}}:
Let $E_{h1}^j$ denotes the $\Em_{h1}^j$ which corresponds to
the $j$-th sample $\xmi{i}$. Therefore, we can bound
\begin{equation*}\label{gradient14211}
\begin{split}
\Pro\left(\frac{1}{n}\sum_{j=1}^n\Em_{h1}^j> \frac{t}{6}\right)\leq &\Pro\left(s\sum_{j=1}^n\left( \sum_{p,q:p\neq q} a_{pq}^j\left(\xm_p^j \xm_q^j- \EE \xm_p^j \xm_q^j\right) \right)>\frac{snt}{6}\right)\\
\led{172} &\exp\left(-\frac{nst}{6}\right) \EE\exp\left(s \sum_{j=1}^n\left( \sum_{p,q:p\neq q} a_{pq}^j\left(\xm_p^j \xm_q^j- \EE \xm_p^j \xm_q^j\right) \right)\right)\\
\led{173} &\exp\left(-\frac{nst}{6}\right)\prod_{j=1}^{n} \EE\exp\left(s \left( \sum_{p,q:p\neq q} a_{pq}^j\left(\xm_p^j \xm_q^j- \EE \xm_p^j \xm_q^j\right) \right)\right)\\
\led{174} &\exp\left(-\frac{nst}{6}\right)\prod_{j=1}^{n} \exp\left(2\tau^2s^2 \sum_{p,q:p\neq q} (a_{pq}^j)^2\right) \quad |s| \leq \frac{1}{2\tau l\sqrt{\omega}}\\
\leq &\exp\left(-\frac{nst}{6}\right)\prod_{j=1}^{n} \exp\left(2\tau^2s^2  l^2\omega\right)\\
\led{175} &\exp\left(-c'n\min\left(\frac{t^2}{\omega l^2\tau^2},
\frac{t}{\sqrt{\omega} l\tau}\right)\right),
\end{split}
\end{equation*}
where \ding{172} holds because of Chebyshev's inequality. \ding{173} holds since $\xmi{i}$ are independent. \ding{174} is established because of Lemma~\ref{sumsafsad678exp}. We have \ding{175} by optimizing $s$. Similarly, we can bound $\Pro\left(\frac{1}{n}\sum_{j=1}^n\Em_{h2}^j> \frac{t}{6}\right)$ as follows:
\begin{equation*}\label{gradient14211}
\begin{split}
\Pro\left(\frac{1}{n}\sum_{j=1}^n\Em_{h2}^j> \frac{t}{6}\right)
\leq &\exp\left(-\frac{nst}{6}\right)\prod_{j=1}^{n} \EE\exp\left(s \left( \sum_{p} a_{pp}^j\left((\xm_p^j)^2- \EE (\xm_p^j)^2\right) \right)\right)\\
\leq &\exp\left(-\frac{nst}{6}\right)\prod_{j=1}^{n} \exp\left(128\tau^4s^2 l^2\omega\right) \quad |s| \leq \frac{1}{\tau^2 l\sqrt{\omega}}\\
\leq &\exp\left(-c''n\min\left(\frac{t^2}{\omega l^2\tau^4},
\frac{t}{\sqrt{\omega} l\tau^2}\right)\right).
\end{split}
\end{equation*}
Finally, since $\xmi{i}$ are independent sub-Gaussian, we can use Hoeffding inequality and obtain
\begin{equation*}\label{gradient14211}
\begin{split}
\Pro\left(\frac{1}{n}\sum_{j=1}^n\Em_{h3}^j> \frac{t}{6}\right)=
\Pro\left(\frac{1}{n}\sum_{j=1}^n\left( \sum_{p} b_{p}^j\left(\xm_p^j- \EE \xm_p^j\right) \right)> \frac{t}{6}\right)\leq
 \exp\left(-\frac{c'''nt^2}{\omega'l^2\tau^2}\right).
\end{split}
\end{equation*}
Since for $s=t$, $\Pro\left(\frac{1}{n}\sum_{j=1}^n\Em_{h1}^j> \frac{t}{6}\right)=\Pro\left(\frac{1}{n}\sum_{j=1}^n\Em_{h2}^j> \frac{t}{6}\right)=\Pro\left(\frac{1}{n}\sum_{j=1}^n\Em_{h3}^j> \frac{t}{6}\right)=0$, the above upper bounds also hold.

{\bf\noindent{III: Bound $\Pro\!\left(\!\Em_h\!>\!\frac{t}{2}\!\right)$}}
By comparing the values of $\omega$ and $\omega'$, we can obtain
\begin{align*}
\Pro\!\left(\!\Em_h\!>\!\frac{t}{2}\!\right) \leq&   \Pro\!\left(\frac{1}{n}\sum_{j=1}^n\Em_{h1}^j\!>\!\frac{t}{6}\right)\!+\!\Pro\! \left(\frac{1}{n}\sum_{j=1}^n\Em_{h2}^j\!>\!\frac{t}{6}\!\right) \!+\!\Pro\left(\frac{1}{n}\sum_{j=1}^n\Em_{h3}^j\!> \!\frac{t}{6}\right)\\
\leq & 3\exp\left(-c_2'n\min\left(\frac{t^2}{ l^2\max\left(\omega \tau^2,\omega \tau^4, \omega_q\tau^2\right)},
\frac{t}{\sqrt{\omega} l\max\left(\tau,\tau^2\right)}\right)\right),
\end{align*}
where $\omega_{q}=r^{2(l-2)}\left(\max_j(\dm_j\dm_{j-1})\right)^2$.

{\bf\noindent{Step 2: Bound $\Pro\!\left(\!\Em_g\!>\!\frac{t}{2}\!\right)$}}
To achieve our goal, for each input sample $\xmi{i}$, we also divide its corresponding $\sum_{t,s}(\Gm_{ts}-\EE \Gm_{ts})$ as $\Em_{h1}$, $\Em_{h2}$ and $\Em_{h3}$. Then we bound the three events separately. Before that, we first give several equalities.

{\bf\noindent{I. Divide $\Gm_{ts}-\EE \Gm_{ts}$}}:
Dividing $\Gm_{ts}-\EE \Gm_{ts}$ is more easy than dividing $\Hm_{ts}-\EE \Hm_{ts}$ since the later has more complex form.  Since $\Gm_{ts}=\left(\Bmi{s-1:1}\xm \xm^T \Bmi{t-1:1}^T\right)\otimes\left(\Bmi{l:s+1}^T \Bmi{l:t+1}\right)$. we also can write the $(i,k)$-th entry $\Gm_{ts}^{ik}$ as the form $\Gm_{ts}^{ik}=\sum_{p,q} z_{pq}^{ik}\xm_p \xm_q+\sum_{p} y_{p}^{ik} \xm_p+r^{ik}$. But here $y_{p}^{ik}=0$.

Then similar to the step in dividing $\Hm_{ts}-\EE \Hm_{ts}$, we can bound
\begin{equation*}\label{gradsdfa78ient14211}
\begin{split}
a_{pq}^2 \leq \omega_g l^2\quad \text{and}\quad \sum_{p,q} a_{pq}^2 \leq \omega_g l^2\quad \text{where}\ \omega_g=r^{4(l-1)} \left(\max_j(\dm_j\dm_{j-1})\right)^2.
\end{split}
\end{equation*}

{\bf\noindent{II. Bound  $\Pro(\Em_{g1}>t/6)$, $\Pro(\Em_{g2}>t/6)$ and $\Pro(\Em_{g3}>t/6)$}}: Since $y_{p}^{ik}=0$, $\Pro(\Em_{h3}>t/6)=0$.
Similar to the above methods, we can bound
\begin{equation*}\label{gradient14211}
\begin{split}
\Pro\left(\frac{1}{n}\sum_{j=1}^n\Em_{g1}^j> \frac{t}{6}\right)\leq  &\exp\left(-c_1'n\left(\frac{t^2}{\omega_g l^2\tau^2},
\frac{t}{\sqrt{\omega_g} l\tau}\right)\right),
\end{split}
\end{equation*}
and
\begin{equation*}\label{gradient14211}
\begin{split}
\Pro\left(\frac{1}{n}\sum_{j=1}^n\Em_{g2}^j> \frac{t}{6}\right)
\leq \exp\left(-c_1''n\left(\frac{t^2}{\omega_g l^2\tau^4},
\frac{t}{\sqrt{\omega_g} l\tau^2}\right)\right).
\end{split}
\end{equation*}

{\bf\noindent{III: Bound $\Pro\!\left(\!\Em_h\!>\!\frac{t}{2}\!\right)$}}
We can obtain $\Pro\!\left(\!\Em_g\!>\!\frac{t}{2}\!\right)$ as follows:
\begin{align*}
\Pro\!\left(\!\Em_g\!>\!\frac{t}{2}\!\right) \leq& \Pro\!\left(\frac{1}{n}\sum_{j=1}^n\Em_{g1}^j\!>\!\frac{t}{6}\right)\!+\!\Pro\! \left(\frac{1}{n}\sum_{j=1}^n\Em_{g2}^j\!>\!\frac{t}{6}\!\right) \!+\!\Pro\left(\frac{1}{n}\sum_{j=1}^n\Em_{g3}^j\!> \!\frac{t}{6}\right)\\
\leq & 2\exp\left(-c_2'n\min\left(\frac{t^2}{\omega_g l^2\max\left(\tau^2,\tau^4\right)},
\frac{t}{\sqrt{\omega_g} l\max\left(\tau,\tau^2\right)}\right)\right).
\end{align*}

{\bf\noindent{Step 3: Bound $\Pro\!\left(\!\Em\!>\!t\!\right)$}} 
Finally, we combine the above results and obtain
\begin{equation*}\label{gradient14211}
\begin{split}
\Pro\left(\Em>t\right) \leq &    \Pro\left(\Em_{h}\!>\!\frac{t}{2}\!\right)\!+\! \Pro\left(\!\Em_{g}\!>\!\frac{t}{2}\!\right)\\
\leq & 5\exp\left(-c_{h'}n\min\left(\frac{t^2}{ \tau^2l^2\max\left(\omega_g ,\omega_g \tau^2, \omega_h \right)},
\frac{t}{\sqrt{\omega_g} l\max\left(\tau,\tau^2\right)}\right)\right),
\end{split}
\end{equation*}
where $\omega_g=\left(\max_j(\dm_j\dm_{j-1})\right)^2 r^{4(l-1)}$ and $\omega_h=\left(\max_j(\dm_j\dm_{j-1})\right)^2 r^{2(l-2)}$.
\end{proof}

\subsubsection{Proof of Lemma~\ref{thmgradisagf5675ent12}}
\begin{proof}[\hypertarget{lemmabound}{Proof}]
Before proving our conclusion, we first give an inequality:
\begin{equation*}\label{gradient1}
\begin{split}
\left\|\emm\right\|_2^2
=\left\|\Bmi{l:1}\xm-\ym\right\|_2^2
\leq \left\|\Bmi{l:1}\xm\right\|_2^2+2\left| \ym^T\Bmi{l:1}\xm\right| +\left\|\ym\right\|_2^2
\overset{\text{\ding{172}}}{\leq}   r_x^2 \omega_f^2+2r_x\omega_f\|\ym\|_2 +\left\|\ym\right\|_2^2,
\end{split}
\end{equation*}
where $\omega_{f}=r^{l}$. Notice, \ding{172} holds since by Lemma~\ref{asfasasfshdfhfdet},  we have $\left\|\Bmi{l:1}\right\|_F^2 \leq r^{2l}$.

Then we consider $\nabla_{\wm}f(\wm,\xm)$. Firstly, by Lemma~\ref{uh2397} we can bound $\|\nabla_{\wmi{j}}f(\wm,\xm)\|_2^2$ as follows:
\begin{equation*}\label{gradient1}
\begin{split}
\|\nabla_{\wmi{j}}f(\wm,\xm)\|_2^2=&\left\|\left((\Bmi{j-1:1}\xm) \otimes\Bmi{l:j+1}^T\right) \emm\right\|_2^2\leq \left\|\Bmi{j-1:1}\right\|_2^2\left\|\xm\right\|_2^2\left\| \Bmi{l:j+1}\right\|_2^2  \left\|\emm\right\|_2^2\\
\overset{\text{\ding{172}}}{\leq}& r_x^2\omega_{f_1}^2\left(r_x^2 \omega_f^2+2r_x\omega_f\|\ym\|_2 +\left\|\ym\right\|_2^2\right),
\end{split}
\end{equation*}
where $\omega_{f_1}=r^{(l-1)}$. \ding{172} holds since we have $\left\|\Bmi{l:j+1}\right\|_F^2\|\Bmi{j-1:1}\|_F^2 \leq r^{2(l-1)}$  by using $\|\Wmi{i}\|_F^2=\|\wmi{i}\|_2^2\leq r^2$.
Therefore, we can further obtain
\begin{equation*}\label{gradient1}
\begin{split}
\|\nabla_{\wm}f(\wm,\xm)\|_2^2=\sum_{i=1}^{l} \|\nabla_{\wmi{i}}f(\wm,\xm)\|_2^2 \leq  lr_x^2\omega_{f_1}^2\left(r_x^2 \omega_f^2+2r_x\omega_f\|\ym\|_2 +\left\|\ym\right\|_2^2\right).
\end{split}
\end{equation*}
Notice, $\ym$ is the label of sample and the weight magnitude $r$ is usually lager than 1. Then we have $\|\ym\|_2\leq r^l$. Also, the values in input data are usually smaller than $r^l$. Thus, we have
\begin{equation*}\label{gradient1}
\begin{split}
\|\nabla_{\wm}f(\wm,\xm)\|_2^2 \leq  c_{t}lr_x^4 r^{4l-2}\triangleq \alpha_g,
\end{split}
\end{equation*}
where $c_t$ is a constant. Then we use the inequality $\left\|\nabla^2 f(\wm,\x)\right\|_{\op}\leq \left\|\nabla^2 f(\wm,\x)\right\|_{F}$ to bound $\left\|\nabla^2 f(\wm,\x)\right\|_{\op}$. Next we only need to give the upper bound of $\left\|\nabla^2 f(\wm,\x)\right\|_{F}$. Let $\omega_{f_2}=r^{l-2}$. We first consider $\Qm_{st}\triangleq\nabla_{\wmi{s}}\left(\nabla_{\wmi{t}}f(\wm,\xm)\right)$.
By Lemma~\ref{uh2397}, if $s<t$, we have
\begin{equation*}\label{gradient1}
\begin{split}
\left\|\Qm_{st}\right\|_F^2= & \left\|\left(\Bmi{t-1:s+1}^T\right)\!\otimes\!\left(\Bmi{s-1:1}\xm\emm^T \Bmi{l:t+1}^T\right)\!+\!\left(\Bmi{s-1:1}\xm \xm^T \Bmi{t-1:1}^T\right)\!\otimes\!\left(\Bmi{l:s+1}^T \Bmi{l:t+1}\right)\right\|_F^2\\
\leq  & 2\left( \left\|\left(\Bmi{t-1:s+1}^T\right)\!\otimes\!\left(\Bmi{s-1:1}\xm\emm^T \Bmi{l:t+1}^T\right)\right\|_F^2\!+\!\left\|\left(\Bmi{s-1:1}\xm \xm^T \Bmi{t-1:1}^T\right)\!\otimes\!\left(\Bmi{l:s+1}^T \Bmi{l:t+1}\right)\right\|_F^2\right)\\
\leq  & 2\left\|\Bmi{t-1:s+1}\right\|_F^2 \left\|\Bmi{s-1:1}\right\|_F^2\left\|\xm\right\|_2^2\left\|\emm\right\|_2^2  \left\|\Bmi{l:t+1}\right\|_F^2\\
&\qquad +2\left\|\Bmi{s-1:1}\right\|_F^2\left\|\xm\right\|_2^2 \left\|\xm\right\|_2^2  \left\|\Bmi{t-1:1}\right\|_F^2\left\|\Bmi{l:s+1}\right\|_F^2\left\|\Bmi{l:t+1}\right\|_F^2\\
\led{172}  & 2\omega_{f_2}^2 r_x^2\left( r_x^2 \omega_f^2+r_x\omega_f\|\ym\|_2 +\left\|\ym\right\|_2^2\right)+2\omega_{f_1}^4 r_x^4,
\end{split}
\end{equation*}
where \ding{172} holds since we use $\left\|\Bmi{l:t+1}\right\|_F^2\left\|\Bmi{t-1:s+1}\right\|_F^2 \left\|\Bmi{s-1:1}\right\|_F^2\leq \omega_{f_2}^2$ and $\left\|\Bmi{s-1:1}\right\|_F^2 \left\|\Bmi{l:s+1}\right\|_F^2\\
\leq \omega_{f_1}^2$. Note that when $s\geq t$, the above inequality also holds. Similarly, consider the values in input data and the values in label, we have
\begin{equation*}\label{gradient1}
\begin{split}
\left\|\Qm_{st}\right\|_F^2 \leq c_{t'}r_x^4 r^{4l-2}\triangleq\alpha_l,
\end{split}
\end{equation*}
where $c_{t'}$ is a constant. Therefore, we can bound
\begin{equation*}\label{gradient1}
\left\|\nabla^2 f(\wm,\x)\right\|_{\op}\leq \left\|\nabla^2 f(\wm,\x)\right\|_{F}\leq \sqrt{\sum_{s=1}^{l}\sum_{t=1}^{l}\|\Qm_{st}\|_F^2}\leq l\sqrt{\alpha_l}.
\end{equation*}

On the other hand, if the activation functions are linear functions, $f(\wm,\xm)$ is  fourth order differentiable when $l\geq 2$. This means that $\nabla_{\xm}\nabla^3_{\wm}f(\wm,\xm)$ exists. Also since for any input $\xm\in\Bi{\dm_0}{r_x}$ and $\wm\in\Omega$, we can always find a universal constant $\alpha_p$ such that
 $$\|\nabla^3_{\wm}f(\wm,\xm)\|_{\op} = \sup_{\|\lam\|_2\leq 1}\lr\lam^{\otimes^3}, \nabla^3_{\wm}f(\wm,\xm)\rl=\sum_{i,j,k}[\nabla^3_{\wm}f(\wm,\xm)]_{ijk} \lam_i\lam_j\lam_k\leq \alpha_p<+\infty.$$
 We complete the proofs.
\end{proof}

\subsubsection{Proof of Lemma~\ref{thm:uniformconvergence321}}
\begin{proof}[\hypertarget{lemmaunihessian}{Proof}]
Recall that the weight of each layer has magnitude bound separately, \textit{i.e.} $\|\wmi{j}\|_2\leq r$. So here we separately assume $\wm_{\epsilon}^j=\{\wm_1^j,\cdots,\wm_{\nee^j}^j\}$ is the $\epsilon/l$-covering net of the ball $\Bi{\dm_j\dm_{j-1}}{r}$ which corresponds to the weight $\wmi{j}$ of the $j$-th layer. Let $\nee^j$ be the $\epsilon/l$-covering number. By $\epsilon$-covering theory in \cite{VRMT}, we can have $\nee^j\leq (3rl/\epsilon)^{\dm_j\dm_{j-1}}$. Let $\wm\in\Omega$ be an arbitrary vector. Since $\wm=[\wmi{1},\cdots,\wmi{l}]$ where $\wmi{j}$ is the weight of the $j$-th layer, we can always find a vector $\wm^j_{k_j}$ in $\wm_\epsilon^j$ such that $\|\wmi{j}-\wm^j_{k_j}\|_2\leq \epsilon/l$. For brevity, let $j_w\in[\nee^j]$  denote the index of $\wm^j_{k_j}$ in $\epsilon$-net $\wm_\epsilon^j$. Then let $\wm_{\kt}=[\wm^j_{k_1};\cdots;\wm^j_{k_j};\cdots;\wm^j_{k_l}]$. This means that we can always find a vector $\wm_{\kt}$ such that $\|\wm-\wm_{\kt}\|_2\leq \epsilon$. Now we use the decomposition strategy to bound our goal:
\begin{equation*}
\begin{aligned}
&\lv\nabla^2 \Jhn(\wm)-\nabla^2\Jm(\wm)\rv_{\op}\\
=& \lv \frac{1}{n}\sum_{i=1}^n \nabla^2 f(\wm,\xmi{i})-\EE(\nabla^2 f(\wm,\xm))\rv_{\op}\\
=&\Bigg\| \frac{1}{n}\sum_{i=1}^n \left(\nabla^2 f(\wm,\xmi{i})-\nabla f(\wm_{\kt},\xmi{i})\right)+\frac{1}{n}\sum_{i=1}^n \nabla^2 f(\wm_{\kt},\xmi{i})-\EE(\nabla^2 f(\wm_{\kt},\xm)) \\
& \ \ +\EE(\nabla^2 f(\wm_{\kt},\xm))-\EE(\nabla^2 f(\wm,\xm))\Bigg\|_{\op}\\
\leq &\lv \frac{1}{n}\sum_{i=1}^n \left(\nabla^2 f(\wm,\xmi{i})-\nabla^2 f(\wm_{\kt},\xmi{i})\right)\rv_{\op}+\lv\frac{1}{n}\sum_{i=1}^n \nabla^2 f(\wm_{\kt},\xmi{i})-\EE(\nabla^2 f(\wm_{\kt},\xm))\rv_{\op}\\
&\ \ +\Bigg\|\EE(\nabla^2 f(\wm_{\kt},\xm))-\EE(\nabla^2 f(\wm,\xm))\Bigg\|_{\op}.
\end{aligned}
\end{equation*}
Here we also define four events $\Em_0$, $\Em_1$, $\Em_2$ and $\Em_3$ as
\begin{equation*}
\begin{aligned}
&\Em_0=\left\{\sup_{\wm\in\Omega}\lv\nabla^2\Jhn(\wm)-\nabla^2\Jm(\wm)\rv_{\op}\geq t\right\},\\
&\Em_1=\left\{\sup_{\wm\in\Omega}\lv \frac{1}{n}\sum_{i=1}^n \left(\nabla^2 f(\wm,\xmi{i})-\nabla^2 f(\wm_{\kt},\xmi{i})\right)\rv_{\op}\geq \frac{t}{3}\right\},\\
&\Em_2=\left\{\sup_{j_w\in[\nee^j], j=[l]}\lv\frac{1}{n}\sum_{i=1}^n \nabla^2 f(\wm_{\kt},\xmi{i})-\EE(\nabla^2 f(\wm_{\kt},\xm))\rv_{\op}\geq \frac{t}{3}\right\},\\
&\Em_3=\left\{\sup_{\wm\in\Omega}\lv\EE(\nabla^2 f(\wm_{\kt},\xm))-\EE(\nabla^2 f(\wm,\xm))\rv_{\op}\geq \frac{t}{3}\right\}.
\end{aligned}
\end{equation*}

Accordingly, we have
\begin{equation*}
\begin{aligned}
\Pro\left(\Em_0\right)\leq \Pro\left(\Em_1\right)+\Pro\left(\Em_2\right)+\Pro\left(\Em_3\right).
\end{aligned}
\end{equation*}
So we can respectively bound $\Pro\left(\Em_1\right)$, $\Pro\left(\Em_2\right)$ and $\Pro\left(\Em_3\right)$ to bound $\Pro\left(\Em_0\right)$.

{\bf\noindent{Step 1. Bound $\Pro\left(\Em_1\right)$}}: We first bound $\Pro\left(\Em_1\right)$ as follows:
\begin{equation*}
\begin{aligned}
\Pro\left( \Em_1 \right) = & \Pro \left( \sup_{\wm \in \Omega} \lv \frac{1}{n}\sum_{i=1}^n \left(\nabla^2 f(\wm,\xmi{i})-\nabla^2 f(\wm_{\kt},\xmi{i})\right)\rv_2\geq \frac{t}{3} \right) \\
\overset{\text{\ding{172}}}{\leq} &\frac{3}{t}\EE \left(  \sup_{\wm \in \Omega}\lv \frac{1}{n}\sum_{i=1}^n\left(\nabla^2 f(\wm,\xmi{i})-\nabla^2 f(\wm_{\kt},\xmi{i})\right)\rv_2\right) \\
\leq &\frac{3}{t}\EE \left(  \sup_{\wm \in \Omega}\lv \nabla^2 f(\wm,\xm)-\nabla^2 f(\wm_{\kt},\xm)\rv_2\right) \\
\leq &\frac{3}{t}\EE \left(  \sup_{\wm \in \Omega}\frac{\lvs \frac{1}{n}\sum_{i=1}^n\left( \nabla^2 f(\wm,\xmi{i})- \nabla^2 f(\wm_{\kt},\xmi{i})\right)\rvs}{\lv\wm-\wm_{\kt}\rv_2} \sup_{\wm\in\Omega} \lv\wm-\wm_{\kt}\rv_2\right) \\
\overset{\text{\ding{173}}}{\leq} &\frac{3\alpha_p\epsilon}{t},
\end{aligned}
\end{equation*}
where \ding{172} holds since by Markov inequality and \ding{173} holds because of Lemma~\ref{thmgradisagf5675ent12}.

Therefore, we can set
\begin{equation*}
t \geq \frac{6\alpha_p \epsilon}{\varepsilon}.
\end{equation*}
Then we can bound $\Pro(\Em_1)$:
\begin{equation*}
\Pro(\Em_1)\leq \frac{\varepsilon}{2}.
\end{equation*}

{\bf\noindent{Step 2. Bound $\Pro\left(\Em_2\right)$}}: By Lemma~\ref{lem:opnorm}, we know that for any matrix $\Xm\in\Rss{d}{d}$, its operator norm can be computed as
\begin{equation*}
\|\Xm\|_{\op} \leq \frac{1}{1-2\epsilon} \sup_{\lam \in \lam_\epsilon} \left|\lr \lam,\Xm\lam\rl\right|.
\end{equation*}
where $\lam_\epsilon=\{\lam_1,\dots,\lam_{\kt}\}$ be an $\epsilon$-covering net of $\Bi{d}{1}$.

Let $\lam_{1/4}$ be the $\frac{1}{4}$-covering net of $\Bi{d}{1}$.  Recall that we use $j_w$ to denote the index of $\wm^j_{k_j}$ in $\epsilon$-net $\wm_\epsilon^j$ and we have $j_w\in[\nee^j],\ (\nee^j\leq (3rl/\epsilon)^{\dm_j\dm_{j-1}})$. Then we can bound $\Pro\left(\Em_2\right)$ as follows:
\begin{equation*}
\begin{aligned}
\Pro\left( \Em_2 \right) = & \Pro \left( \sup_{j_w \in [n_\epsilon^j]\,j\in[l]} \lv\frac{1}{n}\sum_{i=1}^n \nabla^2 f(\wm_{\kt},\xmi{i})-\EE(\nabla^2 f(\wm_{\kt},\xm))\rv_2\geq \frac{t}{3} \right) \\
\leq & \Pro \left( \sup_{j_w \in [n_\epsilon^j]\,j\in[l], \lam\in\lam_{1/4}} 2\left|\lr\lam, \left(\frac{1}{n}\sum_{i=1}^n \nabla^2 f(\wm_{\kt},\xmi{i})-\EE\left(\nabla^2 f(\wm_{\kt},\xm)\right)\right)\lam\rl\right|\geq \frac{t}{3} \right) \\
\leq & 12^d\left(\frac{3lr}{\epsilon}\right)^{\sum_j\dm_j\dm_{j-1}} \!\!\!\!\!\!\!\!\!\! \!\!\!\!\!\sup_{j_w \in [n_\epsilon^j]\,j\in[l], \lam\in\lam_{1/4}} \!\!\!\! \Pro\! \left(\left| \frac{1}{n} \sum_{i=1}^n\!\!\lr\lam, \left(\nabla^2 f(\wm_{\kt},\xmi{i})\!-\!\EE\left(\nabla^2 f(\wm_{\kt},\xm)\right)\right)\lam\rl\right| \!\geq\! \frac{t}{6} \right)\\
\led{172} & 12^d\left(\frac{3lr}{\epsilon}\right)^d 10\exp\left(-c_{h'}n\min\left(\frac{t^2}{36\tau^2l^2\max\left(\omega_g ,\omega_g \tau^2, \omega_h \right)},
\frac{t}{6\sqrt{\omega_g} l\max\left(\tau,\tau^2\right)}\right)\right),
\end{aligned}
\end{equation*}
where \ding{172} holds since by Lemma~\ref{thm: hessian}, we have
\begin{equation*}\label{gradient14211}
\begin{split}
&\Pro\left(\left|\frac{1}{n}\sum_{i=1}^n \left(\lr \lam,(\nabla_{\wm}^2f(\wm,\xm)-\EE \nabla_{\wm}^2f(\wm,\xm))\lam\rl\right)\right|>t\right)\\
&  \qquad\quad\qquad \qquad\qquad\leq 10\exp\left(-c_{h'}n\min\left(\frac{t^2}{ \tau^2l^2\max\left(\omega_g ,\omega_g \tau^2, \omega_h \right)},
\frac{t}{\sqrt{\omega_g} l\max\left(\tau,\tau^2\right)}\right)\right),
\end{split}
\end{equation*}
where $\omega_g=\left(\max_j(\dm_j\dm_{j-1})\right)^2 r^{4(l-1)}$ and $\omega_h=\left(\max_j(\dm_j\dm_{j-1})\right)^2 r^{2(l-2)}$.

Let $d_\epsilon=d\log(36lr/\epsilon)\!+ \!\log(20/\varepsilon)$. Thus, if we set
\begin{equation*}
\begin{aligned}
t\geq \max\left(\sqrt{\frac{36\tau^2l^2\max\left(\omega_g ,\omega_g \tau^2, \omega_h \right)d_\epsilon}{c_{h'}n}},  \frac{6\sqrt{\omega_g} l\max\!\left(\tau,\tau^2\right)d_\epsilon}{c_{h'}n}\!\right)\!,
\end{aligned}
\end{equation*}
then we have
\begin{equation*}
\Pro\left(\Em_2\right)\leq \frac{\varepsilon}{2}.
\end{equation*}

{\bf\noindent{Step 3. Bound $\Pro\left(\Em_3\right)$}}:
We first bound $\Pro\left(\Em_3\right)$ as follows:
\begin{equation*}
\begin{aligned}
\Pro\left( \Em_3 \right) = & \Pro \left( \sup_{\wm\in\Omega} \lv\EE(\nabla^2 f(\wm_{\kt},\xm))-\EE(\nabla^2 f(\wm,\xm))\rv_2 \geq \frac{t}{3} \right) \\
\leq & \Pro \left(\EE \sup_{\wm\in\Omega} \lv(\nabla^2 f(\wm_{\kt},\xm)- \nabla^2 f(\wm,\xm)\rv_2 \geq \frac{t}{3} \right) \\
\leq & \Pro \left(  \sup_{\wm \in \Omega}\frac{\lvs \frac{1}{n}\sum_{i=1}^n\left( \nabla^2 f(\wm,\xmi{i})- \nabla^2 f(\wm_{\kt},\xmi{i})\right)\rvs}{\lv\wm-\wm_{\kt}\rv_2} \sup_{\wm\in\Omega} \lv\wm-\wm_{\kt}\rv_2  \geq \frac{t}{3}\right) \\
\overset{\text{\ding{172}}}{\leq}&  \Pro\left(\alpha_p\epsilon  \geq \frac{t}{3}\right),
\end{aligned}
\end{equation*}
where \ding{172} holds because of Lemma~\ref{thmgradisagf5675ent12}. We set $\epsilon$ enough small such that $\alpha_p  \epsilon  < t/3$ always holds. Then it yields $\Pro\left( \Em_3 \right)=0$.

{\bf\noindent{Step 4. Final result}}: For brevity, let $\omega_2=36\tau^2l^2\max\left(\omega_g ,\omega_g \tau^2, \omega_h \right)$ and $\omega_3=6\sqrt{\omega_g} l\max\!\left(\tau,\tau^2\right)$. To ensure $\Pro(\Em_0)\leq \varepsilon$, we just set $\epsilon=36r/n$ and
\begin{equation*}
\begin{split}
t&\geq \max\left(\frac{6\alpha_p \epsilon}{\varepsilon},\ 3\alpha_p\epsilon,\ \sqrt{\frac{\omega_2 (d\log(36lr/\epsilon)\!+ \!\log(20/\varepsilon))}{c_{h'}n}},  \frac{\omega_3(d\log(36lr/\epsilon)\!+\!\log(20/\varepsilon))}{c_{h'}n}\right)\\
&=\max\left(\frac{216\alpha_pr}{ n\varepsilon},\ \sqrt{\frac{\omega_2 (d\log(nl)\!+ \!\log(20/\varepsilon))}{c_{h'}n}},  \frac{\omega_3(d\log(nl)\!+\!\log(20/\varepsilon))}{c_{h'}n}\right).
\end{split}
\end{equation*}
Thus, if $n\geq c_{h''}\max(\frac{\alpha_p^2r^2}{\tau^2 l^2\omega_h^2\varepsilon^2(\max_j(\dm_j\dm_{j-1}))^2d\log(l)}, d\log(l))$ where $c_{h''}$ is a constant,  there exists a universal constant $c_h$ such that
\begin{equation*}
\sup_{\wm\in\Omega}\left\| \nabla^2\Jhn(\wm)\!-\!\nabla^2\Jm(\wm)\right\|_{\op}\!\leq\! c_h \tau l \omega_h\max_j(\dm_j\dm_{j-1})  \sqrt{\!\frac{d\log(nl)\!+\!\log(20/\varepsilon)}{n}}
\end{equation*}
holds with probability at least $1-\varepsilon$, where $\omega_h=\max\!\left(\tau r^{2(l-1)},r^{2(l-2)},r^{l-2}\right)$.
The proof is completed.
\end{proof}

\subsection{Proofs of Main Theorems}\label{mainproof}
\subsubsection{Proof of Theorem~\ref{thm:stability}}
\begin{proof}[\hypertarget{lemmauniloss}{Proof}]
Recall that  the weight of each layer has magnitude bound separately, \textit{i.e.} $\|\wmi{j}\|_2\leq r$. So here we separately assume $\wm_{\epsilon}^j=\{\wm_1^j,\cdots,\wm_{\nee^j}^j\}$ is the $\epsilon/l$-covering net of the ball $\Bi{\dm_j\dm_{j-1}}{r}$ which corresponds to the weight $\wmi{j}$ of the $j$-th layer. Let $\nee^j$ be the $\epsilon/l$-covering number. By $\epsilon$-covering theory in \cite{VRMT}, we can have $\nee^j\leq (3rl/\epsilon)^{\dm_j\dm_{j-1}}$. Let $\wm\in\Omega$ be an arbitrary vector. Since $\wm=[\wmi{1},\cdots,\wmi{l}]$ where $\wmi{j}$ is the weight of the $j$-th layer, we can always find a vector $\wm^j_{k_j}$ in $\wm_\epsilon^j$ such that $\|\wmi{j}-\wm^j_{k_j}\|_2\leq \epsilon/l$. For brevity, let $j_w\in[\nee^j]$  denote the index of $\wm^j_{k_j}$ in $\epsilon$-net $\wm_\epsilon^j$. Then let $\wm_{\kt}=[\wm^j_{k_1};\cdots;\wm^j_{k_j};\cdots;\wm^j_{k_l}]$. This means that we can always find a vector $\wm_{\kt}$ such that $\|\wm-\wm_{\kt}\|_2\leq \epsilon$. Now we use the decomposition strategy to bound our goal:
\begin{equation*}
\begin{aligned}
&\lvs\Jhn(\wm)-\Jm(\wm)\rvs\! =\! \lvs \frac{1}{n}\sum_{i=1}^n f(\wm,\xmi{i})-\EE(f(\wm,\xm))\rvs \\
=&\Bigg| \frac{1}{n}\!\sum_{i=1}^n\!\! \left(f(\wm,\xmi{i})\!-\!f(\wm_{\kt},\xmi{i})\right)\!+\!\frac{1}{n}\! \sum_{i=1}^n \!\! f(\wm_{\kt},\xmi{i})\!-\!\EE f(\wm_{\kt},\xm)\!+\! \EE f(\wm_{\kt},\xm)\!-\!\EE f(\wm,\xm)\Bigg|\\
\leq &\!\lvs\frac{1}{n}\!\sum_{i=1}^n \!\! \left(f(\wm,\xmi{i})\!-\!f(\wm_{\kt},\xmi{i})\right)\rvs\!+ \!\lvs\frac{1}{n}\!\!\sum_{i=1}^n \!\! f(\wm_{\kt},\xmi{i})\!-\!\EE f(\wm_{\kt},\xm)\rvs\!+\! \Bigg|\EE f(\wm_{\kt},\xm)\!-\!\EE f(\wm,\xm)\Bigg|.
\end{aligned}
\end{equation*}
Then, we define four events $\Em_0$, $\Em_1$, $\Em_2$ and $\Em_3$ as
\begin{equation*}
\begin{aligned}
&\Em_0=\left\{\sup_{\wm\in\Omega}\lvs\Jhn(\wm)-\Jm(\wm)\rvs\geq t\right\},\\
& \Em_1=\left\{\sup_{\wm\in\Omega}\lvs \frac{1}{n}\sum_{i=1}^n \left(f(\wm,\xmi{i})-f(\wm_{\kt},\xmi{i})\right)\rvs\geq \frac{t}{3}\right\},\\
&\Em_2=\left\{\sup_{j_w\in[\nee^j], j=[l]}\lvs\frac{1}{n}\!\sum_{i=1}^n\!\! f(\wm_{\kt},\xmi{i})\!-\!\EE(f(\wm_{\kt},\xm))\rvs\!\geq \! \frac{t}{3}\right\},\\ &\Em_3=\left\{\sup_{\wm\in\Omega}\Bigg|\EE(f(\wm_{\kt},\xm))\!-\!\EE(f(\wm,\xm))\Bigg|\!\geq \! \frac{t}{3}\right\}.
\end{aligned}
\end{equation*}
Accordingly, we have
\begin{equation*}
\begin{aligned}
\Pro\left(\Em_0\right)\leq \Pro\left(\Em_1\right)+\Pro\left(\Em_2\right)+\Pro\left(\Em_3\right).
\end{aligned}
\end{equation*}
So we can respectively bound $\Pro\left(\Em_1\right)$, $\Pro\left(\Em_2\right)$ and $\Pro\left(\Em_3\right)$ to bound $\Pro\left(\Em_0\right)$.

{\bf\noindent{Step 1. Bound $\Pro\left(\Em_1\right)$}}: We first bound $\Pro\left(\Em_1\right)$ as follows:
\begin{equation*}
\begin{aligned}
\Pro\left( \Em_1 \right) = & \Pro \left( \sup_{\wm \in \Omega} \lvs \frac{1}{n}\sum_{i=1}^n \left( f(\wm,\xmi{i})- f(\wm_{\kt},\xmi{i})\right)\rvs\geq \frac{t}{3} \right) \\
\overset{\text{\ding{172}}}{\leq} &\frac{3}{t}\EE \left(  \sup_{\wm \in \Omega}\lvs\frac{1}{n}\sum_{i=1}^n\left(f(\wm,\xmi{i})- f(\wm_{\kt},\xmi{i})\right)\rvs\right) \\
\leq &\frac{3}{t}\EE \left(  \sup_{\wm \in \Omega}\frac{\lvs \frac{1}{n}\sum_{i=1}^n\left( f(\wm,\xmi{i})- f(\wm_{\kt},\xmi{i})\right)\rvs}{\lv\wm-\wm_{\kt}\rv_2} \sup_{\wm\in\Omega} \lv\wm-\wm_{\kt}\rv_2\right) \\
\leq &\frac{3\epsilon}{t}\EE \left(  \sup_{\wm \in \Omega} \lv \nabla \Jhn(\wm,\xm) \rv_2\right),
\end{aligned}
\end{equation*}
where \ding{172} holds since by Markov inequality, we have that for an arbitrary nonnegative random variable $x$, then
\begin{equation*}
\begin{aligned}
\Pro(x\geq t)\leq \frac{\EE(x)}{t}.
\end{aligned}
\end{equation*}

Now we only need to bound $\EE \left(  \sup_{\wm \in \Omega} \lv \nabla \Jhn(\wm,\xm) \rv_2\right)$. Therefore, by Lemma~\ref{thmgradisagf5675ent12}, we have
\begin{equation*}
\begin{aligned}
\EE\! \left(\!\sup_{\wm \in \Omega} \lv \nabla \Jhn(\wm,\xm) \rv_2\!\right)\!=\!  \EE\! \left(\!\sup_{\wm \in \Omega}\! \lv\frac{1}{n}\sum_{i=1}^{n}\!\! \nabla f(\wm,\xmi{i})  \rv_2 \!\right)
\!=\! \EE\! \left(\!\sup_{\wm \in \Omega} \!\!\lv\nabla f(\wm,\xm)\rv_2 \right)
\!\leq \! \sqrt{\alpha_g}.
\end{aligned}
\end{equation*}
where $\alpha_g= c_tlr_x^4 r^{4l-2}.$ Therefore, we have
\begin{equation*}
\begin{aligned}
\Pro\left( \Em_1 \right) \leq \frac{3 \epsilon\sqrt{\alpha_g}}{t}.
\end{aligned}
\end{equation*}
 We further let
\begin{equation*}
t \geq \frac{6\epsilon\sqrt{\alpha_g}}{\varepsilon}.
\end{equation*}
Then we can bound $\Pro(\Em_1)$:
\begin{equation*}
\Pro(\Em_1)\leq \frac{\varepsilon}{2}.
\end{equation*}

{\bf\noindent{Step 2. Bound $\Pro\left(\Em_2\right)$}}: Recall that we use $j_w$ to denote the index of $\wm^j_{k_j}$ in $\epsilon$-net $\wm_\epsilon^j$ and we have $j_w\in[\nee^j],\ (\nee^j\leq (3rl/\epsilon)^{\dm_j\dm_{j-1}})$.  We can bound $\Pro\left(\Em_2\right)$ as follows:
\begin{equation*}
\begin{aligned}
\Pro\left( \Em_2 \right) = & \Pro \left( \sup_{j_w \in [n_\epsilon^j]\,j\in[l]} \lvs\frac{1}{n}\sum_{i=1}^n f(\wm_{\kt},\xmi{i})-\EE(f(\wm_{\kt},\xm))\rvs\geq \frac{t}{3} \right) \\
\leq & \left(\frac{3lr}{\epsilon}\right)^{\sum_j \dm_j\dm_{j-1}} \sup_{j_w \in [n_\epsilon^j]\,j\in[l]} \Pro \left(  \lvs\frac{1}{n}\sum_{i=1}^n f(\wm_{\kt},\xmi{i})-\EE(f(\wm_{\kt},\xm))\rvs\geq \frac{t}{3}\right) \\
\led{172} &  4 \left(\frac{3lr}{\epsilon}\right)^d\exp\left(-c_{f'}n\min \left(\frac{t^2}{9\omega_{f}^2 \max\left(\dm_l\omega_{f}^2\tau^4,\tau^2\right)}, \frac{t}{3\omega_{f}^2\tau^2}\right)\!\right),
\end{aligned}
\end{equation*}
where \ding{172} holds because in Lemma~\ref{thmsafobjeasdfctive}, we have
\begin{align*}
\Pro\!\left(\!\frac{1}{n}\!\sum_{i=1}^n\!\left( f(\wm,\xmi{i})\!-\!\EE(f(\wm,\xmi{i}))\right)\!>\!t\!\right)\leq 2\exp\!\left(\!-c_{f'}n\min\!\left(\!\frac{t^2}{\omega_{f}^2 \max\left(\dm_l\omega_{f}^2\tau^4,\tau^2\right)}, \frac{t}{\omega_{f}^2\tau^2}\!\right)\!\right),
\end{align*}
where $c_{f'}$ is a positive constant and $\omega_{f}=r^{l}$.
Thus, if we set
\begin{equation*}
\begin{aligned}
t\geq\max \left(\sqrt{\frac{9\omega_{f}^2(d\log(3rl/\epsilon)+\log(8/\varepsilon)) \max\left(\dm_l\omega_{f}^2\tau^4,\tau^2\right)}{c_{f'}n}}, \ \frac{3\omega_f^2\tau^2(d\log(3rl/\epsilon)+\log(8/\varepsilon))}{c_{f'}n}\right),
\end{aligned}
\end{equation*}
then we have
\begin{equation*}
\Pro\left(\Em_2\right)\leq \frac{\varepsilon}{2}.
\end{equation*}

{\bf\noindent{Step 3. Bound $\Pro\left(\Em_3\right)$}}:
We first bound $\Pro\left(\Em_3\right)$ as follows:
\begin{equation*}
\begin{aligned}
\Pro\left( \Em_3 \right) = & \Pro \left( \sup_{\wm\in\Omega} \lv\EE(f(\wm_{\kt},\xm))-\EE(f(\wm,\xm))\rv_2 \geq \frac{t}{3} \right) \\
=&  \Pro\left(\sup_{\wm\in\Omega} \frac{\lv \EE\left( f(\wm_{\kt},\xm)-f(\wm,\xm) \rv_2\right)}{\lv\wm-\wm_{\kt}\rv_2} \sup_{\wm\in\Omega} \lv\wm-\wm_{\kt}\rv_2 \geq \frac{t}{3}\right) \\
\leq&  \Pro\left(\epsilon\EE \sup_{\wm\in\Omega} \lv \nabla\Jm_{\wm}(\wm,\xm) \rv_2  \geq \frac{t}{3}\right) \\
\overset{\text{\ding{172}}}{\leq}&  \Pro\left(\sqrt{\alpha_g} \epsilon  \geq \frac{t}{3}\right),
\end{aligned}
\end{equation*}
where \ding{172} holds since we utilize Lemma~\ref{thmgradisagf5675ent12}. We set $\epsilon$ enough small such that $\sqrt{\alpha_g}\epsilon  < t/3$ always holds. Then it yields $\Pro\left( \Em_3 \right)=0$.

{\bf\noindent{Step 4. Final result}}: To ensure $\Pro(\Em_0)\leq \varepsilon$, we just set $\epsilon=3r/n$. Note that $\frac{6\sqrt{\alpha_g} \epsilon}{\varepsilon}> 3\sqrt{\alpha_g}\epsilon$. Thus we can obtain
\begin{equation*}
\begin{split}
t&\!\geq \!\max\!\left(\!\!\frac{6\sqrt{\alpha_g} \epsilon}{\varepsilon}, \sqrt{\!\!\frac{9\omega_{f}^2(d\log(3rl/\epsilon)\!+\!\log(8/\varepsilon)) \max\!\left(\!\dm_l\omega_{f}^2\tau^4,\tau^2\!\right)}{c_{f'}n}}, \frac{3\omega_f^2\tau^2(d\log(3rl/\epsilon)\!+\!\log(8/\varepsilon))}{c_{f'}n} \!\!\right)\\
&=\!\max\!\left(\!\!\frac{18\sqrt{\alpha_g} r}{n\varepsilon},
\sqrt{\frac{9\omega_{f}^2(d\log(ln)\!+\!\log(8/\varepsilon)) \max\!\left(\!\dm_l\omega_{f}^2\tau^4,\tau^2\!\right)}{c_{f'}n}}, \frac{3\omega_f^2\tau^2(d\log(ln)\!+\!\log(8/\varepsilon))}{c_{f'}n} \!\!\right).
\end{split}
\end{equation*}
Note that we have $\alpha_g= c_tlr_x^4 r^{4l-2}$ where $c_t$ is a constant. Then if $n\geq c_{f''} \max(\frac{lr_x^4}{\dm_l d\varepsilon^2\tau^4\log(l)},$ $ d\log(l)/\dm_l)$ where $c_{f''}$ is a constant, there exists such a universal constant $c_f$ such that
\begin{equation*}
\sup_{\wm\in\Omega} \left\|  \Jhn(\wm)- \Jm(\wm)\right\|_2\leq c_f \omega_{f}\tau \max\left(\sqrt{\dm_l}\omega_{f}\tau,1\right) \sqrt{\frac{d\log(nl)+\log(8/\varepsilon)}{n}}
\end{equation*}
holds with probability at least $1-\varepsilon$, where $\omega_{f}=r^{l}$.
\end{proof}

\subsubsection{Proof of Corollary~\ref{stability and generalization}}
\begin{proof}[\hypertarget{cor1}{Proof}]
By Lemma~\ref{stab}, we know $\epsilon_s=\epsilon_g$. Thus, the remaining work is to bound $\epsilon_s$. Actually, we can have
\begin{equation*}
\begin{split}
\left|\EE_{\Ss\sim\D,\Am,(\xmi{1}',\cdots,\xmi{n}'\!)\sim\D}\frac{1}{n}\! \sum_{j=1}^{n}\!\!\left(\!f_j(\wms{j},\!\xmi{j}'\!)\!-\!\!f_j(\wmn,\xmi{j}'\!)\! \right)\!\right| \leq &\EE_{\Ss\sim\D}\left( \sup_{\wm\in\Omega} \left|  \Jhn(\wm)- \Jm(\wm)\right|\right)\\
\leq &\sup_{\wm\in\Omega} \left|  \Jhn(\wm)- \Jm(\wm)\right|\\
\leq& \epsilon_l.
\end{split}
\end{equation*}
Thus, we have $\epsilon_g=\epsilon_s\leq \epsilon_l$. The proof is completed.
\end{proof}

\subsubsection{ Proof of Theorem~\ref{thm:uniformconvergence1}}
\begin{proof}[\hypertarget{lemmaunigradient}{Proof}]
We adopt similar strategy in proofs of Theorem~\ref{thm:stability}. Recall that  the weight of each layer has magnitude bound separately, \textit{i.e.} $\|\wmi{j}\|_2\leq r$. So here we separately assume $\wm_{\epsilon}^j=\{\wm_1^j,\cdots,\wm_{\nee^j}^j\}$ is the $\epsilon/l$-covering net of the ball $\Bi{\dm_j\dm_{j-1}}{r}$ which corresponds to the weight $\wmi{j}$ of the $j$-th layer. Let $\nee^j$ be the $\epsilon/l$-covering number. By $\epsilon$-covering theory in \cite{VRMT}, we can have $\nee^j\leq (3rl/\epsilon)^{\dm_j\dm_{j-1}}$. Let $\wm\in\Omega$ be an arbitrary vector. Since $\wm=[\wmi{1},\cdots,\wmi{l}]$ where $\wmi{j}$ is the weight of the $j$-th layer, we can always find a vector $\wm^j_{k_j}$ in $\wm_\epsilon^j$ such that $\|\wmi{j}-\wm^j_{k_j}\|_2\leq \epsilon/l$. For brevity, let $j_w\in[\nee^j]$  denote the index of $\wm^j_{k_j}$ in $\epsilon$-net $\wm_\epsilon^j$. Then let $\wm_{\kt}=[\wm^j_{k_1};\cdots;\wm^j_{k_j};\cdots;\wm^j_{k_l}]$. Then we can always find a vector $\wm_{\kt}$ such that $\|\wm-\wm_{\kt}\|_2\leq \epsilon$.
Accordingly, we can decompose $\lv\nabla \Jhn(\wm)-\nabla\Jm(\wm)\rv_2$ as
\begin{equation*}
\begin{aligned}
&\lv\nabla \Jhn(\wm)-\nabla\Jm(\wm)\rv_2\\
=& \lv \frac{1}{n}\sum_{i=1}^n \nabla f(\wm,\xmi{i})-\EE(\nabla f(\wm,\xm))\rv_2\\
=&\Bigg\| \frac{1}{n}\sum_{i=1}^n \left(\nabla f(\wm,\xmi{i})-\nabla f(\wm_{\kt},\xmi{i})\right)+\frac{1}{n}\sum_{i=1}^n \nabla f(\wm_{\kt},\xmi{i})-\EE(\nabla f(\wm_{\kt},\xm)) \\
& \ \ +\EE(\nabla f(\wm_{\kt},\xm))-\EE(\nabla f(\wm,\xm))\Bigg\|_2\\
\leq &\lv \frac{1}{n}\sum_{i=1}^n \left(\nabla f(\wm,\xmi{i})-\nabla f(\wm_{\kt},\xmi{i})\right)\rv_2+\lv\frac{1}{n}\sum_{i=1}^n \nabla f(\wm_{\kt},\xmi{i})-\EE(\nabla f(\wm_{\kt},\xm))\rv_2\\
&\ \ +\Bigg\|\EE(\nabla f(\wm_{\kt},\xm))-\EE(\nabla f(\wm,\xm))\Bigg\|_2.
\end{aligned}
\end{equation*}
Here we also define four events $\Em_0$, $\Em_1$, $\Em_2$ and $\Em_3$ as
\begin{equation*}
\begin{aligned}
&\Em_0=\left\{\sup_{\wm\in\Omega}\lv\nabla\Jhn(\wm)-\nabla\Jm(\wm)\rv_2\geq t\right\},\\
&\Em_1=\left\{\sup_{\wm\in\Omega}\lv \frac{1}{n}\sum_{i=1}^n \left(\nabla f(\wm,\xmi{i})-\nabla f(\wm_{\kt},\xmi{i})\right)\rv_2\geq \frac{t}{3}\right\},\\
&\Em_2=\left\{\sup_{j_w\in[\nee^j], j=[l]}\lv\frac{1}{n}\sum_{i=1}^n \nabla f(\wm_{\kt},\xmi{i})-\EE(\nabla f(\wm_{\kt},\xm))\rv_2\geq \frac{t}{3}\right\},\\
&\Em_3=\left\{\sup_{\wm\in\Omega}\Bigg\|\EE(\nabla f(\wm_{\kt},\xm))-\EE(\nabla f(\wm,\xm))\Bigg\|_2\geq \frac{t}{3}\right\}.
\end{aligned}
\end{equation*}
Accordingly, we have
\begin{equation*}
\begin{aligned}
\Pro\left(\Em_0\right)\leq \Pro\left(\Em_1\right)+\Pro\left(\Em_2\right)+\Pro\left(\Em_3\right).
\end{aligned}
\end{equation*}
So we can respectively bound $\Pro\left(\Em_1\right)$, $\Pro\left(\Em_2\right)$ and $\Pro\left(\Em_3\right)$ to bound $\Pro\left(\Em_0\right)$.

{\bf\noindent{Step 1. Bound $\Pro\left(\Em_1\right)$}}: We first bound $\Pro\left(\Em_1\right)$ as follows:
\begin{equation*}
\begin{aligned}
\Pro\left( \Em_1 \right) = & \Pro \left( \sup_{\wm \in \Omega} \lv \frac{1}{n}\sum_{i=1}^n \left(\nabla f(\wm,\xmi{i})-\nabla f(\wm_{\kt},\xmi{i})\right)\rv_2\geq \frac{t}{3} \right) \\
\overset{\text{\ding{172}}}{\leq} &\frac{3}{t}\EE \left(  \sup_{\wm \in \Omega}\lv \frac{1}{n}\sum_{i=1}^n\left(\nabla f(\wm,\xmi{i})-\nabla f(\wm_{\kt},\xmi{i})\right)\rv_2\right) \\
\leq &\frac{3}{t}\EE \left(  \sup_{\wm \in \Omega}\frac{\lv \frac{1}{n}\sum_{i=1}^n\left(\nabla f(\wm,\xmi{i})-\nabla f(\wm_{\kt},\xmi{i})\right)\rv_2}{\lv\wm-\wm_{\kt}\rv_2} \sup_{\wm\in\Omega} \lv\wm-\wm_{\kt}\rv_2\right) \\
\leq &\frac{3\epsilon}{t}\EE \left(  \sup_{\wm \in \Omega} \lv \nabla^2 \Jhn(\wm,\xm) \rv_2\right),
\end{aligned}
\end{equation*}
where \ding{172} holds since by Markov inequality, we have that for an arbitrary nonnegative random variable $x$, then $\Pro(x\geq t)\leq \frac{\EE(x)}{t}$.

Now we only need to bound $\EE \left(  \sup_{\wm \in \Omega} \lv \nabla^2 \Jhn(\wm,\xm) \rv_2\right)$. Now we utilize Lemma~\ref{thmgradisagf5675ent12} to achieve this goal:
\begin{equation*}
\begin{aligned}
\EE \left(\sup_{\wm \in \Omega} \lv \nabla^2 \Jhn(\wm,\xm) \rv_2\right)\leq= \EE \left(\sup_{\wm \in \Omega} \lv\nabla^2 f(\wm,\xm) -\nabla^2 f(\wm^*,\xm)\rv_2 \right)
\leq  l\sqrt{\alpha_l}.
\end{aligned}
\end{equation*}
where $\alpha_l=c_{t'}r_x^4 r^{4l-2}$. Therefore, we have
\begin{equation*}
\begin{aligned}
\Pro\left( \Em_1 \right) \leq \frac{3l\sqrt{\alpha_l}\epsilon}{t}.
\end{aligned}
\end{equation*}
 We further let
\begin{equation*}
t \geq \frac{6l\sqrt{\alpha_l}\epsilon}{\varepsilon}.
\end{equation*}
Then we can bound $\Pro(\Em_1)$:
\begin{equation*}
\Pro(\Em_1)\leq \frac{\varepsilon}{2}.
\end{equation*}

{\bf\noindent{Step 2. Bound $\Pro\left(\Em_2\right)$}}: By Lemma~\ref{lem:2norm}, we know that for any vector $\xm\in\Rs{d}$, its $\ell_2$-norm can be computed as
\begin{equation*}
\|\xm\|_2 \leq \frac{1}{1-\epsilon} \sup_{\lam \in \lam_\epsilon} \lr \lam,\xm\rl.
\end{equation*}
where $\lam_\epsilon=\{\lam_1,\dots,\lam_{\kt}\}$ be an $\epsilon$-covering net of $\Bi{d}{1}$.

Let $\lam_{1/2}$ be the $\frac{1}{2}$-covering net of $\Bi{d}{1}$. Recall that we use $j_w$ to denote the index of $\wm^j_{k_j}$ in $\epsilon$-net $\wm_\epsilon^j$ and we have $j_w\in[\nee^j],\ (\nee^j\leq (3rl/\epsilon)^{\dm_j\dm_{j-1}})$.  Then we can bound $\Pro\left(\Em_2\right)$ as follows:
\begin{equation*}
\begin{aligned}
\Pro\left( \Em_2 \right) = & \Pro \left( \sup_{j_w\in[\nee^j], j=[l]} \lv\frac{1}{n}\sum_{i=1}^n \nabla f(\wm_{\kt},\xmi{i})-\EE(\nabla f(\wm_{\kt},\xm))\rv_2\geq \frac{t}{3} \right) \\
= & \Pro \left( \sup_{j_w\in[\nee^j], j=[l], \lam\in\lam_{1/2}} 2\lr\lam, \frac{1}{n}\sum_{i=1}^n \nabla f(\wm_{\kt},\xmi{i})-\EE\left(\nabla f(\wm_{\kt},\xm)\right)\rl\geq \frac{t}{3} \right) \\
\leq & 6^d\left(\frac{3lr}{\epsilon}\right)^{\sum_j\dm_j\dm_{j-1}}\!\!\!\!\!\!\! \!\!\! \!\!\! \sup_{j_w\in[\nee^j], j=[l], \lam\in\lam_{1/2}}\!\!\!  \Pro \left(\frac{1}{n} \sum_{i=1}^n\lr\lam, \nabla f(\wm_{\kt},\xmi{i})-\EE\left(\nabla f(\wm_{\kt},\xm)\right)\rl\geq \frac{t}{6} \right)\\
\led{172} & 6^d\left(\frac{3r}{\epsilon}\right)^d 6\exp\!\left(\!-c_{g'}n\min\left(\!\frac{t^2}{36l\max \left(\omega_g\tau^2,\omega_g\tau^4,\omega_{g'}\tau^2\right)},
\frac{t}{6\sqrt{l\omega_g}\max\left(\tau,\tau^2\right)}\!\right)\!\right),
\end{aligned}
\end{equation*}
where \ding{172} holds since by Lemma~\ref{thm:gradient12}, we have
\begin{equation*}\label{gradient14211}
\begin{split}
&\Pro\!\left(\!\frac{1}{n}\sum_{i=1}^n\left( \lr\lam,\nabla_{\wm}f(\wm,\xmi{i})-\!\EE\nabla_{\wm}f(\wm,\xmi{i})\rl\right) \!>\!t\right)\\
&\qquad \quad\qquad\qquad\qquad\leq 3\exp\left(-c_{g'}n\min\left(\frac{t^2}{ l \max\left(\omega_g\tau^2,\omega_g\tau^4,\omega_{g'}\tau^2\right)},
\frac{t}{\sqrt{l\omega_g}\max\left(\tau,\tau^2\right)}\right)\right),
\end{split}
\end{equation*}
where $c_{g'}$ is a constant; $\omega_g=\dm_0 r^{2(2l-1)}\max_j(\dm_j\dm_{j-1})$ and $\omega_{g'}= r^{2(l-1)}\max_j(\dm_j\dm_{j-1})$.

Let $\omega_2=36l\max \left(\omega_g\tau^2,\omega_g\tau^4,\omega_{g'}\tau^2\right)$ and $\omega_3=6\sqrt{l\omega_g} \max\!\left(\tau,\tau^2\right)$.  Thus, if we set
\begin{equation*}
\begin{aligned}
t\geq \max\left(\sqrt{\frac{\omega_2(d\log(18lr/\epsilon)\!+ \!\log(12/\varepsilon))}{c_{g'}n}}, \frac{\omega_3(d\log(18lr/\epsilon)\!+\!\log(12/\varepsilon))}{c_{g'}n}\right),
\end{aligned}
\end{equation*}
then we have
\begin{equation*}
\Pro\left(\Em_2\right)\leq \frac{\varepsilon}{2}.
\end{equation*}

{\bf\noindent{Step 3. Bound $\Pro\left(\Em_3\right)$}}:
We first bound $\Pro\left(\Em_3\right)$ as follows:
\begin{equation*}
\begin{aligned}
\Pro\left( \Em_3 \right) = & \Pro \left( \sup_{\wm\in\Omega} \lv\EE(f(\wm_{\kt},\xm))-\EE(f(\wm,\xm))\rv_2 \geq \frac{t}{3} \right) \\
=&  \Pro\left(\sup_{\wm\in\Omega} \frac{\lv \EE\left( f(\wm_{\kt},\xm)-f(\wm,\xm) \rv_2\right)}{\lv\wm-\wm_{\kt}\rv_2} \sup_{\wm\in\Omega} \lv\wm-\wm_{\kt}\rv_2 \geq \frac{t}{3}\right) \\
\leq&  \Pro\left(\epsilon\EE \sup_{\wm\in\Omega} \lv \nabla^2 \Jhn(\wm,\xm) \rv_2  \geq \frac{t}{3}\right) \\
\leq&  \Pro\left(l\sqrt{\alpha_l}\epsilon  \geq \frac{t}{3}\right).
\end{aligned}
\end{equation*}
We set $\epsilon$ enough small such that $l\sqrt{\alpha_l}\epsilon  < t/3$ always holds. Then it yields $\Pro\left( \Em_3 \right)=0$.

{\bf\noindent{Step 4. Final result}}: Finally, to ensure $\Pro(\Em_0)\leq \varepsilon$, we just set $\epsilon=18r/n$ and
\begin{equation*}
\begin{split}
t&\geq \max\left(\frac{6l\sqrt{\alpha_l} \epsilon}{\varepsilon},\ 3l\sqrt{\alpha_l}\epsilon,\ \sqrt{\frac{\omega_2 (d\log(18lr/\epsilon)\!+ \!\log(12/\varepsilon))}{c_{g'}n}},  \frac{\omega_3(d\log(18lr/\epsilon)\!+\!\log(12/\varepsilon))}{c_{g'}n}\right)\\
&=\max\left(\frac{108l\sqrt{\alpha_l} r}{n\varepsilon},\ \sqrt{\frac{\omega_2 (d\log(nl)\!+ \!\log(12/\varepsilon))}{c_{g'}n}},  \frac{\omega_3(d\log(nl)\!+\!\log(12/\varepsilon))}{c_{g'}n}\right).
\end{split}
\end{equation*}
Notice, we have $\alpha_l= c_{t'} r_x^4r^{4l-2}$ where $c_{t'}$ is a constant. Therefore, there exists two universal constants $c_{g'}$ and $c_g$ such that if $n\geq c_{g'}\max(\frac{l^2r^2 r_x^4}{\dm_0d^2\varepsilon^2\tau^4 \log(l)}$, $d\log(l))$, then
\begin{equation*}
\sup_{\wm\in\Omega}\left\| \nabla\Jhn(\wm)\!-\!\nabla\Jm(\wm)\right\|_2\!\leq\! c_g \tau\omega_g \sqrt{l\max_j(\dm_j\dm_{j-1})}\sqrt{ \frac{d\log(nl)\!+\!\log(12/\varepsilon)}{n}}
\end{equation*}
holds with probability at least $1-\varepsilon$, where
$\omega_g=\max\left(\tau \sqrt{\dm_0} r^{2l-1}, \sqrt{\dm_0}  r^{2l-1}, r^{l-1} \right)$.
\end{proof}

\subsubsection{ Proof of Theorem~\ref{thm:localminimal}}
\begin{proof}[\hypertarget{lemmaunilocalminima}{Proof}]
Suppose that $\{\wmii{1},\wmii{2},\cdots,\wmii{m}\}$ are the non-degenerate critical points of $\Jm(\wm)$. So for any $\wmii{k}$, it obeys
\begin{align*}
\inf_i \left|\lambda_i^k\left(\nabla^2 \Jm(\wmii{k})\right)\right| \geq \zeta,
\end{align*}
where $\lambda_i^k\left(\nabla^2 \Jm(\wmii{k})\right)$ denotes the $i$-th eigenvalue of the Hessian $\nabla^2 \Jm(\wmii{k})$ and $\zeta$ is a constant. We further define a set $D=\{\wm\in\Rs{d}\,|\,\|\nabla\Jm(\wm)\|_{2}\leq \epsilon\ \text{and}\ \inf_i |\lambda_i\left(\nabla^2 \Jm(\wmii{k})\right)| \geq \zeta \}$. According to Lemma~\ref{lemma:Decomposition}, $D = \cup_{k=1}^{\infty} D_k$ where each $D_k$ is a disjoint component with $\wmii{k}\in D_k$ for $k\leq m$  and $D_k$ does not contain any critical point of $\Jm(\wm)$ for $k\geq m+1$. On the other hand, by the continuity of $\nabla\Jm(\wm)$, it yields $\|\nabla\Jm(\wm)\|_2=\epsilon$ for $\wm\in\partial D_k$. Notice, we set the value of $\epsilon$ blow which is actually a function related to $n$.

Then by utilizing Theorem~\ref{thm:uniformconvergence1}, we let sample number $n$ sufficient large such that
\begin{equation*}
\sup_{\wm\in\Omega} \left\| \nabla\Jhn(\wm)-\nabla\Jm(\wm)\right\|_2\leq z_g\sqrt{\frac{ d\log(nl)+\log(12/\varepsilon) }{n}}\triangleq \frac{\epsilon}{2}
\end{equation*}
holds with probability at least $1-\varepsilon$, where $z_g=c_g \tau \omega_g \sqrt{l\max_j(\dm_j\dm_{j-1})}$ in which $\omega_g=\max\left(\tau \sqrt{\dm_0} r^{2l-1}\right.$, $\left.\sqrt{\dm_0}  r^{2l-1}, r^{l-1} \right)$.  This further gives that for arbitrary $\wm\in D_k$, we have
\begin{align}
\inf_{\wm \in D_k}\left\| t \nabla\Jhn(\wm)+(1-t)\nabla\Jm(\wm)\right\|_2=&\inf_{\wm \in D_k}\left\| t \left(\nabla\Jhn(\wm)-\nabla\Jm(\wm)\right)+\nabla\Jm(\wm)\right\|_2 \notag\\
\geq &\inf_{\wm \in D_k} \left\|\nabla\Jm(\wm)\right\|_2- \sup_{\wm \in D_k}t\left\|\nabla\Jhn(\wm)-\nabla\Jm(\wm)\right\|_2\notag\\
\geq &\frac{\epsilon}{2}.\label{gradienthhhadfaf}
\end{align}
Similarly, by utilizing Lemma~\ref{thm:uniformconvergence321}, let $n$ be sufficient large such that
\begin{equation*}
\sup_{\wm\in\Omega} \left\| \nabla^2\Jhn(\wm)-\nabla^2\Jm(\wm)\right\|_{\op}\leq z_s\sqrt{\frac{d\log(nl)+\log(20/\varepsilon)}{n}}\leq \frac{\zeta}{2}
\end{equation*}
where $z_s=c_h \tau l \omega_h\max_j(\dm_j\dm_{j-1}) $ in which $\omega_h=\max\!\left(\tau r^{2(l-1)},r^{2(l-1)},r^{l-2}\right)$, holds with probability at least $1-\varepsilon$.  Assume that $\bmm\in\Rs{d}$ is a vector and satisfies $\bmm^T\bmm=1$. In this case, we can bound $\lambda_i^k\left(\nabla^2 \Jhn(\wm)\right)$ for arbitrary $\wm\in D_k$ as follows:
\begin{equation*}
\begin{split}
\inf_{\wm \in D_k}\left|\lambda_i^k\left(\nabla^2 \Jhn(\wm)\right)\right|=&\inf_{\wm \in D_k}\min_{\bmm^T\bmm=1} \left|\bmm^T \nabla^2 \Jhn(\wm)\bmm\right|\\
=&\inf_{\wm \in D_k}\min_{\bmm^T\bmm=1} \left|\bmm^T \left(\nabla^2 \Jhn(\wm)-\nabla^2 \Jm(\wm)\right)\bmm+\bmm^T \nabla^2 \Jm(\wm)\bmm\right|\\
\geq&\inf_{\wm \in D_k}\min_{\bmm^T\bmm=1}\left|\bmm^T \nabla^2 \Jm(\wm)\bmm\right|-\min_{\bmm^T\bmm=1} \left|\bmm^T \left(\nabla^2 \Jhn(\wm)-\nabla^2 \Jm(\wm)\right)\bmm\right|\\
\geq&\inf_{\wm \in D_k}\min_{\bmm^T\bmm=1}\left|\bmm^T \nabla^2 \Jm(\wm)\bmm\right|-\max_{\bmm^T\bmm=1} \left|\bmm^T \left(\nabla^2 \Jhn(\wm)-\nabla^2 \Jm(\wm)\right)\bmm\right|\\
=& \inf_{\wm \in D_k} \inf_i |\lambda_i^k\left(\nabla^2 f(\wm_{(k)},\xm)\right)|-\left\|\nabla^2 \Jhn(\wm)-\nabla^2 \Jm(\wm)\right\|_{\op}\\
\geq &\frac{\zeta}{2}.
\end{split}
\end{equation*}
This means that in each set $D_k$, $\nabla^2\Jhn(\wm)$ has no zero eigenvalues. Then, combine this and Eqn.~\eqref{gradienthhhadfaf}, by Lemma~\ref{lemma:Stability2} we know that if the population risk $\Jm(\wm)$ has no critical point in $D_k$, then the empirical risk $\Jhn(\wm)$ has also no critical point in $D_k$; otherwise it also holds. By Lemma~\ref{lemma:Stability2}, we can also obtain that in $D_k$, if $\Jm(\wm)$ has a unique critical point $\wm_{(k)}$ with non-degenerate index $s_{k}$, then $\Jhn(\wm)$ also has a unique critical point $\wm_{(k)}^n$ in $D_k$ with the same non-degenerate index $s_k$. The first conclusion is proved.

Now we bound the distance between the corresponding critical points of $\Jm(\wm)$ and $\Jhn(\wm)$. Assume that in $D_k$, $\Jm(\wm)$ has a unique critical point $\wmii{k}$ and $\Jhn(\wm)$ also has a unique critical point $\wmin{k}$.
Then, there exists $t\in[0,1]$ such that for any $\zm\in\partial \Bi{d}{1}$, we have
\begin{align*}
\epsilon \geq& \|\nabla \Jm(\wmin{k})\|_2\\
=&\max_{\zm^T\zm=1}
\langle \nabla \Jm(\wmin{k}), \zm \rangle\\
 = &\max_{\zm^T\zm=1} \langle \nabla \Jm(\wmii{k}), \zm \rangle +\langle \nabla^2 \Jm(\wmii{k}+t(\wmin{k}-\wmii{k}))(\wmin{k}-\wmii{k}), \zm \rangle\\
\overset{\text{\ding{172}}}{\geq} & \lr \left(\nabla^2 \Jm(\wmii{k})\right)^2(\wmin{k}-\wmii{k}), (\wmin{k}-\wmii{k})\rl^{1/2}\\
\overset{\text{\ding{173}}}{\geq} & \zeta \|\wmin{k}-\wmii{k}\|_2,
\end{align*}
where \ding{172} holds since $\nabla \Jm(\wmii{k})=\bm{0}$ and \ding{173} holds since $\wmii{k}+t(\wmin{k}-\wmii{k})$ is in $D_k$ and for any $\wm \in D_k$ we have $\inf_i |\lambda_i\left(\nabla^2 \Jm(\wm)\right)| \geq \zeta$. Consider the conditions in Lemma~\ref{thm:uniformconvergence321} and Theorem~\ref{thm:uniformconvergence1}, we can obtain that if
$n\geq c_h\max(l^2r^2 r_x^4/(\dm_0d^2\varepsilon^2\tau^4 \log(l)),d\log(l)/\zeta^2)$ where $c_h$ is a constant, then
\begin{align*}
\|\wmin{k}-\wmii{k}\|_2\leq \frac{2c_g \tau \omega_g}{\zeta}\sqrt{l\max_j(\dm_j\dm_{j-1})} \sqrt{\!\frac{d\log(nl)\!+\!\log(12/\varepsilon)}{n}}
\end{align*}
holds with probability at least $1-\varepsilon$, where $\omega_g=\max\left(\tau \sqrt{\dm_0} r^{2l-1}, \sqrt{\dm_0}  r^{2l-1}, r^{l-1} \right)$.
\end{proof}

\subsection{Proof of Other Lemmas}\label{otherlemmas}

\subsubsection{Proof of Lemma~\ref{sumsafa34sad678safexp}}
\begin{lem}\cite{RigolletSnote}\label{asfkjdr45}
Suppose a random variable $x$ is $\tau^2$-sub-Gaussian, then the random variable $x^2-\EE x^2$ is sub-exponential and obeys:
\begin{equation}\label{EafasFAFQ}
\begin{split}
\EE \left(\exp \lambda\left(x^2-\EE x^2\right)\right)\leq
 &\exp \left(\frac{256\lambda^2\tau^4}{2}\right),\quad |\lambda|\leq \frac{1}{16\tau^2}.
\end{split}
\end{equation}
\end{lem}
\begin{proof}[\hypertarget{lemmasumofsubga}{Proof}]
Here we utilize Lemma~\ref{asfkjdr45} to prove our conclusion. We have
\begin{align*}
\EE  \exp \left(\lambda\left(\sum_{i=1}^k\am_{i} \xm_i^2-\EE\left(\sum_{i=1}^k\am_{i} \xm_i^2\right)\right)\right)
\overset{\text{\ding{172}}}{=}&\prod_{i=1}^{k}\EE \exp\left(\lambda\am_{i}\left( \xm_i^2-\EE \xm_i^2\right)\right)\\
\overset{\text{\ding{173}}}{\leq} &\prod_{i=1}^{k}\EE \exp\left(128\lambda^2\am_{i}^2\tau_i^4\right), \quad |\lambda|\leq \frac{1}{\max_i \am_i\tau^2}\\
\leq &\EE \exp\left(128\lambda^2\tau^4\left(\sum_{i=1}^{k}\am_{i}^2 \right)\right),
\end{align*}
where \ding{172} holds since $\xm_i$ are independent and \ding{173} holds because of Lemma~\ref{asfkjdr45}.
\end{proof}

\subsubsection{Proof of Lemma~\ref{asfasasfshdfhfdet}}\label{12rs43}
\begin{proof}[\hypertarget{lemma6}{Proof}]
Since the $\ell_2$-norm of each $\wmi{j}$ is bounded, \emph{i.e.} $\|\wmi{j}\|_2\leq r\, (1\leq j\leq l)$, we can obtain
\begin{equation*}\label{gradient1}
\begin{split}
\left\|\Bm_{s:t}\right\|_F^2\leq \left\|\Wmi{s}\right\|_F^2 \left\|\Wmi{s-1}\right\|_F^2 \cdots \left\|\Wmi{t}\right\|_F^2\leq r^{2(t-s+1)}\!\triangleq \!\omega_{r}^2\! \overset{\text{\ding{172}}}{\leq} \! \max\left(r^2,r^{2l}\right),
\end{split}
\end{equation*}
where \ding{172} holds since the function $r^{2x}$ obtains its maximum at two endpoints $x=1$ and $x=l$ for case $r<1$ and $r\geq 1$, respectively. On the other hand, we have $\left\|\Bm_{s:t}\right\|_{\op}\leq \left\|\Bm_{s:t}\right\|_F\leq \omega_{r}$. Specifically, we have $\left\|\Bm_{l:1}\right\|_F^2\leq r^{2l}\triangleq \omega_{f}^2.$
\end{proof}

\section{Proofs for Deep nonlinear Neural Networks}\label{deepnonlinear}
In this section, we first present the technical  lemmas in Sec.~\ref{34534dfs4}. Then in Sec.~\ref{safaswerq234} we give the proofs of these lemmas.  Next, we utilize these technical lemmas to prove the results in~Theorems~\ref{thm:stability_sig} $\sim$
\ref{thm:localminimal_sig} and Corollary~\ref{stabiliaftysig} in Sec.~\ref{saf3we543ydf}. Finally, we give the proofs of other lemmas in Sec.~\ref{definitiondsfadf}.

\subsection{Technical Lemmas}\label{34534dfs4}
Here we present the key lemmas and theorems for proving our desired results. For brevity, we define an operation $\mcode{G}$ which maps an arbitrary vector $\zm\in\Rs{k}$ into a diagonal matrix $\mcode{G}(\zm)\in\Rss{k}{k}$ with its $i$-th diagonal entry equal to $\sigma(\zm_i)(1-\sigma(\zm_i))$ in which $\zm_i$ denotes the $i$-th entry of $\zm$. We further define $\Am_i\in\Rss{\dm_{i-1}}{\dm_i}$ as follows:
\begin{equation}\label{gradidsfa345ent1}
\begin{split}
\Am_i=(\Wmi{i})^T\mcode{G}(\umi{i})\in\Rss{\dm_{i-1}}{\dm_i}\quad (i=1,\cdots,l),
\end{split}
\end{equation}
where $\Wmi{i}$ is the weight matrix in the $i$-th layer and $\umi{i}$ is the linear output of the $i$-th layer. In this section, we define
\begin{equation}\label{gradsfa2343543ient1}
\begin{split}
\Bm_{s:t}=\Am_s\Am_{s+1}\cdots\Am_{t}\in\Rss{\dm_{s-1}}{\dm_t},\ (s\leq t)\quad \text{and}\quad \Bm_{s:t}=\Im ,\ (s> t).
\end{split}
\end{equation}
\begin{lem}\label{sig_loss}
Suppose that the activation function in deep neural network are sigmoid functions. Then the gradient of $f(\wm,\xm)$ with respect to $\wmi{j}$ can be formulated as
\begin{equation*}\label{gradient1}
\nabla_{\wmi{j}}f(\wm,\xm)=\vect{\left(\mcode{G}(\umi{j}) \Bm_{j+1:l} (\vmi{l}-\ym)\right) (\vmi{j-1})^T},\ (j=1,\cdots,l-1),
\end{equation*}
and
\begin{equation*}
\nabla_{\wmi{l}}f(\wm,\xm) =\vect{\left(\mcode{G}(\umi{l})(\vmi{l}-\ym)\right) (\vmi{l-1})^T}.
\end{equation*}
Besides, the loss $f(\wm,\xm)$ is $\alpha$-Lipschitz, $$\|\nabla_{\wm}f(\wm,\xm)\|_2\leq \alpha,$$
where $\alpha=\sqrt{\frac{1}{16}c_yc_d\left(1 + c_r (l-1)\right)}$ in which $c_y$, $c_d$ and $c_r$ are defined as
\begin{equation*}
\begin{split}
\|\vmi{l}-\ym\|_2^2\leq c_y<+\infty,  \quad c_{d}=\max(\dm_0,\dm_1,\cdots,\dm_l)  \quad \text{and}\quad c_r=\max\left(\frac{r^2}{16},\left(\frac{r^2}{16}\right)^{l-1}\right).
\end{split}
\end{equation*}
\end{lem}

\begin{lem}\label{sig_gradient}
Suppose that the activation functions in deep neural network are sigmoid functions. Then there exists two universal constants $c_{s_1}$ and $c_{s_2}$ such that
\begin{equation*}\label{gradient1}
\begin{split}
\left\|\nabla^2_{\wm}f(\wm,\xm)\right\|_{\op}\leq \left\|\nabla^2_{\wm}f(\wm,\xm)\right\|_{F}
\leq \varsigma,
\end{split}
\end{equation*}
where $\varsigma=\sqrt{c_{s_1}c_rc_d^2 l^2\left(c_{s_2}c_d^2+l^2c_r\right)}$ in which $c_{d}=\max_i \dm_i$ and $c_r=\max\left(\frac{r^2}{16},\left(\frac{r^2}{16}\right)^{l-1}\right)$. Moreover,
the gradient $\nabla_{\wm} f(\wm,\xm)$ is $\varsigma$-Lipschitz, \textit{i.e.}
\begin{equation*}\label{gradient1}
\begin{split}
\left\|\nabla_{\wm}f(\wm_1,\xm)-\nabla_{\wm}f(\wm_2,\xm)\right\|_2
\leq& \varsigma\|\wm_1-\wm_2\|_2.
\end{split}
\end{equation*}
Similarly, there also exist a universal constant $\xi$ such that
\begin{equation*}\label{gradient1}
\begin{split}
\left\|\nabla^3_{\wm}f(\wm,\xm)\right\|_{\op}\leq \left\|\nabla^3_{\wm}f(\wm,\xm)\right\|_{F}
\leq \xi.
\end{split}
\end{equation*}
\end{lem}

\begin{lem}\label{exp fsd_gsfasradient}
Suppose that the activation function in deep neural network are sigmoid functions. Then we have
\begin{equation*}\label{gradient1}
\begin{split}
\left\|\nabla_{\wm}\nabla_{\xm}f(\wm,\xm)\right\|_{\op}\leq \left\|\nabla_{\wm}\nabla_{\xm}f(\wm,\xm)\right\|_{F}
\leq& \beta,
\end{split}
\end{equation*}
where $\beta=\sqrt{\frac{2^6}{3^8}c_yc_r(l+2)\left(dc_r+(l-1)lc_dc_r+lc_d\right)}$ in which $c_y$, $c_d$ and $c_r$ are defined in Lemma~\ref{sig_loss}.
\end{lem}

\begin{lem}\label{thm:gradsgfdsgdient12}
Suppose that the input sample  $\xm$ obeys Assumption~\ref{assumption12sg} and the activation functions in deep neural network are sigmoid functions. The gradient of the loss is $8\beta^2\tau^2$-sub-Gaussian. Specifically, for any $\lam\in\Rs{d}$, we have
\begin{equation*}
\EE\left(\lr\lam ,\nabla_{\wm}f(\wm,\xm)-\EE \nabla_{\wm}f(\wm,\xm)\rl \right)\leq \exp\left(\frac{8\beta^2\tau^2\|\lam \|_2^2}{2}\right),
\end{equation*}
where $\beta=\sqrt{\frac{2^6}{3^8}c_yc_r(l+2)\left(dc_r+(l-1)lc_dc_r+lc_d\right)}$ in which $c_y$, $c_d$ and $c_r$ are defined in Lemma~\ref{sig_loss}.
\end{lem}

\begin{lem}\label{thm:garadifghtent12}
Suppose that the input sample $\xm$ obeys Assumption~\ref{assumption12sg} and the activation functions in deep neural network are sigmoid functions. The Hessian of the loss, evaluated on a unit vector, is sub-Gaussian. Specifically, for any unit $\lam\in\SS^{d-1}$ (\textit{i.e.} $\|\lam\|_2=1$), there exist universal constant $\gamma$ such that
\begin{equation*}
\EE\left(s \lr\lam,\left(\nabla^2_{\wm}f(\wm,\xm)-\EE \nabla^2_{\wm}f(\wm,\xm)\right)\lam\rl \right)\leq \exp\left(\frac{8s^2\gamma^2 \tau^2}{2}\right).
\end{equation*}
Notice, $\gamma$ obeys $\gamma\geq\|\nabla_{\xm}\nabla^2_{\wm}f(\wm,\xm)\|_{\op}$.
\end{lem}

\begin{lem} \label{thm:uniformconvergence1fds_sig}
Assume that the input sample $\xm$ obeys Assumption~\ref{assumption12sg} and the activation functions in deep neural network are sigmoid functions. Then the sample Hessian uniformly converges to the population Hessian in operator norm. That is, there exists such two universal constants $c_{m'}$ and $c_m$ such that if $n\geq \frac{c_{m'}\xi^2 r^2}{\gamma^2\tau^2d\log(l)}$, then
\begin{equation*}
\sup_{\wm\in\Omega}\!\left\| \nabla^2\Jhn(\wm)\!-\!\nabla^2\Jm(\wm)\right\|_{\op}\!\leq\! c_m\gamma\tau\sqrt{\frac{d\log(nl)\!+\!\log(4/\varepsilon)}{n}}
\end{equation*}
holds with probability at least $1-\varepsilon$. Here $\gamma$ is the same parameter in Lemma~\ref{thm:garadifghtent12}.
\end{lem}

\subsection{Proofs of Technical Lemmas}\label{safaswerq234}
For brevity, we also define
$$\Dm_{s:t}=\|\Wmi{s}\|_F^2\cdots\|\Wmi{t}\|_F^2\ (s\leq t)\quad \text{and}\quad \Dm_{s:t}=1 ,\ (s> t).$$
We define a matrix $\Pm_k\in\Rss{\dm_k^2}{\dm_k}$ whose $((s-1)\dm_k+s,s)\ (s=1,\cdots,\dm_k)$ entry equal to $\sigma(\umi{k}_{s})(1-\sigma(\umi{k}_{s}))(1-2\sigma(\umi{k}_{s}))$ and rest entries are all $0$. On the other hand, since the values in $\vmi{l}$ belong to the range $[0,1]$ and $\ym$ is the label, $\|\vmi{l}-\ym\|_2^2$ can be bounded:
$$\|\vmi{l}-\ym\|_2^2\leq c_y<+\infty,$$
where $c_y$ is a universal constant. We further define $c_d=\max(\dm_0,\dm_1,\cdots,\dm_l)$.

Then we give a lemma to summarize the properties of $\mcode{G}(\umi{i})$ defined in Eqn.~\eqref{gradidsfa345ent1}, $\Bm_{s:t}$ defined in Eqn.~\eqref{gradsfa2343543ient1}, $\Dm_{s:t}$ and $\Pm_k$.
\begin{lem}\label{asfashdfhfdet}
For $\mcode{G}(\umi{i})$ defined in Eqn.~\eqref{gradidsfa345ent1}, $\Bm_{s:t}$ defined in Eqn.~\eqref{gradsfa2343543ient1}, $\Dm_{s:t}$ and $\Pm_k$, we have the following properties:
\begin{myitemize}
\item[(1)] For arbitrary matrices $\Mm$ and $\Nm$ of proper sizes, we have
    \begin{equation*}\label{gradient1}
\begin{split}
\|\mcode{G}(\umi{i})\Mm\|_F^2\leq \frac{1}{16}\|\Mm\|_F^2\quad \text{and} \quad \|\Nm\mcode{G}(\umi{i})\|_F^2\leq \frac{1}{16}\|\Nm\|_F^2.
\end{split}
\end{equation*}
\item[(2)] For arbitrary matrices $\Mm$ and $\Nm$ of proper sizes, we have
\begin{equation*}\label{gradient1}
\begin{split}
\|\Pm_k\Mm\|_F^2\leq \frac{2^6}{3^8}\|\Mm\|_F^2\quad \text{and}\quad \|\Nm\Pm_k\|_F^2\leq \frac{2^6}{3^8}\|\Nm\|_F^2.
\end{split}
\end{equation*}
\item[(3)] For $\Bm_{s:t}$ and $\Dm_{s:t}$, we have
\begin{equation*}\label{gradient1}
\begin{split}
\left\|\Bm_{s:t}\right\|_F^2\leq \frac{1}{16^{t-s+1}}\Dm_{s:t}\quad \text{and} \quad \frac{1}{16^{t-s+1}}\Dm_{s:t} \leq  c_{st}\leq c_r,
\end{split}
\end{equation*}
where $c_{st}=\left(\frac{r}{4}\right)^{2(t-s+1)}$ and $c_r=\max\left(\frac{r^2}{16},\left(\frac{r^2}{16}\right)^{l-1}\right)$.
\end{myitemize}
\end{lem}
It should be pointed out that we defer the proof of Lemma~\ref{asfashdfhfdet} to Sec.~\ref{definitiondsfadf}.

\subsubsection{Proof of Lemma~\ref{sig_loss}}\label{afsafsaggjx}
\begin{proof}[\hypertarget{lemma32}{Proof}]
We use chain rule to compute the gradient of $f(\wm,\xm)$ with respect to $\wmi{j}$. We first compute several basis gradient. According to the relationship between $\umi{j}, \vmi{j}, \Wmi{j}$ and $\f(\wm,\xm)$, we have
\begin{equation}\label{chain_rule}
\begin{split}
&\nabla_{\vmi{l}}f(\wm,\xm)= \vmi{l}-\ym,\\ &\nabla_{\vmi{i}}f(\wm,\xm)=\frac{\partial \umi{i+1}}{\partial \vmi{i}}\frac{\partial f(\wm,\xm)}{\partial \umi{i+1}}=(\Wmi{i+1})^T\frac{\partial f(\wm,\xm)}{\partial \umi{i+1}},\quad\ \qquad\quad(i=1,\cdots,l-1),\\
&\nabla_{\umi{i}}f(\wm,\xm)=\frac{\partial \vmi{i}}{\partial \umi{i}} \frac{\partial f(\wm,\xm)}{\partial \vmi{i}}=\mcode{G}(\umi{i})\frac{\partial f(\wm,\xm)}{\partial \vmi{i}},\quad \qquad\qquad\quad\ \ \ (i=1,\cdots,l),\\
&\nabla_{\Wmi{i}}f(\wm,\xm)=\frac{\partial \umi{i}}{\partial \wmi{i}} \left(\frac{\partial f(\wm,\xm)}{\partial \umi{i}}\right)^T =\vmi{i-1}\left(\frac{\partial f(\wm,\xm)}{\partial \umi{i}}\right)^T, \qquad (i=1,\cdots,l).
\end{split}
\end{equation}

Then by chain rule, we can easily compute the gradient of $f(\wm,\xm)$ with respect to $\wmi{j}$ which is formulated as
\begin{equation*}\label{gradient1}
\nabla_{\wmi{j}}f(\wm,\xm)=\vect{\vmi{j-1}\left(\mcode{G}(\umi{j}) \Am_{j+1}\Am_{j+2}\cdots \Am_{l}(\vmi{l}-\ym)\right)^T},\ (j=1,\cdots,l-1),
\end{equation*}
and
\begin{equation*}
\nabla_{\wmi{l}}f(\wm,\xm) =\vect{\vmi{l-1}\left(\mcode{G}(\umi{l})(\vmi{l}-\ym)\right)^T}.
\end{equation*}
Besides, since the values in $\vmi{l}$ belong to the range $[0,1]$. Combine with Lemma~\ref{asfashdfhfdet}, we can bound $\|\nabla_{\wm} f(w,x)\|_2$ as follows:
\begin{equation*}
\begin{split}
\left\|\nabla_{\wm} f(\wm,\xm)\right\|_2^2=&\sum_{j=1}^l \left\|\nabla_{\wmi{j}}f(\wm,\xm)\right\|_2^2\\
=&\left\|\vmi{l-1}\left(\mcode{G}(\umi{l})(\vmi{l}-\ym)\right)^T \right\|_F^2 + \sum_{j=1}^{l-1}\left\| \vmi{j-1}\left(\mcode{G}(\umi{j}) \Bm_{j+1:l}(\vmi{l}-\ym)\right)^T\right\|_F^2\\
\leq &\frac{1}{16}\dm_{l-1}\left\|\vmi{l}-\ym \right\|_2^2 + \frac{1}{16}\left\|\vmi{l}-\ym\right\|_2^2\sum_{j=1}^{l-1} \dm_{j-1}\left\| \Bm_{j+1:l}\right\|_F^2\\
\led{172} &\frac{1}{16}c_yc_d + \frac{1}{16}c_y c_d c_r (l-1),
\end{split}
\end{equation*}
where $c_y, c_d, c_r$ are defined as
\begin{equation*}
\begin{split}
\|\vmi{l}-\ym\|_2^2\leq c_y,  \quad c_{d}=\max(\dm_0,\dm_1,\cdots,\dm_l)  \quad \text{and}\quad c_r=\max\left(\frac{r^2}{16},\left(\frac{r^2}{16}\right)^{l-1}\right).
\end{split}
\end{equation*}
Notice, \ding{172} holds since in Lemma~\ref{asfashdfhfdet}, we have
\begin{equation*}\label{gradient1}
\begin{split}
\left\|\Bm_{s:t}\right\|_F^2\leq  \left(\frac{r}{4}\right)^{2(t-s+1)}\leq \max\left(\frac{r^2}{16},\left(\frac{r^2}{16}\right)^{l-1}\right).
\end{split}
\end{equation*}

Thus, we can obtain
\begin{equation*}
\begin{split}
\|\nabla_{\wm} f(w,x)\|_2
\leq \sqrt{\frac{1}{16}c_yc_d\left(1 + c_r (l-1)\right)}\ \ \triangleq \alpha.
\end{split}
\end{equation*}
The proof is completed.
\end{proof}

\subsubsection{Proof of Lemma~\ref{sig_gradient}}
For convenience, we first give the computation of some gradients.
\begin{lem}\label{sig_gdgsaghhradsdfsagfasgient}
Assume the activation functions in deep neural network are sigmoid functions. Then the following properties hold:
\begin{myitemize}
\item[(1)] We can compute the gradients $\frac{\partial  f(\wm,\xm)}{\partial \umi{i}}$ and $\frac{\partial  f(\wm,\xm)}{\partial \vmi{i}}$ as
\begin{equation*}\label{gsfasgdasgaradient1}
\begin{split}
\frac{\partial  f(\wm,\xm)}{\partial \umi{i}}=\mcode{G}(\umi{i})\Bm_{i+1:l} (\vmi{l}-\ym)\quad \text{and}\quad \frac{\partial  f(\wm,\xm)}{\partial \vmi{i}}=\Bm_{i+1:l}(\vmi{l}-\ym).
\end{split}
\end{equation*}
\item[(2)] We can compute the gradient $\frac{\partial \umi{i}}{\partial \wmi{j}}$ as
    \begin{equation*}\label{gsfasgsdgsadasgaradient1}
\begin{split}
&\frac{\partial \umi{i}}{\partial \wmi{j}}
=(\vmi{j-1})^T\otimes\left(\mcode{G}(\umi{j}) \Bm_{j+1:i-1}(\Wmi{i})^T\right)^T\in\Rss{\dm_i}{\dm_j\dm_{j-1}}, \ (i>j).\\
&\frac{\partial \umi{i}}{\partial \wmi{i}}
= (\vmi{i-1})^T\otimes \Im_{\dm_i}\in\Rss{\dm_i}{\dm_i\dm_{i-1}}, \ (i=j).
\end{split}
\end{equation*}
\item[(3)]  We can compute the gradient $\frac{\partial \vmi{i}}{\partial \wmi{j}}$ as
\begin{equation*}\label{gsfasgsdgsadasgaradient1}
\begin{split}
\frac{\partial \vmi{i}}{\partial \wmi{j}}
=(\vmi{j-1})^T \otimes \left(\mcode{G}(\umi{j}) \Bm_{j+1:i}\right)^T\in\Rss{\dm_i}{\dm_j\dm_{j-1}}, \ (i\geq j).
\end{split}
\end{equation*}
\end{myitemize}
\end{lem}
It should be pointed out that the proof of Lemma~\ref{sig_gdgsaghhradsdfsagfasgient} can be founded  Sec.~\ref{definitiondsfadf}.

\begin{proof}[\hypertarget{lemma33}{Proof}]
To prove our conclusion, we have two steps: computing the Hessian and bounding its operation norm.

{\bf\noindent{Step 1. Compute the Hessian}}:
We first consider the computation of $\frac{\partial^2 f(\wm,\xm)}{\partial \wmi{i}^T\partial \wmi{j} }$:
\begin{equation*}\label{gradient1}
\begin{split}
\frac{\partial^2 f(\wm,\xm)}{\partial \wmi{i}^T\partial \wmi{j}}=&\frac{\partial\left(\vect{\left(\mcode{G}(\umi{j}) \Am_{j+1}\Am_{j+2}\cdots \Am_{l}(\vmi{l}-\ym)\right) (\vmi{j-1})^T}\right)}{\partial \wmi{i}^T}.
\end{split}
\end{equation*}
Recall that we define
\begin{equation*}\label{gradient1}
\begin{split}
&\Bm_{s:t}=\Am_s\Am_{s+1}\cdots\Am_{t}\in\Rss{\dm_{s-1}}{\dm_t},\ (s\leq t)\quad \text{and}\quad \Bm_{s:t}=\Im ,\ (s> t).
\end{split}
\end{equation*}
Then we have
\begin{equation*}\label{gradient1}
\begin{split}
\frac{\partial^2 f(\wm,\xm)}{\partial \wmi{i}^T\partial \wmi{j}}=& \left(\vmi{j-1}(\vmi{l}-\ym)^T\Bm_{j+1:l}^T\right)\otimes \left(\Im_{\dm_{j}}\right)\frac{\partial \vect{\mcode{G}(\umi{j})}}{\partial \wmi{i}^T} (\triangleq \Qm_1^{ij})\\
&+\!\!\sum_{k=j+1}^{l}\!\! \left(\vmi{j-1}(\vmi{l}-\ym)^T\Bm_{k+1:l}^T\right)\!\otimes\! \left(\mcode{G}(\umi{j})\Bm_{j+1:k-1}\Wm_{k}^T\right) \frac{\partial \vect{\mcode{G}(\umi{k})}}{\partial \wmi{i}^T}(\triangleq \Qm_2^{ij})\\
&+ \left(\vmi{j-1}(\vmi{l}-\ym)^T\Bm_{i+1:l}^T\mcode{G}(\umi{i})\right)\otimes \left(\mcode{G}(\umi{j})\Bm_{j+1:i-1}\right) \frac{\partial \vect{\Wm_{i}^T}}{\partial \wmi{i}^T}(\triangleq \Qm_3^{ij})\\
&+\vmi{j-1} \otimes \left(\mcode{G}(\umi{j})\Bm_{j+1:l}\right) \frac{\partial (\vmi{l}-\ym) }{\partial \wmi{i}^T} (\triangleq \Qm_4^{ij})\\
&+\Im_{\dm_{j-1}}\otimes \left(\mcode{G}(\umi{j})\Bm_{j+1:l}(\vmi{l}-\ym)\right) \frac{\partial \vmi{j-1}}{\partial \wmi{i}^T}(\triangleq \Qm_5^{ij})\\
\end{split}
\end{equation*}

{\bf\noindent{Case I: $i>j$}}.
We first consider the case that $i\!>\!j$. In this is case, $\Qm_1^{ij}\!=\!\bm{0}$ since $\frac{\partial \vect{\mcode{G}(\umi{j})}}{\partial \wmi{i}^T}=\bm{0}$.
Computing $\Qm_2^{ij}$ needs more efforts. By utilizing the computation of $\frac{\partial \umi{k}}{\partial \wmi{i}}$ in Lemma~\ref{sig_gdgsaghhradsdfsagfasgient}, we have
\begin{equation*}\label{gsfasgsdgsadasgaradient1}
\begin{split}
\frac{\partial \vect{\mcode{G}(\umi{k})}}{\partial \wmi{i}}\!=\!\frac{\partial\vect{\mcode{G}(\umi{k})}}{\partial \umi{k}}\frac{\partial \umi{k}}{\partial \wmi{i}}\!=\! \Pm_k\! \left(\!\vmi{i-1})^T\!\otimes\!\left(\mcode{G}(\umi{i}) \Bm_{i+1:k-1}(\Wmi{k})^T\!\right)^T\!\right),  (k\!>\!i)
\end{split}
\end{equation*}
where $\Pm_k$ is a matrix of size $\dm_k^2\times \dm_k$ whose $((s-1)\dm_k+s,s)\ (s=1,\cdots,\dm_k)$ entry equal to $\sigma(\umi{k}_{s})(1-\sigma(\umi{k}_{s}))(1-2\sigma(\umi{k}_{s}))$ and rest entries are all $0$. When $k=i$,
\begin{equation*}\label{gsfasgsdgsadasgaradient1}
\begin{split}
\frac{\partial \vect{\mcode{G}(\umi{k})}}{\partial \wmi{k}}=\frac{\partial\vect{\mcode{G}(\umi{k})}}{\partial \umi{k}}\frac{\partial \umi{k}}{\partial \wmi{k}}= \Pm_k \left((\vmi{k-1})^T\otimes \Im_{\dm_k} \right)\in\Rss{\dm_k^2}{\dm_{k}\dm_{k-1}}.
\end{split}
\end{equation*}
Note that for $k<i$, we have $\frac{\partial \mcode{G}(\umi{k})}{\partial \wmi{i}}=\bm{0}$. For brevity, let
\begin{equation}\label{g35345ent1}
\begin{split}
\Dm_k\triangleq\left(\left(\vmi{j-1}(\vmi{l}-\ym)^T\Bm_{k+1:l}^T \right)\otimes \left(\mcode{G}(\umi{j})\Bm_{j+1:k-1}\Wm_{k}^T\right)\right)\, (k=i,\cdots,l).
\end{split}
\end{equation}
Therefore, we have
\begin{equation*}\label{gsfasgsdgsadasgaradient1}
\begin{split}
\Qm_2^{ij}=\Dm_i\Pm_i \left((\vmi{i-1})^T\otimes \Im_{\dm_i} \right)+\sum_{k=i+1}^{l} \Dm_k\Pm_k \left((\vmi{i-1})^T\otimes\left(\mcode{G}(\umi{i}) \Bm_{i+1:k-1}(\Wmi{k})^T\right)^T\right).
\end{split}
\end{equation*}
Then we consider $\Qm_3^{ij}$.
\begin{equation*}\label{gsfasgsdgsadasgaradient1}
\begin{split}
\Qm_3^{ij}=\left(\vmi{j-1}(\vmi{l}-\ym)^T\Bm_{i+1:l}^T\mcode{G}(\umi{i})\right)\otimes \left(\mcode{G}(\umi{j})\Bm_{j+1:i-1}\right).
\end{split}
\end{equation*}
Also we can use the computation of $\frac{\partial \vmi{l}}{\partial \wmi{i}}$ in Lemma~\ref{sig_gdgsaghhradsdfsagfasgient} and compute $\Qm_4^{ij}$ as follows:
\begin{equation*}\label{gsfasgsdgsadasgaradient1}
\begin{split}
\Qm_4^{ij}=&\vmi{j-1} \!\otimes\! \left(\mcode{G}(\umi{j})\Bm_{j+1:l}\right) \!\frac{\partial (\vmi{l}-\ym) }{\partial \wmi{i}^T}\\
=&\left(\vmi{j-1} \!\otimes\! \left(\mcode{G}(\umi{j})\Bm_{j+1:l}\right)\right) \! \left(\!(\vmi{i-1})^T \!\otimes\! \left(\!\mcode{G}(\umi{i}) \Bm_{i+1:l}\right)^T \!\right).
\end{split}
\end{equation*}
Finally, since $i>j$, we can compute $\Qm_5^{ij}=\bm{0}$.

{\bf\noindent{Case II: $i=j$}}.
We first consider $\frac{\partial \mcode{G}(\umi{k})}{\partial \wmi{k}}$:
\begin{equation*}\label{gsfasgsdgsadasgaradient1}
\begin{split}
\frac{\partial \vect{\mcode{G}(\umi{k})}}{\partial \wmi{k}^T}=\frac{\partial\vect{\mcode{G}(\umi{k})}}{\partial \umi{k}}\frac{\partial \umi{k}}{\partial \wmi{k}^T}= \Pm_k \left((\vmi{k-1})^T\otimes \Im_{\dm_k} \right)\in\Rss{\dm_k^2}{\dm_{k}\dm_{k-1}},
\end{split}
\end{equation*}
where $\Pm_k$ is a matrix of size $\dm_k^2\times \dm_k$ whose $(s,(s-1)\dm_k+s)$ entry equal to $\sigma(\umi{k}_{s})(1-\sigma(\umi{k}_{s}))(1-2\sigma(\umi{k}_{s}))$ and rest entries are all $0$. $\Qm_1^{jj}$ can be computed as
\begin{equation*}\label{gsfasgsdgsadasgaradient1}
\begin{split}
\Qm_1^{jj}=&\left(\vmi{j-1}(\vmi{l}-\ym)^T\Bm_{j+1:l}^T\right)\otimes \left(\Im_{\dm_{j}}\right)\frac{\partial \vect{\mcode{G}(\umi{j})}}{\partial \wmi{j}^T}\\
=&\left(\left(\vmi{j-1}(\vmi{l}-\ym)^T\Bm_{j+1:l}^T\right)\otimes \left(\Im_{\dm_{j}}\right)\right)\left(\Pm_j \left((\vmi{j-1})^T\otimes \Im_{\dm_j} \right)\right).
\end{split}
\end{equation*}
As for $\Qm_2^{jj}$, by Eqn.~\eqref{g35345ent1} we have
\begin{equation*}\label{gsfasgsdgsadasgaradient1}
\begin{split}
\Qm_2^{jj}=&\sum_{k=j+1}^{l} \Dm_k\Pm_k \left(\vmi{j-1})^T\otimes\left(\mcode{G}(\umi{j}) \Bm_{j+1:k-1}(\Wmi{k})^T\right)^T\right).
\end{split}
\end{equation*}
Since $i=j$, $\Qm_3^{jj}$ does not exist. For convenience, we just set $\Qm_3^{jj}=\bm{0}$.

Now we consider $\Qm_4^{jj}$ which can be computed as follows:
\begin{equation*}\label{gsfasgsdgsadasgaradient1}
\begin{split}
\Qm_4^{jj} =& \vmi{j-1}\! \otimes\! \left(\mcode{G}(\umi{j})\Bm_{j+1:l}\right)\! \frac{\partial (\vmi{l}\!-\!\ym) }{\partial \wmi{j}^T}\\
=&\left(\vmi{j-1} \!\otimes\! \left(\mcode{G}(\umi{j})\Bm_{j+1:l}\right)\right)\! \left((\vmi{j-1})^T\! \otimes\! \left(\mcode{G}(\umi{j}) \Bm_{j+1:l}\right)^T\! \right).
\end{split}
\end{equation*}
Finally, since $i=j$, we can compute $\Qm_5^{jj}=\bm{0}$.

{\bf\noindent{Case III: $i<j$}}. Since $\frac{\partial^2 f(\wm,\xm)}{\partial \wm\partial \wm^T}$ is symmetrical, we have $\Qm_k^{ij}=\Qm_k^{ji}\ (k=1,\cdots,5)$.

{\bf\noindent{Step 2. Bound the operation norm of Hessian}}: We mainly use Lemma~\ref{asfashdfhfdet} to achieve this goal. From Lemma~\ref{asfashdfhfdet}, we have
\begin{myitemize}
\item[(1)] For arbitrary matrices $\Mm$ and $\Nm$ of proper size, we have
    \begin{equation*}\label{gradient1}
\begin{split}
\|\mcode{G}(\umi{i})\Mm\|_F^2\leq \frac{1}{16}\|\Mm\|_F^2\quad \text{and} \quad \|\Nm\mcode{G}(\umi{i})\|_F^2\leq \frac{1}{16}\|\Nm\|_F^2.
\end{split}
\end{equation*}
\item[(2)] For arbitrary matrices $\Mm$ and $\Nm$ of proper size, we have
\begin{equation*}\label{gradient1}
\begin{split}
\|\Pm_k\Mm\|_F^2\leq \frac{2^6}{3^8}\|\Mm\|_F^2\quad \text{and}\quad \|\Nm\Pm_k\|_F^2\leq \frac{2^6}{3^8}\|\Nm\|_F^2.
\end{split}
\end{equation*}
\item[(3)] For $\Bm_{s:t}$ and $\Dm_{s:t}$, we have
\begin{equation*}\label{gradient1}
\begin{split}
\left\|\Bm_{s:t}\right\|_F^2\leq \frac{1}{16^{t-s+1}}\Dm_{s:t}\quad \text{and} \quad \frac{1}{16^{t-s+1}}\Dm_{s:t} \leq c_r,
\end{split}
\end{equation*}
where $c_r=\max\left(\frac{r^2}{16},\left(\frac{r^2}{16}\right)^{l}\right)$.
\end{myitemize}
The values of entries in $\vmi{h}$ are bounded by $0\leq \sigma(\umi{i}_h)\leq 1$ which leads to $\left\|\vmi{h}\right\|_F^2 \leq \dm_h\leq c_d$, where $c_d=\max_i \dm_i$. On the other hand, since the values in $\vmi{l}$ belong to the range $[0,1]$ and $\ym$ is the label, $\|\vmi{l}-\ym\|_2^2$ can be bounded:
$$\|\vmi{l}-\ym\|_2^2\leq c_y<+\infty,$$
where $c_y$ is a universal constant.

We first define
$$\Cm_k^{ij}=\Dm_k\Pm_k \left(\vmi{i-1})^T\otimes\left(\mcode{G}(\umi{i}) \Bm_{i+1:k-1}(\Wmi{k})^T\right)^T\right)\ \text{and}\ \Cm^{ij}=\Dm_i \Pm_i \left((\vmi{i-1})^T\otimes \Im_{\dm_i} \right),$$
where $\Dm_k$ is defined in Eqn.~\eqref{g35345ent1}.

{\bf\noindent{Case I: $i>j$}}. According to the definition of $\Cm^{ij}$ and $\Cm_k^{ij}$, we have $\Qm_2^{ij}= \Cm^{ij}+\sum_{k=i+1}^{l} \Cm_k^{ij}$. So we have
\begin{equation*}\label{gradient1}
\begin{split}
\left\|\frac{\partial^2 f(\wm,\xm)}{\partial \wmi{i}^T\partial \wmi{j}}\right\|_F^2=&\left\|\Qm_1^{ij}+ \Qm_2^{ij}+ \Qm_3^{ij}+ \Qm_4^{ij} +\Qm_5^{ij}\right\|_F^2\\
=&\left\|\Cm^{ij}+\sum_{k=i+1}^{l} \Cm_k^{ij}+ \Qm_3^{ij}+ \Qm_4^{ij}\right\|_F^2\\
=&(l-i+3)\left(\left\|\Cm^{ij}\right\|_F^2+\sum_{k=i+1}^{l} \left\|\Cm_k^{ij}\right\|_F^2+ \left\|\Qm_3^{ij}\right\|_F^2+ \left\|\Qm_4^{ij}\right\|_F^2\right).
\end{split}
\end{equation*}
Here we bound each term separately:
\begin{equation*}\label{gradient1}
\begin{split}
\left\|\Cm^{ij}\right\|_F^2
\leq&\left\|\vmi{j-1}\right\|_F^2\left\| \vmi{l}-\ym \right\|_F^2 \left\|\Bm_{i+1:l}\right\|_F^2  \frac{1}{16}\left\|\Bm_{j+1:i-1}\Wm_{i}^T\right\|_F^2 \frac{2^6}{3^8}\left\|\vmi{i-1}\right\|_F^2 \left\|\Im_{\dm_i}\right\|_F^2 \\
\leq&\frac{2^6}{3^8}c_y \dm_{j-1}\dm_{i-1}\dm_{i} \frac{1}{16^{l-i}}\Dm_{i+1:l}  \frac{1}{16^{i-j}}\Dm_{j+1:i}\\
\leq&\frac{2^6}{3^8}c_y\dm_{j-1}\dm_{i-1}\dm_{i}  \frac{1}{16^{l-j}}\Dm_{j+1:l}\\
\leq &\frac{2^6}{3^8}c_y\dm_{j-1}\dm_{i-1}\dm_{i}  c_r.
\end{split}
\end{equation*}
Similarly, we can bound $\|\Cm_k^{ij}\|_F^2$ as follows:
\begin{equation*}\label{gradient1}
\begin{split}
&\left\|\Cm_k^{ij}\right\|_F^2\\
\leq&\left\|\vmi{j-1}\right\|_F^2\! \left\|\vmi{l}\!-\!\ym\right\|_F^2 \! \!\left\|\Bm_{k+1:l}\right\|_F^2\!  \frac{1}{16}\left\|\Bm_{j+1:k-1}\Wm_{k}^T\right\|_F^2 \! \frac{2^6}{3^8}\left\|\vmi{i-1}\right\|_F^2 \! \frac{1}{16}\left\|\Bm_{i+1:k-1}(\Wmi{k})^T\right\|_F^2 \\
\leq &\frac{2^6}{3^8} c_y \dm_{j-1}\dm_{i-1}  \frac{1}{16^{l-k}}\Dm_{k+1:l} \frac{1}{16^{k-j-1}}\Dm_{j+1:k} \frac{1}{16^{k-i-1}}\Dm_{i+1:k}
\\
=&\frac{2^6}{3^8} c_y \dm_{j-1}\dm_{i-1} \frac{1}{16^{l-j-1}}\Dm_{j+1:l} \frac{1}{16^{k-i-1}}\Dm_{i+1:k}
\\
\leq &\frac{2^{14}}{3^8} c_y \dm_{j-1}\dm_{i-1} c_r^2.
\end{split}
\end{equation*}

We also bound $\left\|\Qm_3^{ij}\right\|_F^2$ as
\begin{equation*}\label{gradient1}
\begin{split}
\left\|\Qm_3^{ij}\right\|_F^2
\leq \left\|\vmi{j-1}\right\|_F^2 \left\|\vmi{l}-\ym\right\|_F^2  \frac{1}{16}\left\|\Bm_{i+1:l}\right\|_F^2  \frac{1}{16}\left\|\Bm_{j+1:i-1}\right\|_F^2
\leq  \frac{1}{2^8}c_y\dm_{j-1}c_r.
\end{split}
\end{equation*}
Finally, we bound $\left\|\Qm_4^{ij}\right\|_F^2$ as follows:
\begin{equation*}\label{gradient1}
\begin{split}
\left\|\Qm_4^{ij}\right\|_F^2
\leq \left\|\vmi{j-1} \right\|_F^2 \frac{1}{16} \|\Bm_{j+1:l}\|_F^2 \left\|\vmi{i-1} \right\|_F^2 \frac{1}{16}\|\Bm_{i+1:l}\|_F^2
\leq \frac{ 1}{2^8}\dm_{j-1}\dm_{i-1}c_r^2.
\end{split}
\end{equation*}
Note that $\dm_i\leq c_d$. Thus, we can bound $\left\|\frac{\partial^2 f(\wm,\xm)}{\partial \wmi{j} \partial \wmi{i}^T}\right\|_F^2$ as
\begin{equation*}\label{gradient1}
\begin{split}
&\left\|\frac{\partial^2 f(\wm,\xm)}{\partial \wmi{i}^T\partial \wmi{j}}\right\|_F^2\\
\leq &(l-i+3)\left(\frac{2^6}{3^8}c_y\dm_{j-1}\dm_{i-1}\dm_{i}  c_r+\sum_{k=i+1}^{l} \frac{2^{14}}{3^8} c_y \dm_{j-1}\dm_{i-1} c_r^2+ \frac{1}{2^8}c_y\dm_{j-1}c_r+ \frac{ 1}{2^8}\dm_{j-1}\dm_{i-1}c_r^2\right)\\
\leq &(l+1)\left(\frac{64}{6561}c_y c_d^3 c_r+ \frac{4096}{6561} c_y (l-2)c_d^2 c_r^2+ \frac{1}{256}c_yc_dc_r+ \frac{ 1}{256}c_dc_r^2\right).
\end{split}
\end{equation*}

{\bf\noindent{Case II: $i=j$}}.  According to the definition of $\Cm^{ij}$ and $\Cm_k^{ij}$, we have $\Qm_2^{jj}= \sum_{k=j+1}^{l} \Cm_k^{jj}$.

Similarly, we have
\begin{equation*}\label{gradient1}
\begin{split}
\left\|\frac{\partial^2 f(\wm,\xm)}{\partial \wmi{i}^T\partial \wmi{j}}\right\|_F^2=&\left\|\Qm_1^{jj}+ \Qm_2^{jj}+ \Qm_3^{jj}+ \Qm_4^{jj} +\Qm_5^{jj}\right\|_F^2 =\left\| \Qm_1^{jj}+ \sum_{k=j+1}^{l} \Cm_k^{jj}+  \Qm_4^{ij}\right\|_F^2\\
\leq&(l-j+2)\left(\left\|\Qm_1^{jj}\right\|_F^2 +\sum_{k=j+1}^{l} \left\|\Cm_k^{jj}\right\|_F^2+ \left\|\Qm_4^{jj}\right\|_F^2 \right).
\end{split}
\end{equation*}

Thus, we can bound $\left\|\Qm_1^{jj}\right\|_F^2 $ first:
\begin{equation*}\label{gradient1}
\begin{split}
\left\|\Qm_1^{jj}\right\|_F^2
\leq \left\|\vmi{j-1}\right\|_F^2 \left\|\vmi{l}-\ym\right\|_F^2 \left\|\Bm_{j+1:l}\right\|_F^2 \left\|\Im_{\dm_j}\right\|_F^2  \frac{2^6}{3^8}\left\|\vmi{j-1}\right\|_F^2 \left\|\Im_{\dm_j}\right\|_F^2
\leq \frac{2^6}{3^8}c_y \dm_{j-1}^2\dm_j^2 c_r.
\end{split}
\end{equation*}

As for $\Qm_2^{jj}$, we have
\begin{equation*}\label{gradient1}
\begin{split}
&\left\|\Cm_k^{ij}\right\|_F^2\\
\leq &\left\|\vmi{j-1}\right\|_F^2 \!\left\|\vmi{l}-\ym\right\|_F^2\! \left\|\Bm_{k+1:l}\right\|_F^2 \! \frac{1}{16}\!\left\|\Bm_{j+1:k-1}\Wm_{k}^T\right\|_F^2 \! \frac{2^6}{3^8}\!\left\|\vmi{j-1}\right\|_F^2 \! \frac{1}{16}\!\left\|\Bm_{j+1:k-1}(\Wmi{k})^T\right\|_F^2\\
=&\frac{2^6}{3^8} c_y \dm_{j-1}^2  \frac{1}{16^{l-k}}\Dm_{k+1:l} \frac{1}{16^{k-j-1}}\Dm_{j+1:k} \frac{1}{16^{k-j-1}}\Dm_{j+1:k}
\\
\leq &\frac{2^{14}}{3^8} c_y \dm_{j-1}^2 c_r^2.
\end{split}
\end{equation*}
Then we bound $\|\Qm_4^{jj}\|_F^2$:
\begin{equation*}\label{gradient1}
\begin{split}
\left\|\Qm_4^{jj}\right\|_F^2
\leq \left\|\vmi{j-1} \right\|_F^2 \frac{1}{16} \|\Bm_{j+1:l}\|_F^2 \left\|\vmi{j-1} \right\|_F^2 \frac{1}{16} \|\Bm_{j+1:l}\|_F^2
\leq \frac{1}{2^8} \dm_{j-1}^2 c_r^2.
\end{split}
\end{equation*}

Note that for any input, we have $c_v=\max_{j}\left\|\vmi{j-1}(\vmi{l}-\ym)^T\right\|_F^2\leq \max_{j}\|\vmi{j-1}\|_{F}^2$ $\left\|(\vmi{l}-\ym \right\|_F^2\leq c_yc_d$, where $\left\|\vmi{l}-\ym \right\|_F^2$ can be bounded by a constant $c_y$.
Thus, we can bound $\left\|\frac{\partial^2 f(\wm,\xm)}{\partial \wmi{i}^T\partial \wmi{j}}\right\|_F^2$ as

\begin{equation*}\label{gradient1}
\begin{split}
\left\|\frac{\partial^2 f(\wm,\xm)}{\partial \wmi{i}^T\partial \wmi{j}}\right\|_F^2
\leq &(l-i+3)\left(\frac{2^6}{3^8}c_y \dm_{j-1}^2\dm_j^2 c_r+\sum_{k=i+1}^{l} \frac{2^{14}}{3^8} c_y \dm_{j-1}^2 c_r^2+ \frac{1}{2^8} \dm_{j-1}^2 c_r^2\right)\\
\leq &(l+2)\left(\frac{64}{6561}c_y c_d^2\dm_{j-1} \dm_j  c_r+\frac{4096}{6561} c_y (l-1)c_d^2 c_r^2+ \frac{1}{256} c_d^2 c_r^2\right).
\end{split}
\end{equation*}

{\bf\noindent{Case III: $i<j$}}. Since $\frac{\partial^2 f(\wm,\xm)}{\partial \wm\partial \wm^T}$ is symmetrical, we have $\Qm_k^{ij}=\Qm_k^{ji}\ (k=1,\cdots,5)$. Thus, it yields
\begin{equation*}\label{gradient1}
\begin{split}
\left\|\frac{\partial^2 f(\wm,\xm)}{ \partial \wmi{i}^T\partial \wmi{j}}\right\|_F^2
\leq &(l+1)\left(\frac{64}{6561}c_y c_d^3 c_r+ \frac{4096}{6561} c_y (l-2)c_d^2 c_r^2+ \frac{1}{256}c_yc_dc_r+ \frac{ 1}{256}c_dc_r^2\right).
\end{split}
\end{equation*}

{\bf\noindent{Final result}}: Thus we can bound
\begin{equation*}\label{gradient1}
\begin{split}
\left\|\nabla^2_{\wm}f(\wm,\xm)\right\|_{\op}\leq &\left\|\nabla^2_{\wm}f(\wm,\xm)\right\|_F\\
\leq& \sqrt{(l-1)l\max_{i,j:i\neq j}\left\|\frac{\partial^2 f(\wm,\xm)}{\partial \wmi{j} \partial \wmi{i}^T}\right\|_F^2+ \sum_{j=1}^{l} \left\|\frac{\partial^2 f(\wm,\xm)}{\partial \wmi{j} \partial \wmi{i}^T}\right\|_F^2}\\
\leq& \left(\!(l-1)l (l\!+\!1)\left(\frac{64}{6561}c_y c_d^3 c_r\!+\! \frac{4096}{6561} c_y (l-2)c_d^2 c_r^2\!+\! \frac{1}{256}c_yc_dc_r+ \frac{ 1}{256}c_dc_r^2\!\right)\right.\\
&\quad \left.+(l+2)\left(\frac{64}{6561}c_y c_d^2 d c_r+\frac{4096}{6561} c_y (l-1)lc_d^2 c_r^2+ \frac{1}{256} lc_d^2 c_r^2\right) \right)^{\frac{1}{2}}\\
\leq& \sqrt{c_{s_1}c_rc_d^2 l^2\left(c_{s_2}c_d^2+l^2c_r\right)},
\end{split}
\end{equation*}
where $c_{s_1}$ and $c_{s_2}$ are two constants.

Since $\left\|\nabla^2_{\wm}f(\wm,\xm)\right\|_{\op}\leq \left\|\nabla^2_{\wm}f(\wm,\xm)\right\|_{F}$, we know that the gradient $\nabla_{\wm} f(\wm,\xm)$ is $\varsigma$-Lipschitz, where $\varsigma=\sqrt{c_{s_1}c_rc_d^2 l^2\left(c_{s_2}c_d^2+l^2c_r\right)}$.

On the other hand, 
since for any input $\xm$, $\sigma(\xm)$ belongs to $[0,1]$, the values of the entries of $\nabla^3_{\wm}f(\wm,\xm)$ can be bounded. Thus, we can bound
 $$\|\nabla^3_{\wm}f(\wm,\xm)\|_{\op} = \sup_{\|\lam\|_2\leq 1}\lr\lam^{\otimes^3}, \nabla^3_{\wm}f(\wm,\xm)\rl= [\nabla^3_{\wm}f(\wm,\xm)]_{ijk}\lam_i\lam_j\lam_k\leq \xi<+\infty.$$
 We complete the proof.
\end{proof}

\subsubsection{Proof of Lemma~\ref{exp fsd_gsfasradient}}
For convenience, we first give the computation of some gradients.
\begin{lem}\label{sig_gdgsasdfkjl45asgient}
Assume the activation functions in deep neural network are sigmoid functions. Then we can compute the gradients $\frac{\partial \umi{j}}{\partial \umi{1}}$ and $\frac{\partial \vmi{j}}{\partial \umi{1}}$ as
\begin{equation*}\label{gsfasgsdgsadasgaradient1}
\begin{split}
&\frac{\partial \umi{j}}{\partial \umi{1}}
=\left(\mcode{G}(\umi{1})\Am_{2}\cdots\Am_{j-1}(\Wm^j)^T \right)^T \in\Rss{\dm_j}{\dm_1}, \ (j>1).\\
&\frac{\partial \vmi{j}}{\partial \umi{1}}
= \left(\mcode{G}(\umi{1})\Am_{2}\cdots\Am_{j} \right)^T\in\Rss{\dm_j}{\dm_{1}}, \ (j>1).
\end{split}
\end{equation*}
\end{lem}

It should be pointed out that the proof of Lemma~\ref{sig_gdgsasdfkjl45asgient} can be founded  Sec.~\ref{definitiondsfadf}.

\begin{proof}[\hypertarget{lemma34}{Proof}]
To prove our conclusion, we have two steps: computing $\nabla_{\xm}\nabla_{\wm} f(\wm,\xm)$ and bounding its operation norm.

{\bf\noindent{Step 1. Compute $\nabla_{\xm} \nabla_{\wm}f(\wm,\xm)$}}:

We first consider the computation of $\frac{\partial^2 f(\wm,\xm)}{\partial \xm^T\partial \wmi{j} }$:
\begin{equation*}\label{gradient1}
\begin{split}
\frac{\partial^2 f(\wm,\xm)}{ \partial \xm^T\partial \wmi{j}}=&\frac{\partial\left(\vect{\left(\mcode{G}(\umi{j}) \Am_{j+1}\Am_{j+2}\cdots \Am_{l}(\vmi{l}-\ym)\right) (\vmi{j-1})^T}\right)}{\partial \xm^T}.
\end{split}
\end{equation*}
Recall that we define
\begin{equation*}\label{gradient1}
\begin{split}
&\Am_i=(\Wmi{i})^T\mcode{G}(\umi{i})\in\Rss{\dm_{i-1}}{\dm_i}.\\
&\Bm_{s:t}=\Am_s\Am_{s+1}\cdots\Am_{t}\in\Rss{\dm_{s-1}}{\dm_t},\ (s\leq t)\quad \text{and}\quad \Bm_{s:t}=\Im ,\ (s> t).\\
\end{split}
\end{equation*}
Then we have
\begin{equation*}\label{gradient1}
\begin{split}
\frac{\partial^2 f(\wm,\xm)}{ \partial \xm^T\partial \wmi{j}}=& \left(\vmi{j-1}(\vmi{l}-\ym)^T\Bm_{j+1:l}^T\right)\otimes \left(\Im_{\dm_{j}}\right)\frac{\partial \vect{\mcode{G}(\umi{j})}}{\partial \xm^T} (\triangleq \Qm_1^{j})\\
&+\!\!\sum_{k=j+1}^{l}\!\! \left(\vmi{j-1}(\vmi{l}-\ym)^T\Bm_{k+1:l}^T\right)\!\otimes\! \left(\mcode{G}(\umi{j})\Bm_{j+1:k-1}\Wm_{k}^T\right) \frac{\partial \vect{\mcode{G}(\umi{k})}}{\partial \xm^T}(\triangleq \Qm_2^{j})\\
&+\vmi{j-1} \otimes \left(\mcode{G}(\umi{j})\Bm_{j+1:l}\right) \frac{\partial (\vmi{l}-\ym) }{\partial \xm^T} (\triangleq \Qm_3^{j})\\
&+\Im_{\dm_{j-1}}\otimes \left(\mcode{G}(\umi{j})\Bm_{j+1:l}(\vmi{l}-\ym)\right) \frac{\partial \vmi{j-1}}{\partial \xm^T}(\triangleq \Qm_4^{j})\\
\end{split}
\end{equation*}
By using Lemma~\ref{sig_gdgsasdfkjl45asgient}, we can compute $\Qm_1^{ij}$ as
\begin{equation*}\label{gsfasgsdgsadasgaradient1}
\begin{split}
\frac{\partial \vect{\mcode{G}(\umi{k})}}{\partial \xm^T}=\frac{\partial\vect{\mcode{G}(\umi{k})}}{\partial \umi{k}}\frac{\partial \umi{k}}{\partial \xm^T}= \Pm_k \left(\mcode{G}(\umi{1})\Bm_{2:k-1} (\Wm^k)^T \right)^T.
\end{split}
\end{equation*}
Thus, we have
\begin{equation*}\label{gsfasgsdgsadasgaradient1}
\begin{split}
\Qm_1^{j}=&\left(\vmi{j-1}(\vmi{l}-\ym)^T\Bm_{j+1:l}^T\right)\otimes \Im_{\dm_{j}} \frac{\partial \vect{\mcode{G}(\umi{j})}}{\partial \xm^T}\\
=&\left(\left(\vmi{j-1}(\vmi{l}-\ym)^T\Bm_{j+1:l}^T\right)\otimes \Im_{\dm_{j}}\right) \Pm_k \left(\mcode{G}(\umi{1})\Bm_{2:k-1}(\Wm^k)^T \right)^T.
\end{split}
\end{equation*}

As for $\Qm_2^j$, we also can utilize Lemma~\ref{sig_gdgsasdfkjl45asgient} to compute it:
\begin{equation*}\label{gsfasgsdgsadasgaradient1}
\begin{split}
\Qm_2^{j}=&\sum_{k=j+1}^{l} \left(\vmi{j-1}(\vmi{l}-\ym)^T\Bm_{k+1:l}^T\right)\otimes \left(\mcode{G}(\umi{j})\Bm_{j+1:k-1}\Wm_{k}^T\right) \frac{\partial \vect{\mcode{G}(\umi{k})}}{\partial \xm^T}\\
=&\sum_{k=i+1}^{l}\!\! \left(\left(\vmi{j-1}(\vmi{l}-\ym)^T\Bm_{k+1:l}^T\right)\!\otimes\! \left(\mcode{G}(\umi{j})\Bm_{j+1:k-1}\Wm_{k}^T\right) \right)\Pm_k \left(\mcode{G}(\umi{1})\Bm_{2:k-1}(\Wm^k)^T \right)^T.
\end{split}
\end{equation*}
Then we consider $\Qm_3^{ij}$.
\begin{equation*}\label{gsfasgsdgsadasgaradient1}
\begin{split}
\Qm_3^{j}=\vmi{j-1} \otimes \left(\mcode{G}(\umi{j})\Bm_{j+1:l}\right) \frac{\partial (\vmi{l}-\ym) }{\partial \xm^T}=\left( \vmi{j-1} \otimes \left(\mcode{G}(\umi{j})\Bm_{j+1:l}\right) \right) \left(\mcode{G}(\umi{1})\Bm_{2:l} \right)^T.
\end{split}
\end{equation*}
$\Qm_4^{j}$ can be computed as follows:
\begin{equation*}\label{gsfasgsdgsadasgaradient1}
\begin{split}
\Qm_4^{j}\!=\!\Im_{\dm_{j-1}}\!\otimes\! \left(\!\mcode{G}(\umi{j})\Bm_{j+1:l}(\vmi{l}\!-\!\ym)\!\right) \frac{\partial \vmi{j-1}}{\partial \xm^T} \!=\! \left(\! \Im_{\dm_{j-1}}\!\otimes\! \left(\!\mcode{G}(\umi{j})\Bm_{j+1:l}(\vmi{l}\!-\!\ym)\!\right)\! \right) \!\left(\!\mcode{G}(\umi{1})\Bm_{2:j} \!\right)^T\!\!.
\end{split}
\end{equation*}

{\bf\noindent{Step 2. Bound the operation norm of Hessian}}: We mainly use Lemma~\ref{asfashdfhfdet} to achieve this goal. From Lemma~\ref{asfashdfhfdet}, we have
\begin{myitemize}
\item[(1)] For arbitrary matrices $\Mm$ and $\Nm$ of proper size, we have
    \begin{equation*}\label{gradient1}
\begin{split}
\|\mcode{G}(\umi{i})\Mm\|_F^2\leq \frac{1}{16}\|\Mm\|_F^2\quad \text{and} \quad \|\Nm\mcode{G}(\umi{i})\|_F^2\leq \frac{1}{16}\|\Nm\|_F^2.
\end{split}
\end{equation*}
\item[(2)] For arbitrary matrices $\Mm$ and $\Nm$ of proper size, we have
\begin{equation*}\label{gradient1}
\begin{split}
\|\Pm_k\Mm\|_F^2\leq \frac{2^6}{3^8}\|\Mm\|_F^2\quad \text{and}\quad \|\Nm\Pm_k\|_F^2\leq \frac{2^6}{3^8}\|\Nm\|_F^2.
\end{split}
\end{equation*}
\item[(3)] For $\Bm_{s:t}$ and $\Dm_{s:t}$, we have
\begin{equation*}\label{gradient1}
\begin{split}
\left\|\Bm_{s:t}\right\|_F^2\leq \frac{1}{16^{t-s+1}}\Dm_{s:t}\quad \text{and} \quad \frac{1}{16^{t-s+1}}\Dm_{s:t} \leq c_r,
\end{split}
\end{equation*}
where $c_r=\max\left(\frac{r^2}{4},\left(\frac{r^2}{16}\right)^{l-1}\right)$.
\end{myitemize}
The values of entries in $\vmi{h}$ are bounded by $0\leq \sigma(\umi{i}_h)\leq 1$ which leads to $\left\|\vmi{h}\right\|_F^2 \leq \dm_h\leq c_d$, where $c_d=\max_i \dm_i$. On the other hand, since the values in $\vmi{l}$ belong to the range $[0,1]$ and $\ym$ is the label, $\|\vmi{l}-\ym\|_2^2$ can be bounded:
$$\|\vmi{l}-\ym\|_2^2\leq c_y<+\infty,$$
where $c_y$ is a universal constant.

We first define
$$\Cm_k^{j}=\left(\left(\vmi{j-1}(\vmi{l}-\ym)^T\Bm_{k+1:l}^T\right)\otimes \left(\mcode{G}(\umi{j})\Bm_{j+1:k-1}\Wm_{k}^T\right)\right) \Pm_k \left(\mcode{G}(\umi{1})\Bm_{2:k-1}(\Wm^k)^T \right)^T.$$
Then we have $\Qm_2^{j}= \sum_{k=j+1}^{l} \Cm_k^{j}$. So we have
\begin{equation*}\label{gradient1}
\begin{split}
\left\|\frac{\partial^2 f(\wm,\xm)}{ \partial \xm^T\partial \wmi{j}}\right\|_F^2=&\left\|\Qm_1^{j}+ \Qm_2^{j}+ \Qm_3^{j}+ \Qm_4^{j}\right\|_F^2=\left\|\Qm_1^{j}+\sum_{k=j+1}^{l} \Cm_k^{j}+ \Qm_3^{j}+ \Qm_4^{j}\right\|_F^2\\
=&(l-j+3)\left(\left\|\Qm_1^{j}\right\|_F^2+\sum_{k=j+1}^{l} \left\|\Cm_k^{j}\right\|_F^2+ \left\|\Qm_3^{j}\right\|_F^2+ \left\|\Qm_4^{j}\right\|_F^2\right).
\end{split}
\end{equation*}
Then we bound each term separately:
\begin{equation*}\label{gradient1}
\begin{split}
\left\|\Qm_1^{j}\right\|_F^2
\leq \left\|\vmi{j-1}\right\|_F^2\left\|\vmi{l}-\ym\right\|_F^2 \left\|\Bm_{j+1:l}\right\|_F^2  \left\|\Im_{\dm_{j}} \right\|_F^2 \frac{2^6}{3^8}\frac{1}{16}\left\|\Bm_{2:k-1}(\Wm^k)^T\right\|_F^2
\leq \frac{2^6}{3^8} c_y \dm_{j-1}\dm_j c_r^2.
\end{split}
\end{equation*}
Similarly, we bound $\left\|\Cm_k^{j}\right\|_F^2$:
\begin{equation*}\label{gradient1}
\begin{split}
\left\|\Cm_k^{j}\right\|_F^2
=&\left\|\vmi{j-1}\right\|_F^2\left\|\vmi{l}-\ym\right\|_F^2 \left\|\Bm_{k+1:l}\right\|_F^2  \frac{1}{16}\left\|\Bm_{j+1:k-1}\Wm_{k}^T\right\|_F^2 \frac{2^6}{3^8} \frac{1}{16}\left\|\Bm_{2:k-1}(\Wmi{k})^T\right\|_F^2 \\
=&\frac{2^6}{3^8} c_y \dm_{j-1} \frac{1}{16^{l-k}}\Dm_{k+1:l} \frac{1}{16^{k-j-1}}\Dm_{j+1:k} \frac{1}{16^{k-1}}\Dm_{2:k}\\
\leq &\frac{2^6}{3^8} c_y \dm_{j-1} c_r^2.
\end{split}
\end{equation*}

We also bound $\left\|\Qm_3^{ij}\right\|_F^2$ as
\begin{equation*}\label{gradient1}
\begin{split}
\left\|\Qm_3^{ij}\right\|_F^2
\leq \left\|\vmi{j-1}\right\|_2^2 \frac{1}{16}\left\|\Bm_{j+1:l}\right\|_F^2  \frac{1}{16}\left\|\Bm_{2:l}\right\|_F^2
\leq \frac{1}{2^8}\dm_{j-1}c_r^2.
\end{split}
\end{equation*}
Finally, we bound $\left\|\Qm_4^{j}\right\|_F^2$ as follows:
\begin{equation*}\label{gradient1}
\begin{split}
\left\|\Qm_4^{j}\right\|_F^2
=\left\|\Im_{\dm_{j-1}} \right\|_F^2 \frac{1}{16} \|\Bm_{j+1:l}\|_F^2 \left\|\vmi{l}-\ym\right\|_F^2 \frac{1}{16}\|\Bm_{2:j}\|_F^2
\leq \frac{1}{2^8}c_y \dm_{j-1} c_r.
\end{split}
\end{equation*}
Since $c_d=\max_i \dm_{i}$, we can bound $\left\|\frac{\partial^2 f(\wm,\xm)}{\partial \wmi{j} \partial \xm^T}\right\|_F^2$ as
\begin{equation*}\label{gradient1}
\begin{split}
\left\|\frac{\partial^2 f(\wm,\xm)}{ \partial \xm^T\partial \wmi{j}}\right\|_F^2
\!\leq &(l-j+3)\!\left(\!\frac{2^6}{3^8} c_y \dm_{j-1}\dm_j c_r^2+\!\!\sum_{k=j+1}^{l} \frac{2^6}{3^8} c_y \dm_{j-1} c_r^2+ \frac{1}{2^8}c_y \dm_{j-1} c_r +\frac{1}{2^8}c_y \dm_{j-1} c_r\!\right)\\
\leq &(l+2)\left(\frac{2^6}{3^8} c_y \dm_{j-1}\dm_j c_r^2+ \frac{2^6}{3^8} c_y (l-1)c_d c_r^2+ \frac{1}{2^7}c_y c_d c_r\right).
\end{split}
\end{equation*}

{\bf\noindent{Final result}}: Thus we can bound
\begin{equation*}\label{gradient1}
\begin{split}
\left\|\nabla_{\wm}\nabla_{\xm}f(\wm,\xm)\right\|_{\op}\leq &\left\|\nabla_{\wm}\nabla_{\xm}f(\wm,\xm)\right\|_F\\
\leq& \sqrt{\sum_{j=1}^{l}\left\|\frac{\partial^2 f(\wm,\xm)}{\partial \wmi{j} \partial \xm^T}\right\|_F^2}\\
\leq& \sqrt{\sum_{j=1}^{l} (l+2)\left(\frac{2^6}{3^8} c_y \dm_{j-1}\dm_j c_r^2+ \frac{2^6}{3^8} c_y (l-1)c_d c_r^2+ \frac{1}{2^7}c_y c_d c_r\right)}\\
\overset{\text{\ding{172}}}{\leq} & \sqrt{\frac{2^6}{3^8}c_yc_r(l+2)\left(dc_r+(l-1)lc_dc_r+lc_d\right)},
\end{split}
\end{equation*}
where \ding{172} holds since we have $d=\sum_{j=1}^{l}\dm_{j-1} \dm_{j}$. The proof is completed.
\end{proof}

\subsubsection{Proof of Lemmas~\ref{thm:gradsgfdsgdient12}~and~ \ref{thm:garadifghtent12}}

\begin{lem}~\cite{Alessandro,RigolletSnote}\label{Lip_gaus}
Let $(\xm_1,\cdots,\xm_k)$ be a vector of \textit{i.i.d.} Gaussian variables from $\mathcal{N}(0, \tau^2)$ and let $f\,:\,\mathbb{R}^{\dm_0}\to\mathbb{R}$ be $L$-Lipschitz. Then the variable $f(\xm)-\EE f(\xm)$ is sub-Gaussian. That is, we have
\begin{equation*}
\Pro\left(f(\xm)-\EE f(\xm)>t\right)\leq  \exp\left(-\frac{t^2}{2L^2\tau^2}\right),\quad (\forall t\geq 0),
\end{equation*}
or
\begin{equation*}
\EE\left( \lambda(f(\xm)-\EE f(\xm))\right)\leq \exp\left(4\lambda^2 L^2\tau^2\right),\quad (\forall \lambda \geq 0).
\end{equation*}
Remarkably, this is a dimension free inequality.
\end{lem}

\begin{proof}[\hypertarget{lemma36}{Proof} of Lemma~\ref{thm:gradsgfdsgdient12}]
We first define a function $g(\xm)=\zm^T\nabla_{\wm}f(\wm,\xm)$ where $\zm\in\Rs{d}$ is a constant vector. Then we have $\nabla_{\xm} g(\xm)=\nabla_{\xm} \left(\zm^T\nabla_{\wm}f(\wm,\xm)\right)=\nabla_{\xm}\nabla_{\wm} f(\wm,\xm)\zm$. Then by Lemma~\ref{exp fsd_gsfasradient}, we can obtain $\|\nabla_{\xm} g(\xm)\|_2=\|\nabla_{\xm}\nabla_{\wm}f(\wm,\xm)\zm\|_2\leq \beta\|\zm\|_2$, where $\beta=\sqrt{\frac{2^6}{3^8}c_yc_r(l+2)\left(dc_r+(l-1)lc_dc_r+lc_d\right)}$ in which $c_y$, $c_d$ and $c_r$ are defined in Lemma~\ref{exp fsd_gsfasradient}.
This means $g(\xm)$ is $\beta\|\zm\|_2$-Lipschitz. Thus, by Lemma~\ref{Lip_gaus}, we have
\begin{equation*}
\EE\left(s \lr\zm,\nabla_{\wm}f(\wm,\xm)-\EE \nabla_{\wm}f(\wm,\xm)\rl \right)=\EE\left(s\left(g(\xm)-\EE g(\xm)\right)\right)\leq \exp\left( 4s^2\beta^2\|\zm\|_2^2\tau^2 \right).
\end{equation*}

Let $\lam=s\zm$. This further gives
\begin{equation*}
\EE\left(\lr\lam ,\nabla_{\wm}f(\wm,\xm)-\EE \nabla_{\wm}f(\wm,\xm)\rl \right)\leq \exp\left(4\beta^2\tau^2\|\lam \|_2^2\right),
\end{equation*}
which means $\lr\lam,\nabla_{\wm}f(\wm,\xm)-\EE \nabla_{\wm}f(\wm,\xm)\rl$ is $8\beta^2\tau^2$-sub-Gaussian.
\end{proof}

\begin{proof}[\hypertarget{lemma37}{Proof} of Lemma~\ref{thm:garadifghtent12}]
We first define a function $h(\xm)=\zm^T\nabla^2_{\wm}f(\wm,\xm)\zm$ where $\zm\in\SS^{d-1}$, \textit{i.e.} $\|\zm\|_2=1$. Then $h(\wm)$ is a $\gamma$-Lipschitz function, where $\gamma=\|\nabla_{\xm}\nabla^2_{\wm}f(\wm,\xm)\|_{op}$. Note that since the sigmoid function is infinitely differentiable function, $\nabla_{\xm}\nabla^2_{\wm}f(\wm,\xm)$ exists. Also since for any input $x$, $\sigma(x)$ belongs to $[0,1]$. Thus, the values of the entries in $\nabla_{\xm}\nabla^2_{\wm}f(\wm,\xm)$ can be bounded. So according to the definition of the operation norm of a 3-way tensor, the operation norm of $\nabla_{\xm}\nabla^2_{\wm}f(\wm,\xm)$ can be bounded by a constant. Without loss of generality, let $\|\nabla_{\xm}\nabla^2_{\wm}f(\wm,\xm)\|_{\op}\leq\gamma<+\infty$.
Thus, by Lemma~\ref{Lip_gaus}, we have
\begin{equation*}
\EE\left(s \lr\zm,\left(\nabla^2_{\wm}f(\wm,\xm)-\EE \nabla^2_{\wm}f(\wm,\xm)\right)\zm\rl \right)=\EE\left(s\left(h(\xm)-\EE h(\xm)\right)\right)\leq \exp\left(\frac{8s^2\gamma^2 \tau^2}{2}\right).
\end{equation*}
This means that the hessian of the loss evaluated on a unit vector is $8\gamma^2\tau^2$-sub-Gaussian.
\end{proof}

\subsubsection{Proof of Lemma~\ref{thm:uniformconvergence1fds_sig}}
\begin{proof}[\hypertarget{uniforadfmlocasig}{Proof}]

Recall that  the weight of each layer has magnitude bound separately, \textit{i.e.} $\|\wmi{j}\|_2\leq r$. So here we separately assume $\wm_{\epsilon}^j=\{\wm_1^j,\cdots,\wm_{\nee^j}^j\}$ is the $\epsilon/l$-covering net of the ball $\Bi{\dm_j\dm_{j-1}}{r}$ which corresponds to the weight $\wmi{j}$ of the $j$-th layer. Let $\nee^j$ be the $\epsilon/l$-covering number. By $\epsilon$-covering theory in \cite{VRMT}, we can have $\nee^j\leq (3rl/\epsilon)^{\dm_j\dm_{j-1}}$. Let $\wm\in\Omega$ be an arbitrary vector. Since $\wm=[\wmi{1},\cdots,\wmi{l}]$ where $\wmi{j}$ is the weight of the $j$-th layer, we can always find a vector $\wm^j_{k_j}$ in $\wm_\epsilon^j$ such that $\|\wmi{j}-\wm^j_{k_j}\|_2\leq \epsilon/l$. For brevity, let $j_w\in[\nee^j]$  denote the index of $\wm^j_{k_j}$ in $\epsilon$-net $\wm_\epsilon^j$. Then let $\wm_{\kt}=[\wm^j_{k_1};\cdots;\wm^j_{k_j};\cdots;\wm^j_{k_l}]$. In this case, we can always find a vector $\wm_{\kt}$ such that $\|\wm-\wm_{\kt}\|_2\leq \epsilon$.
Accordingly, we can decompose $\lv\nabla^2 \Jhn(\wm)-\nabla^2\Jm(\wm)\rv_{\op}$ as follows:
\begin{equation*}
\begin{aligned}
&\lv\nabla^2 \Jhn(\wm)-\nabla^2\Jm(\wm)\rv_{\op}\\
=& \lv \frac{1}{n}\sum_{i=1}^n \nabla^2 f(\wm,\xmi{i})-\EE(\nabla^2 f(\wm,\xm))\rv_{\op}\\
=&\Bigg\| \frac{1}{n}\sum_{i=1}^n \left(\nabla^2 f(\wm,\xmi{i})-\nabla f(\wm_{\kt},\xmi{i})\right)+\frac{1}{n}\sum_{i=1}^n \nabla^2 f(\wm_{\kt},\xmi{i})-\EE(\nabla^2 f(\wm_{\kt},\xm)) \\
& \ \ +\EE(\nabla^2 f(\wm_{\kt},\xm))-\EE(\nabla^2 f(\wm,\xm))\Bigg\|_{\op}\\
\leq &\lv \frac{1}{n}\sum_{i=1}^n \left(\nabla^2 f(\wm,\xmi{i})-\nabla^2 f(\wm_{\kt},\xmi{i})\right)\rv_{\op}+\lv\frac{1}{n}\sum_{i=1}^n \nabla^2 f(\wm_{\kt},\xmi{i})-\EE(\nabla^2 f(\wm_{\kt},\xm))\rv_{\op}\\
&\ \ +\Bigg\|\EE(\nabla^2 f(\wm_{\kt},\xm))-\EE(\nabla^2 f(\wm,\xm))\Bigg\|_{\op}.
\end{aligned}
\end{equation*}
Here we also define four events $\Em_0$, $\Em_1$, $\Em_2$ and $\Em_3$ as
\begin{equation*}
\begin{aligned}
&\Em_0=\left\{\sup_{\wm\in\Omega}\lv\nabla^2\Jhn(\wm)-\nabla^2\Jm(\wm)\rv_{\op}\geq t\right\},\\
&\Em_1=\left\{\sup_{\wm\in\Omega}\lv \frac{1}{n}\sum_{i=1}^n \left(\nabla^2 f(\wm,\xmi{i})-\nabla^2 f(\wm_{\kt},\xmi{i})\right)\rv_{\op}\geq \frac{t}{3}\right\},\\
&\Em_2=\left\{\sup_{j_w\in[\nee^j], j=[l]}\lv\frac{1}{n}\sum_{i=1}^n \nabla^2 f(\wm_{\kt},\xmi{i})-\EE(\nabla^2 f(\wm_{\kt},\xm))\rv_{\op}\geq \frac{t}{3}\right\},\\
&\Em_3=\left\{\sup_{\wm\in\Omega}\lv\EE(\nabla^2 f(\wm_{\kt},\xm))-\EE(\nabla^2 f(\wm,\xm))\rv_{\op}\geq \frac{t}{3}\right\}.
\end{aligned}
\end{equation*}
Accordingly, we have
\begin{equation*}
\begin{aligned}
\Pro\left(\Em_0\right)\leq \Pro\left(\Em_1\right)+\Pro\left(\Em_2\right)+\Pro\left(\Em_3\right).
\end{aligned}
\end{equation*}
So we can respectively bound $\Pro\left(\Em_1\right)$, $\Pro\left(\Em_2\right)$ and $\Pro\left(\Em_3\right)$ to bound $\Pro\left(\Em_0\right)$.

{\bf\noindent{Step 1. Bound $\Pro\left(\Em_1\right)$}}: We first bound $\Pro\left(\Em_1\right)$ as follows:
\begin{equation*}
\begin{aligned}
\Pro\left( \Em_1 \right) = & \Pro \left( \sup_{\wm \in \Omega} \lv \frac{1}{n}\sum_{i=1}^n \left(\nabla^2 f(\wm,\xmi{i})-\nabla^2 f(\wm_{\kt},\xmi{i})\right)\rv_2\geq \frac{t}{3} \right) \\
\overset{\text{\ding{172}}}{\leq} &\frac{3}{t}\EE \left(  \sup_{\wm \in \Omega}\lv \frac{1}{n}\sum_{i=1}^n\left(\nabla^2 f(\wm,\xmi{i})-\nabla^2 f(\wm_{\kt},\xmi{i})\right)\rv_2\right) \\
\leq &\frac{3}{t}\EE \left(  \sup_{\wm \in \Omega}\frac{\lv \frac{1}{n}\sum_{i=1}^n\left(\nabla^2 f(\wm,\xmi{i})-\nabla^2 f(\wm_{\kt},\xmi{i})\right)\rv_2}{\lv\wm-\wm_{\kt}\rv_2} \sup_{\wm\in\Omega} \lv\wm-\wm_{\kt}\rv_2\right) \\
\leq &\frac{3\epsilon}{t}\EE \left(  \sup_{\wm \in \Omega} \lv \frac{1}{n}\sum_{i=1}^{n}\nabla^3 f(\wm,\xmi{i}) \rv_{\op}\right)\\
\overset{\text{\ding{173}}}{\leq} &\frac{3\xi\epsilon}{t},\\
\end{aligned}
\end{equation*}
where \ding{172} holds since by Markov inequality and \ding{173} holds because of Lemma~\ref{sig_gradient}.

Therefore, we can set
\begin{equation*}
t \geq \frac{6\xi \epsilon}{\varepsilon}.
\end{equation*}
Then we can bound $\Pro(\Em_1)$:
\begin{equation*}
\Pro(\Em_1)\leq \frac{\varepsilon}{2}.
\end{equation*}

{\bf\noindent{Step 2. Bound $\Pro\left(\Em_2\right)$}}: By Lemma~\ref{lem:opnorm}, we know that for any matrix $\Xm\in\Rss{d}{d}$, its operator norm can be computed as
\begin{equation*}
\|\Xm\|_{\op} \leq \frac{1}{1-2\epsilon} \sup_{\lam \in \lam_\epsilon} \left|\lr \lam,\Xm\lam\rl\right|.
\end{equation*}
where $\lam_\epsilon=\{\lam_1,\dots,\lam_{\kt}\}$ be an $\epsilon$-covering net of $\Bi{d}{1}$.

Let $\lam_{1/4}$ be the $\frac{1}{4}$-covering net of $\Bi{d}{1}$. Recall that we use $j_w$ to denote the index of $\wm^j_{k_j}$ in $\epsilon$-net $\wm_\epsilon^j$ and we have $j_w\in[\nee^j],\ (\nee^j\leq (3rl/\epsilon)^{\dm_j\dm_{j-1}})$. Then we can bound $\Pro\left(\Em_2\right)$ as follows:
\begin{equation*}
\begin{aligned}
\Pro\!\left( \Em_2 \right) \!= & \Pro \left( \sup_{j_w\in[\nee^j], j=[l]} \lv\frac{1}{n}\sum_{i=1}^n \nabla^2 f(\wm_{\kt},\xmi{i})-\EE(\nabla^2 f(\wm_{\kt},\xm))\rv_2\geq \frac{t}{3} \right) \\
= & \Pro \left( \sup_{j_w\in[\nee^j], j=[l], \lam\in\lam_{1/4}} 2\left|\lr\lam, \left(\frac{1}{n}\sum_{i=1}^n \nabla^2 f(\wm_{\kt},\xmi{i})-\EE\left(\nabla^2 f(\wm_{\kt},\xm)\right)\right)\lam\rl\right|\geq \frac{t}{3} \right) \\
\leq & 12^d\left(\frac{3lr}{\epsilon}\right)^{\sum_j\dm_j\dm_{j-1}}\!\! \!\!\!\!\!\!\!\! \!\!\!\!\! \sup_{j_w\in[\nee^j], j=[l], \lam\in\lam_{1/4}}\!\!\! \!\!\Pro\! \left( \left| \frac{1}{n}\sum_{i=1}^n\lr\lam, \left(\nabla^2 f(\wm_{\kt},\xmi{i})\!-\!\EE\left(\nabla^2 f(\wm_{\kt},\xm)\right)\right)\lam\rl\right|\!\geq\! \frac{t}{6}\!\right).
\end{aligned}
\end{equation*}
Since by Lemma~\ref{thm:garadifghtent12}, $\lr\lam,\left(\nabla^2_{\wm}f(\wm,\xm)-\EE \nabla^2_{\wm}f(\wm,\xm)\right)\lam\rl$ where $\lam\in\Bi{d}{1}$ is $8\gamma^2\tau^2$-sub-Gaussian, \textit{i.e.}
\begin{equation*}
\EE\left(s \lr\lam,\left(\nabla^2_{\wm}f(\wm,\xm)-\EE \nabla^2_{\wm}f(\wm,\xm)\right)\lam\rl \right)\leq \exp\left(\frac{8s^2\gamma^2 \tau^2}{2}\right).
\end{equation*}
Thus, $\frac{1}{n}\sum_{i=1}^n \lr\lam,\left(\nabla^2_{\wm}f(\wm,\xm)-\EE \nabla^2_{\wm}f(\wm,\xm)\right)\lam\rl$ is $8\gamma^2\tau^2/n$-sub-Gaussian random variable. So we can obtain
\begin{equation*}
\begin{aligned}
\Pro \left(\left|\frac{1}{n} \sum_{i=1}^n\lr\ym,\left(\nabla^2_{\wm}f(\wm,\xm)-\EE \nabla^2_{\wm}f(\wm,\xm)\right)\ym\rl\right|\geq \frac{t}{6} \right)\leq 2\exp\left(-\frac{n t^2}{72\gamma^2\tau^2}\right).
\end{aligned}
\end{equation*}
Note $d=\sum_j\dm_j\dm_{j-1}$. Then the probability of $\Em_2$ is upper bounded as
\begin{equation*}
\begin{aligned}
\Pro\left(\Em_2\right)\leq 2\exp\left(-\frac{n t^2}{72\gamma^2\tau^2}+ d\log\left(\frac{36lr}{\epsilon}\right)\right).
\end{aligned}
\end{equation*}
Thus, if we set
\begin{equation*}
\begin{aligned}
t\geq \gamma \tau\sqrt{\frac{72\left(d\log(36lr/\epsilon)+\log(4/\varepsilon)\right)}{n}},
\end{aligned}
\end{equation*}
then we have
\begin{equation*}
\Pro\left(\Em_2\right)\leq \frac{\varepsilon}{2}.
\end{equation*}

{\bf\noindent{Step 3. Bound $\Pro\left(\Em_3\right)$}}:
We first bound $\Pro\left(\Em_3\right)$ as follows:
\begin{equation*}
\begin{aligned}
\Pro\left( \Em_3 \right) = & \Pro \left( \sup_{\wm\in\Omega} \lv\EE(\nabla^2 f(\wm_{\kt},\xm))-\EE(\nabla^2 f(\wm,\xm))\rv_2 \geq \frac{t}{3} \right) \\
\leq & \Pro \left(\EE \sup_{\wm\in\Omega} \lv(\nabla^2 f(\wm_{\kt},\xm)- \nabla^2 f(\wm,\xm)\rv_2 \geq \frac{t}{3} \right) \\
=& \Pro \left(\EE \sup_{\wm\in\Omega} \frac{\lv \left(\nabla^2 f(\wm,\xm)-\nabla^2 f(\wm_{\kt},\xm)\right)\rv_2}{\lv\wm-\wm_{\kt}\rv_2} \sup_{\wm\in\Omega} \lv\wm-\wm_{\kt}\rv_2 \geq \frac{t}{3} \right) \\
\leq & \Pro \left(\EE \sup_{\wm\in\Omega} \lv \nabla^3 f(\wm,\xm) \rv_{\op} \geq \frac{t}{3} \right) \\
\leq&  \Pro\left(\xi\epsilon  \geq \frac{t}{3}\right).
\end{aligned}
\end{equation*}
We set $\epsilon$ enough small such that $\xi  \epsilon  < t/3$ always holds. Then it yields $\Pro\left( \Em_3 \right)=0$.

{\bf\noindent{Step 4. Final result}}: To ensure $\Pro(\Em_0)\leq \varepsilon$, we just set $\epsilon=36r /n$ and
\begin{equation*}
\begin{split}
t\!\geq\! \max\left(\!\frac{6\xi\epsilon}{\varepsilon},\gamma \tau\sqrt{\frac{72\left(d\log(36lr/\epsilon)\!+\!\log(4/\varepsilon)\right)}{n}} \right)\!=\!\max\left(\!\frac{108\xi r}{n\varepsilon}, c_4'\gamma \tau \sqrt{\frac{d\log(nl)\!+\!\log(4/\varepsilon)}{n}} \right).
\end{split}
\end{equation*}
Therefore, there exists such two universal constants $c_{m'}$ and $c_m$ such that if $n\geq \frac{c_{m'}\xi^2 r^2}{\gamma^2\tau^2d\log(l)}$, then
\begin{equation*}
\sup_{\wm\in\Omega}\!\left\| \nabla^2\Jhn(\wm)\!-\!\nabla^2\Jm(\wm)\right\|_{\op}\!\leq\! c_m\gamma\tau\sqrt{\frac{d\log(nl)\!+\!\log(4/\varepsilon)}{n}}
\end{equation*}
holds with probability at least $1-\varepsilon$.
\end{proof}

\subsection{Proofs of Main Theories}\label{saf3we543ydf}
\subsubsection{Proof of Theorem~\ref{thm:stability_sig}}
\begin{proof}[\hypertarget{lemmaunisigloss}{Proof}]
Recall that  the weight of each layer has magnitude bound separately, \textit{i.e.} $\|\wmi{j}\|_2\leq r$. So here we separately assume $\wm_{\epsilon}^j=\{\wm_1^j,\cdots,\wm_{\nee^j}^j\}$ is the $\epsilon/l$-covering net of the ball $\Bi{\dm_j\dm_{j-1}}{r}$ which corresponds to the weight $\wmi{j}$ of the $j$-th layer. Let $\nee^j$ be the $\epsilon/l$-covering number. By $\epsilon$-covering theory in \cite{VRMT}, we can have $\nee^j\leq (3rl/\epsilon)^{\dm_j\dm_{j-1}}$. Let $\wm\in\Omega$ be an arbitrary vector. Since $\wm=[\wmi{1},\cdots,\wmi{l}]$ where $\wmi{j}$ is the weight of the $j$-th layer, we can always find a vector $\wm^j_{k_j}$ in $\wm_\epsilon^j$ such that $\|\wmi{j}-\wm^j_{k_j}\|_2\leq \epsilon/l$. For brevity, let $j_w\in[\nee^j]$  denote the index of $\wm^j_{k_j}$ in $\epsilon$-net $\wm_\epsilon^j$. Then let $\wm_{\kt}=[\wm^j_{k_1};\cdots;\wm^j_{k_j};\cdots;\wm^j_{k_l}]$. This means that  we can always find a vector $\wm_{\kt}$ such that $\|\wm-\wm_{\kt}\|_2\leq \epsilon$.
Accordingly, we can decompose $\left| \Jhn(\wm)-\Jm(\wm)\right|$ as
\begin{equation*}
\begin{aligned}
&\lvs\Jhn(\wm)-\Jm(\wm)\rvs\! =\! \lvs \frac{1}{n}\sum_{i=1}^n f(\wm,\xmi{i})-\EE(f(\wm,\xm))\rvs \\
=&\Bigg| \frac{1}{n}\!\sum_{i=1}^n\!\! \left(f(\wm,\xmi{i})\!-\!f(\wm_{\kt},\xmi{i})\right)\!+\!\frac{1}{n}\! \sum_{i=1}^n \!\! f(\wm_{\kt},\xmi{i})\!-\!\EE f(\wm_{\kt},\xm)\!+\! \EE f(\wm_{\kt},\xm)\!-\!\EE f(\wm,\xm)\Bigg|\\
\leq &\!\lvs\frac{1}{n}\!\sum_{i=1}^n \!\! \left(f(\wm,\xmi{i})\!-\!f(\wm_{\kt},\xmi{i})\right)\rvs\!+ \!\lvs\frac{1}{n}\!\!\sum_{i=1}^n \!\! f(\wm_{\kt},\xmi{i})\!-\!\EE f(\wm_{\kt},\xm)\rvs\!+\! \Bigg|\EE f(\wm_{\kt},\xm)\!-\!\EE f(\wm,\xm)\Bigg|.
\end{aligned}
\end{equation*}
Then, we define four events $\Em_0$, $\Em_1$, $\Em_2$ and $\Em_3$ as
\begin{equation*}
\begin{aligned}
&\Em_0=\left\{\sup_{\wm\in\Omega}\lvs\Jhn(\wm)-\Jm(\wm)\rvs\geq t\right\},\\
& \Em_1=\left\{\sup_{\wm\in\Omega}\lvs \frac{1}{n}\sum_{i=1}^n \left(f(\wm,\xmi{i})-f(\wm_{\kt},\xmi{i})\right)\rvs\geq \frac{t}{3}\right\},\\
&\Em_2=\left\{\sup_{j_w\in[\nee^j], j=[l]}\lvs\frac{1}{n}\!\sum_{i=1}^n\!\! f(\wm_{\kt},\xmi{i})\!-\!\EE(f(\wm_{\kt},\xm))\rvs\!\geq \! \frac{t}{3}\right\},\\ &\Em_3=\left\{\sup_{\wm\in\Omega}\Bigg|\EE(f(\wm_{\kt},\xm))\!-\!\EE(f(\wm,\xm))\Bigg|\!\geq \! \frac{t}{3}\right\}.
\end{aligned}
\end{equation*}

Accordingly, we have
\begin{equation*}
\begin{aligned}
\Pro\left(\Em_0\right)\leq \Pro\left(\Em_1\right)+\Pro\left(\Em_2\right)+\Pro\left(\Em_3\right).
\end{aligned}
\end{equation*}
So we can respectively bound $\Pro\left(\Em_1\right)$, $\Pro\left(\Em_2\right)$ and $\Pro\left(\Em_3\right)$ to bound $\Pro\left(\Em_0\right)$.

{\bf\noindent{Step 1. Bound $\Pro\left(\Em_1\right)$}}: We first bound $\Pro\left(\Em_1\right)$ as follows:
\begin{equation*}
\begin{aligned}
\Pro\left( \Em_1 \right) = & \Pro \left( \sup_{\wm \in \Omega} \lvs \frac{1}{n}\sum_{i=1}^n \left( f(\wm,\xmi{i})- f(\wm_{\kt},\xmi{i})\right)\rvs\geq \frac{t}{3} \right) \\
\overset{\text{\ding{172}}}{\leq} &\frac{3}{t}\EE \left(  \sup_{\wm \in \Omega}\lvs\frac{1}{n}\sum_{i=1}^n\left(f(\wm,\xmi{i})- f(\wm_{\kt},\xmi{i})\right)\rvs\right) \\
\leq &\frac{3}{t}\EE \left(  \sup_{\wm \in \Omega}\frac{\lvs \frac{1}{n}\sum_{i=1}^n\left( f(\wm,\xmi{i})- f(\wm_{\kt},\xmi{i})\right)\rvs}{\lv\wm-\wm_{\kt}\rv_2} \sup_{\wm\in\Omega} \lv\wm-\wm_{\kt}\rv_2\right) \\
\leq &\frac{3\epsilon}{t}\EE \left(  \sup_{\wm \in \Omega} \lv \nabla \Jhn(\wm,\xm) \rv_2\right),
\end{aligned}
\end{equation*}
where \ding{172} holds since by Markov inequality, for an arbitrary nonnegative random variable $x$, then we have
\begin{equation*}
\begin{aligned}
\Pro(x\geq t)\leq \frac{\EE(x)}{t}.
\end{aligned}
\end{equation*}

Now we only need to bound $\EE \left(  \sup_{\wm \in \Omega} \lv \nabla \Jhn(\wm,\xm) \rv_2\right)$. Then by Lemma~\ref{sig_loss}, we can bound it as follows:
\begin{equation*}
\begin{aligned}
\EE \left(\sup_{\wm \in \Omega} \lv \nabla \Jhn(\wm,\xm) \rv_2\right)\leq& \EE \left(\sup_{\wm \in \Omega} \lv\frac{1}{n}\sum_{i=1}^{n} \nabla f(\wm,\xmi{i})\rv_2 \right)\leq \alpha,
\end{aligned}
\end{equation*}
where $\alpha=\sqrt{\frac{1}{16}c_yc_d\left(1 + c_r (l-1)\right)}$ in which $c_y$, $c_d$ and $c_r$ are defined in Lemma~\ref{sig_loss}.

Therefore, we have
\begin{equation*}
\begin{aligned}
\Pro\left( \Em_1 \right) \leq \frac{3\alpha \epsilon}{t}.
\end{aligned}
\end{equation*}
 We further let
\begin{equation*}
t \geq \frac{6\alpha \epsilon}{\varepsilon}.
\end{equation*}
Then we can bound $\Pro(\Em_1)$:
\begin{equation*}
\Pro(\Em_1)\leq \frac{\varepsilon}{2}.
\end{equation*}

{\bf\noindent{Step 2. Bound $\Pro\left(\Em_2\right)$}}: Recall that we use $j_w$ to denote the index of $\wm^j_{k_j}$ in $\epsilon$-net $\wm_\epsilon^j$ and we have $j_w\in[\nee^j],\ (\nee^j\leq (3rl/\epsilon)^{\dm_j\dm_{j-1}})$. We can bound $\Pro\left(\Em_2\right)$ as follows:
\begin{equation*}
\begin{aligned}
\Pro\left( \Em_2 \right) = & \Pro \left( \sup_{j_w\in[\nee^j], j=[l]} \lvs\frac{1}{n}\sum_{i=1}^n f(\wm_{\kt},\xmi{i})-\EE(f(\wm_{\kt},\xm))\rvs \geq \frac{t}{3} \right) \\
\leq & \left(\frac{3lr}{\epsilon}\right)^{\sum_j\dm_j\dm_{j-1}} \sup_{j_w\in[\nee^j], j=[l]} \Pro \left(  \lvs\frac{1}{n}\sum_{i=1}^n f(\wm_{j},\xmi{i})-\EE(f(\wm_{j},\xm))\rvs\geq \frac{t}{3}\right).
\end{aligned}
\end{equation*}
Since when the activation functions are sigmoid functions, the loss $f(\wm,\xm)$ is $\alpha$-Lipschitz. Besides, we assume $\xm$ to be a vector of \textit{i.i.d.} Gaussian variables from $\mathcal{N}(0, \tau^2)$. Then by Lemma~\ref{Lip_gaus}, we know that the variable $f(\xm)-\EE f(\xm)$ is $8\alpha^2\tau^2$-sub-Gaussian. Thus, we have
\begin{equation*}
\Pro\left(\left|f(\xm)-\EE f(\xm)\right|>t\right)\leq 2 \exp\left(-\frac{t^2}{2\alpha^2\tau^2}\right),\quad (\forall t\geq 0),
\end{equation*}
where $\alpha=\sqrt{\frac{1}{16}c_yc_d\left(1 + c_r (l-1)\right)}$ in which $c_y$, $c_d$ and $c_r$ are defined in Lemma~\ref{sig_loss}. Therefore, we can obtain that $\frac{1}{n}\sum_{i=1}^n f(\wm_{j},\xmi{i})-\EE(f(\wm_{j},\xm))$ is $8\alpha^2\tau^2/n$-sub-Gaussian random variable. Thus, we can obtain
\begin{equation*}
\begin{aligned}
\Pro \left( \lvs \frac{1}{n}\sum_{i=1}^n f(\wm_{j},\xmi{i})-\EE(f(\wm_{j},\xm))\rvs\geq \frac{t}{3}\right)\leq 2\exp\left(-\frac{nt^2}{18\alpha^2\tau^2}\right).
\end{aligned}
\end{equation*}
Notice $\sum_j\dm_j\dm_{j-1}=d$. In this case, the probability of $\Em_2$ is upper bounded as
\begin{equation*}
\begin{aligned}
\Pro\left(\Em_2\right)\leq 2\exp\left(-\frac{nt^2}{18\alpha^2\tau^2}+ d\log\left(\frac{3lr}{\epsilon}\right)\right).
\end{aligned}
\end{equation*}
Thus, if we set
\begin{equation*}
\begin{aligned}
t\geq \alpha \tau\sqrt{\frac{18\left(d\log(3lr/\epsilon)+\log(4/\varepsilon)\right)}{n}},
\end{aligned}
\end{equation*}
then we have
\begin{equation*}
\Pro\left(\Em_2\right)\leq \frac{\varepsilon}{2}.
\end{equation*}

{\bf\noindent{Step 3. Bound $\Pro\left(\Em_3\right)$}}:
We first bound $\Pro\left(\Em_3\right)$ as follows:
\begin{equation*}
\begin{aligned}
\Pro\left( \Em_3 \right) = & \Pro \left( \sup_{\wm\in\Omega} \lvs\EE(f(\wm_{\kt},\xm))-\EE(f(\wm,\xm))\rvs \geq \frac{t}{3} \right) \\
=&  \Pro\left(\sup_{\wm\in\Omega} \frac{\lvs \EE\left( f(\wm_{\kt},\xm)-f(\wm,\xm)\right)\rvs}{\lv\wm-\wm_{\kt}\rv_2} \sup_{\wm\in\Omega} \lv\wm-\wm_{\kt}\rv_2 \geq \frac{t}{3}\right) \\
\leq&  \Pro\left(\epsilon\EE \sup_{\wm\in\Omega} \lv \nabla\Jm_{\wm}(\wm,\xm) \rv_2  \geq \frac{t}{3}\right) \\
\overset{\text{\ding{172}}}{\leq}&  \Pro\left(\alpha \epsilon  \geq \frac{t}{3}\right),
\end{aligned}
\end{equation*}
where \ding{172} holds since by Lemma~\ref{sig_loss}, for arbitrary $\xm$ and $\wm\in\Omega$, we have $\|\nabla_{\wm} f(\wm,\xm)\|_2\leq \alpha$.
We set $\epsilon$ enough small such that $\alpha \epsilon  < t/3$ always holds. Then it yields $\Pro\left( \Em_3 \right)=0$.

{\bf\noindent{Step 4. Final result}}: Notice, we have $\frac{6\alpha  \epsilon}{\varepsilon}\geq 3\alpha \epsilon$. To ensure $\Pro(\Em_0)\leq \varepsilon$, we just set $\epsilon=3 r/n$ and
\begin{equation*}
\begin{split}
t\!\geq\! \max\!\left(\!\frac{6\alpha  \epsilon}{\varepsilon}, \alpha \tau\sqrt{\frac{18\left(d\log(3lr/\epsilon)+\log(4/\varepsilon)\right)}{n}} \right)\!\!=\!\max\!\left(\!\frac{18\alpha r}{n\varepsilon},  \alpha \tau\sqrt{\frac{18\left(d\log(nl)\!+\!\log(4/\varepsilon)\right)}{n}}\right).
\end{split}
\end{equation*}


Therefore, if $n\geq 18r^2/(d\tau^2\varepsilon^2\log(l))$, then
\begin{equation*}
\sup_{\wm\in\Omega} \lvs  \Jhn(\wm)- \Jm(\wm) \rvs \leq  \tau\sqrt{\frac{9}{8}c_yc_d\left(1 + c_r (l-1)\right)} \sqrt{\frac{d\log(nl)+\log(4/\varepsilon)}{n}}
\end{equation*}
holds with probability at least $1-\varepsilon$, where $c_y$, $c_d$, and $c_r$ are defined as
\begin{equation*}
\begin{split}
\|\vmi{l}-\ym\|_2^2\leq c_y<+\infty,  \quad c_{d}=\max(\dm_0,\dm_1,\cdots,\dm_l)  \quad \text{and}\quad c_r=\max\left(\frac{r^2}{16},\left(\frac{r^2}{16}\right)^{l-1}\right).
\end{split}
\end{equation*}
The proof is completed.
\end{proof}

\subsubsection{Proof of Corollary~\ref{stabiliaftysig}}
\begin{proof}[\hypertarget{cor2}{Proof}]
By Lemma~\ref{stab}, we know $\epsilon_s=\epsilon_g$. Thus, the remaining work is to bound $\epsilon_s$. Actually, we can have
\begin{equation*}
\begin{split}
\left|\EE_{\Ss\sim\D,\Am,(\xmi{1}',\cdots,\xmi{n}'\!)\sim\D}\frac{1}{n}\! \sum_{j=1}^{n}\!\!\left(\!f_j(\wms{j},\!\xmi{j}'\!)\!-\!\!f_j(\wmn,\xmi{j}'\!)\! \right)\!\right| \leq &\EE_{\Ss\sim\D}\left( \sup_{\wm\in\Omega} \left|  \Jhn(\wm)- \Jm(\wm)\right|\right)\\
\leq &\sup_{\wm\in\Omega} \left|  \Jhn(\wm)- \Jm(\wm)\right|\\
\leq& \epsilon_n.
\end{split}
\end{equation*}
Thus, we have $\epsilon_g=\epsilon_s\leq \epsilon_n$. The proof is completed.
\end{proof}

\subsubsection{ Proof of Theorem~\ref{thm:uniformconvergence1_sig}}

\begin{proof}[\hypertarget{uniformconvergencsig}{Proof}]
Recall that  the weight of each layer has magnitude bound separately, \textit{i.e.} $\|\wmi{j}\|_2\leq r$. So here we separately assume $\wm_{\epsilon}^j=\{\wm_1^j,\cdots,\wm_{\nee^j}^j\}$ is the $\epsilon/l$-covering net of the ball $\Bi{\dm_j\dm_{j-1}}{r}$ which corresponds to the weight $\wmi{j}$ of the $j$-th layer. Let $\nee^j$ be the $\epsilon/l$-covering number. By $\epsilon$-covering theory in \cite{VRMT}, we can have $\nee^j\leq (3rl/\epsilon)^{\dm_j\dm_{j-1}}$. Let $\wm\in\Omega$ be an arbitrary vector. Since $\wm=[\wmi{1},\cdots,\wmi{l}]$ where $\wmi{j}$ is the weight of the $j$-th layer, we can always find a vector $\wm^j_{k_j}$ in $\wm_\epsilon^j$ such that $\|\wmi{j}-\wm^j_{k_j}\|_2\leq \epsilon/l$. For brevity, let $j_w\in[\nee^j]$  denote the index of $\wm^j_{k_j}$ in $\epsilon$-net $\wm_\epsilon^j$. Then let $\wm_{\kt}=[\wm^j_{k_1};\cdots;\wm^j_{k_j};\cdots;\wm^j_{k_l}]$. This means that we can always find a vector $\wm_{\kt}$ such that $\|\wm-\wm_{\kt}\|_2\leq \epsilon$.
Accordingly, we can decompose $\lv\nabla \Jhn(\wm)-\nabla\Jm(\wm)\rv_2$ as follows:
\begin{equation*}
\begin{aligned}
&\lv\nabla \Jhn(\wm)-\nabla\Jm(\wm)\rv_2\\
=& \lv \frac{1}{n}\sum_{i=1}^n \nabla f(\wm,\xmi{i})-\EE(\nabla f(\wm,\xm))\rv_2\\
=&\Bigg\| \frac{1}{n}\sum_{i=1}^n \left(\nabla f(\wm,\xmi{i})-\nabla f(\wm_{\kt},\xmi{i})\right)+\frac{1}{n}\sum_{i=1}^n \nabla f(\wm_{\kt},\xmi{i})-\EE(\nabla f(\wm_{\kt},\xm)) \\
& \ \ +\EE(\nabla f(\wm_{\kt},\xm))-\EE(\nabla f(\wm,\xm))\Bigg\|_2\\
\leq &\lv \frac{1}{n}\sum_{i=1}^n \left(\nabla f(\wm,\xmi{i})-\nabla f(\wm_{\kt},\xmi{i})\right)\rv_2+\lv\frac{1}{n}\sum_{i=1}^n \nabla f(\wm_{\kt},\xmi{i})-\EE(\nabla f(\wm_{\kt},\xm))\rv_2\\
&\ \ +\Bigg\|\EE(\nabla f(\wm_{\kt},\xm))-\EE(\nabla f(\wm,\xm))\Bigg\|_2.
\end{aligned}
\end{equation*}
Here we also define four events $\Em_0$, $\Em_1$, $\Em_2$ and $\Em_3$ as
\begin{equation*}
\begin{aligned}
&\Em_0=\left\{\sup_{\wm\in\Omega}\lv\nabla\Jhn(\wm)-\nabla\Jm(\wm)\rv_2\geq t\right\},\\
&\Em_1=\left\{\sup_{\wm\in\Omega}\lv \frac{1}{n}\sum_{i=1}^n \left(\nabla f(\wm,\xmi{i})-\nabla f(\wm_{\kt},\xmi{i})\right)\rv_2\geq \frac{t}{3}\right\},\\
&\Em_2=\left\{\sup_{j_w\in[\nee^j], j=[l]} \lv\frac{1}{n}\sum_{i=1}^n \nabla f(\wm_{\kt},\xmi{i})-\EE(\nabla f(\wm_{\kt},\xm))\rv_2\geq \frac{t}{3}\right\},\\
&\Em_3=\left\{\sup_{\wm\in\Omega}\Bigg\|\EE(\nabla f(\wm_{\kt},\xm))-\EE(\nabla f(\wm,\xm))\Bigg\|_2\geq \frac{t}{3}\right\}.
\end{aligned}
\end{equation*}
Accordingly, we have
\begin{equation*}
\begin{aligned}
\Pro\left(\Em_0\right)\leq \Pro\left(\Em_1\right)+\Pro\left(\Em_2\right)+\Pro\left(\Em_3\right).
\end{aligned}
\end{equation*}
So we can respectively bound $\Pro\left(\Em_1\right)$, $\Pro\left(\Em_2\right)$ and $\Pro\left(\Em_3\right)$ to bound $\Pro\left(\Em_0\right)$.

{\bf\noindent{Step 1. Bound $\Pro\left(\Em_1\right)$}}: We first bound $\Pro\left(\Em_1\right)$ as follows:
\begin{equation*}
\begin{aligned}
\Pro\left( \Em_1 \right) = & \Pro \left( \sup_{\wm \in \Omega} \lv \frac{1}{n}\sum_{i=1}^n \left(\nabla f(\wm,\xmi{i})-\nabla f(\wm_{\kt},\xmi{i})\right)\rv_2\geq \frac{t}{3} \right) \\
\overset{\text{\ding{172}}}{\leq} &\frac{3}{t}\EE \left(  \sup_{\wm \in \Omega}\lv \frac{1}{n}\sum_{i=1}^n\left(\nabla f(\wm,\xmi{i})-\nabla f(\wm_{\kt},\xmi{i})\right)\rv_2\right) \\
\leq &\frac{3}{t}\EE \left(  \sup_{\wm \in \Omega}\frac{\lv \frac{1}{n}\sum_{i=1}^n\left(\nabla f(\wm,\xmi{i})-\nabla f(\wm_{\kt},\xmi{i})\right)\rv_2}{\lv\wm-\wm_{\kt}\rv_2} \sup_{\wm\in\Omega} \lv\wm-\wm_{\kt}\rv_2\right) \\
\leq &\frac{3\epsilon}{t}\EE \left(  \sup_{\wm \in \Omega} \lv \nabla^2 \Jhn(\wm,\xm) \rv_2\right),
\end{aligned}
\end{equation*}
where \ding{172} holds because of Markov inequality.
Then, we bound $\EE \left(  \sup_{\wm \in \Omega} \lv \nabla^2 \Jhn(\wm,\xm) \rv_2\right)$ as follows:
\begin{equation*}
\begin{aligned}
\EE\!  \left(\!\sup_{\wm \in \Omega}\! \lv \nabla^2 \! \Jhn(\wm,\xm) \rv_2\!\right)\!\!\leq \!\EE \!\left(\!\sup_{\wm \in \Omega}\! \lv \frac{1}{n}\! \sum_{i=1}^n  \nabla^2\! f(\wm,\!\xm) \rv_2 \right)\!\!=\! \EE \left(\!\sup_{\wm \in \Omega} \!\lv\! \nabla^2\! f(\wm,\!\xm)\rv_2 \!\right)
\!\overset{\text{\ding{172}}}{\leq} \! \varsigma,
\end{aligned}
\end{equation*}
where \ding{172} holds since by Lemma~\ref{sig_gradient}, we have
\begin{equation*}\label{gradient1}
\begin{split}
\left\|\nabla^2_{\wm}f(\wm,\xm)\right\|_{\op}\leq \left\|\nabla^2_{\wm}f(\wm,\xm)\right\|_{F}
\leq& \varsigma,
\end{split}
\end{equation*}
where  $\varsigma=\sqrt{c_{s_1}c_rc_d^2 l^2\left(c_{s_2}c_d^2+l^2c_r\right)}$ in which $c_{d}=\max_i \dm_i$ and $c_r=\max\left(\frac{r^2}{16},\left(\frac{r^2}{16}\right)^{l-1}\right)$. Therefore, we have
\begin{equation*}
\begin{aligned}
\Pro\left( \Em_1 \right) \leq \frac{3\varsigma\epsilon}{t}.
\end{aligned}
\end{equation*}
 We further let
\begin{equation*}
t \geq \frac{6\varsigma\epsilon}{\varepsilon}.
\end{equation*}
Then we can bound $\Pro(\Em_1)$:
\begin{equation*}
\Pro(\Em_1)\leq \frac{\varepsilon}{2}.
\end{equation*}


{\bf\noindent{Step 2. Bound $\Pro\left(\Em_2\right)$}}: By Lemma~\ref{lem:2norm}, we know that for any vector $\xm\in\Rs{d}$, its $\ell_2$-norm can be computed as
\begin{equation*}
\|\xm\|_2 \leq \frac{1}{1-\epsilon} \sup_{\lam \in \lam_\epsilon} \lr \lam,\xm\rl.
\end{equation*}
where $\lam_\epsilon=\{\lam_1,\dots,\lam_{\kt}\}$ be an $\epsilon$-covering net of $\Bi{d}{1}$.

 Let $\lam_{1/2}$ be the $\frac{1}{2}$-covering net of $\Bi{d}{1}$. Recall that we use $j_w$ to denote the index of $\wm^j_{k_j}$ in $\epsilon$-net $\wm_\epsilon^j$ and we have $j_w\in[\nee^j],\ (\nee^j\leq (3rl/\epsilon)^{\dm_j\dm_{j-1}})$. Then we can bound $\Pro\left(\Em_2\right)$ as follows:
\begin{equation*}
\begin{aligned}
\Pro\left( \Em_2 \right) = & \Pro \left( \sup_{j_w\in[\nee^j], j=[l]} \lv\frac{1}{n}\sum_{i=1}^n \nabla f(\wm_{\kt},\xmi{i})-\EE(\nabla f(\wm_{\kt},\xm))\rv_2\geq \frac{t}{3} \right) \\
= & \Pro \left( \sup_{j_w\in[\nee^j], j=[l], \lam\in\lam_{1/2}} 2\lr\lam, \frac{1}{n}\sum_{i=1}^n \nabla f(\wm_{\kt},\xmi{i})-\EE\left(\nabla f(\wm_{\kt},\xm)\right)\rl\geq \frac{t}{3} \right) \\
\leq & 6^d\left(\frac{3lr}{\epsilon}\right)^{\sum_j\dm_j\dm_{j-1}}\!\!\!\!\!\!\!\!\!\! \sup_{j_w\in[\nee^j], j=[l], \lam\in\lam_{1/2}} \Pro \left(\frac{1}{n} \sum_{i=1}^n\lr\lam, \nabla f(\wm_{\kt},\xmi{i})-\EE\left(\nabla f(\wm_{\kt},\xm)\right)\rl\geq \frac{t}{6} \right).
\end{aligned}
\end{equation*}

Since by Lemma~\ref{thm:gradsgfdsgdient12}, $\lr \ym,\nabla f(\wm,\xm)\rl$ is $8\beta^2\tau^2$-sub-Gaussian, \textit{i.e.}
\begin{equation*}
\EE\left(\lr\lam ,\nabla_{\wm}f(\wm,\xm)-\EE \nabla_{\wm}f(\wm,\xm)\rl \right)\leq \exp\left(\frac{8\beta^2\tau^2\|\lam \|_2^2}{2}\right),
\end{equation*}
where $\beta=\sqrt{\frac{2^6}{3^8}c_yc_r(l+2)\left(dc_r+(l-1)lc_dc_r+lc_d\right)}$ in which $c_y$, $c_d$ and $c_r$ are defined in Lemma~\ref{thm:gradsgfdsgdient12}. Thus, $\frac{1}{n}\sum_{i=1}^n \lr \ym,\nabla f(\wm,\xm)\rl$ is $8\beta^2\tau^2/n$-sub-Gaussian random variable. Thus, we can obtain
\begin{equation*}
\begin{aligned}
\Pro \left(\frac{1}{n} \sum_{i=1}^n\lr\ym, \nabla f(\wm_{\kt},\xmi{i})-\EE\left(\nabla f(\wm_{\kt},\xm)\right)\rl\geq \frac{t}{6} \right)\leq \exp\left(-\frac{n t^2}{72\beta^2\tau^2}\right).
\end{aligned}
\end{equation*}
Notice, $\sum_{j}\dm_j\dm_{j-1}=d$. In this case, the probability of $\Em_2$ is upper bounded as
\begin{equation*}
\begin{aligned}
\Pro\left(\Em_2\right)\leq \exp\left(-\frac{n t^2}{72\beta^2\tau^2}+ d\log\left(\frac{18r}{\epsilon}\right)\right).
\end{aligned}
\end{equation*}
Thus, if we set
\begin{equation*}
\begin{aligned}
t\geq \beta \tau\sqrt{\frac{72\left(d\log(18lr/\epsilon)+\log(4/\varepsilon)\right)}{n}},
\end{aligned}
\end{equation*}
then we have
\begin{equation*}
\Pro\left(\Em_2\right)\leq \frac{\varepsilon}{2}.
\end{equation*}

{\bf\noindent{Step 3. Bound $\Pro\left(\Em_3\right)$}}:
We first bound $\Pro\left(\Em_3\right)$ as follows:
\begin{equation*}
\begin{aligned}
\Pro\left( \Em_3 \right) = & \Pro \left( \sup_{\wm\in\Omega} \lv\EE(\nabla f(\wm_{\kt},\xm))-\EE(\nabla f(\wm,\xm))\rv_2 \geq \frac{t}{3} \right) \\
=&  \Pro\left(\sup_{\wm\in\Omega} \frac{\lv \EE\left( \nabla f(\wm_{\kt},\xm)-\nabla f(\wm,\xm) \rv_2\right)}{\lv\wm-\wm_{\kt}\rv_2} \sup_{\wm\in\Omega} \lv\wm-\wm_{\kt}\rv_2 \geq \frac{t}{3}\right) \\
\leq&  \Pro\left(\epsilon\EE \sup_{\wm\in\Omega} \lv \nabla^2 \Jhn(\wm,\xm) \rv_2  \geq \frac{t}{3}\right) \\
\overset{\text{\ding{172}}}{\leq}&  \Pro\left(\varsigma \epsilon  \geq \frac{t}{3}\right).
\end{aligned}
\end{equation*}
where \ding{172} holds since by Lemma~\ref{sig_gradient}. We set $\epsilon$ enough small such that $\varsigma \epsilon  < t/3$ always holds. Then it yields $\Pro\left( \Em_3 \right)=0$.

{\bf\noindent{Step 4. Final result}}: To ensure $\Pro(\Em_0)\leq \varepsilon$, we just set $\epsilon=18r/n$ and
\begin{equation*}
\begin{split}
t\geq& \max\left(\frac{6\varsigma\epsilon}{\varepsilon},\ \beta \tau\sqrt{\frac{72\left(d\log(18lr/\epsilon)+\log(4/\varepsilon)\right)}{n}}\right)\\
=&\max\left(\frac{108\varsigma r}{n\varepsilon},\ \beta \tau \sqrt{\frac{72\left(d\log(nl)+\log(4/\varepsilon)\right)}{n}} \right).
\end{split}
\end{equation*}
%

Note that $\varsigma=\mathcal{O}(\sqrt{lc_d}\beta)$. Therefore, there exists a universal constant $c_{y'}$ such that if $n\geq c_{y'}c_d lr^2/(d\tau^2\varepsilon^2\log(l)$, then
\begin{equation*}
\sup_{\wm\in\Omega}\left\| \nabla\Jhn(\wm)\!-\!\nabla\Jm(\wm)\right\|_2\!\!\leq\! \tau\!\sqrt{\!\frac{512}{729}c_yc_r(l\!+\!2)\left(dc_r\!+\!(l\!-\!1) lc_dc_r\!+\!lc_d\right)} \sqrt{\!\frac{d\log(nl)\!+\!\log(4/\varepsilon)}{n}}
\end{equation*}
holds with probability at least $1-\varepsilon$.
\end{proof}

\subsubsection{ Proof of Theorem~\ref{thm:localminimal_sig}}

\begin{proof}[\hypertarget{uniformdfglocasig}{Proof}]
Suppose that $\{\wmii{1},\wmii{2},\cdots,\wmii{m}\}$ are the non-degenerate critical points of $\Jm(\wm)$. So for any $\wmii{k}$, it obeys
\begin{align*}
\inf_i \left|\lambda_i^k\left(\nabla^2 \Jm(\wmii{k})\right)\right| \geq \zeta,
\end{align*}
where $\lambda_i^k\left(\nabla^2 \Jm(\wmii{k})\right)$ denotes the $i$-th eigenvalue of the Hessian $\nabla^2 \Jm(\wmii{k})$ and $\zeta$ is a constant. We further define a set $D=\{\wm\in\Rs{d}\,|\,\|\nabla\Jm(\wm)\|_{2}\leq \epsilon\ \text{and}\ \inf_i |\lambda_i\left(\nabla^2 \Jm(\wmii{k})\right)| \geq \zeta \}$. According to Lemma~\ref{lemma:Decomposition}, $D = \cup_{k=1}^{\infty} D_k$ where each $D_k$ is a disjoint component with $\wmii{k}\in D_k$ for $k\leq m$  and $D_k$ does not contain any critical point of $\Jm(\wm)$ for $k\geq m+1$. On the other hand, by the continuity of $\nabla\Jm(\wm)$, it yields $\|\nabla\Jm(\wm)\|_2=\epsilon$ for $\wm\in\partial D_k$. Notice, we set the value of $\epsilon$ blow which is actually a function related $n$.

Then by utilizing Theorem~\ref{thm:uniformconvergence1_sig}, we let sample number $n$ sufficient large such that
\begin{equation*}
\sup_{\wm\in\Omega} \left\| \nabla\Jhn(\wm)-\nabla\Jm(\wm)\right\|_2\leq \beta\tau\sqrt{\frac{ d\log(nl)+\log(4/\varepsilon) }{n}}\triangleq \frac{\epsilon}{2}
\end{equation*}
where $\beta=\sqrt{\!\frac{512}{729}c_yc_r(l\!+\!2)\left(dc_r\!+\!(l\!-\!1) lc_dc_r\!+\!lc_d\right)} $,
holds with probability at least $1-\varepsilon$.  This further gives that for arbitrary $\wm\in D_k$, we have
\begin{align}
\inf_{\wm \in D_k}\left\| t \nabla\Jhn(\wm)+(1-t)\nabla\Jm(\wm)\right\|_2=&\inf_{\wm \in D_k}\left\| t \left(\nabla\Jhn(\wm)-\nabla\Jm(\wm)\right)+\nabla\Jm(\wm)\right\|_2 \notag\\
\geq &\inf_{\wm \in D_k} \left\|\nabla\Jm(\wm)\right\|_2- \sup_{\wm \in D_k}t\left\|\nabla\Jhn(\wm)-\nabla\Jm(\wm)\right\|_2\notag\\
\geq &\frac{\epsilon}{2}.\label{grasa23disdenthhhadfaf}
\end{align}
Similarly, by utilizing Lemma~\ref{thm:uniformconvergence1fds_sig}, let $n$ be sufficient large such that
\begin{equation*}
\sup_{\wm\in\Omega} \left\| \nabla^2\Jhn(\wm)-\nabla^2\Jm(\wm)\right\|_{\op}\leq c_m\gamma\tau\sqrt{\frac{d\log(nl)+\log(4/\varepsilon)}{n}}\leq \frac{\zeta}{2}
\end{equation*}
holds with probability at least $1-\varepsilon$.  Assume that $\bmm\in\Rs{d}$ is a vector and satisfies $\bmm^T\bmm=1$. In this case, we can bound $\lambda_i^k\left(\nabla^2 \Jhn(\wm)\right)$ for arbitrary $\wm\in D_k$ as follows:
\begin{equation*}
\begin{split}
\inf_{\wm \in D_k}\left|\lambda_i^k\left(\nabla^2 \Jhn(\wm)\right)\right|=&\inf_{\wm \in D_k}\min_{\bmm^T\bmm=1} \left|\bmm^T \nabla^2 \Jhn(\wm)\bmm\right|\\
=&\inf_{\wm \in D_k}\min_{\bmm^T\bmm=1} \left|\bmm^T \left(\nabla^2 \Jhn(\wm)-\nabla^2 \Jm(\wm)\right)\bmm+\bmm^T \nabla^2 \Jm(\wm)\bmm\right|\\
\geq&\inf_{\wm \in D_k}\min_{\bmm^T\bmm=1}\left|\bmm^T \nabla^2 \Jm(\wm)\bmm\right|-\min_{\bmm^T\bmm=1} \left|\bmm^T \left(\nabla^2 \Jhn(\wm)-\nabla^2 \Jm(\wm)\right)\bmm\right|\\
\geq&\inf_{\wm \in D_k}\min_{\bmm^T\bmm=1}\left|\bmm^T \nabla^2 \Jm(\wm)\bmm\right|-\max_{\bmm^T\bmm=1} \left|\bmm^T \left(\nabla^2 \Jhn(\wm)-\nabla^2 \Jm(\wm)\right)\bmm\right|\\
=& \inf_{\wm \in D_k} \inf_i |\lambda_i^k\left(\nabla^2 f(\wm_{(k)},\xm)\right)-\left\|\nabla^2 \Jhn(\wm)-\nabla^2 \Jm(\wm)\right\|_{\op}\\
\geq &\frac{\zeta}{2}.
\end{split}
\end{equation*}
This means that in each set $D_k$, $\nabla^2\Jhn(\wm)$ has no zero eigenvalues. Then, combining this and Eqn.~\eqref{grasa23disdenthhhadfaf}, by Lemma~\ref{lemma:Stability2} we know that if the population risk $\Jm(\wm)$ has no critical point in $D_k$, then the empirical risk $\Jhn(\wm)$ has also no critical point in $D_k$; otherwise it also holds. By Lemma~\ref{lemma:Stability2}, we can also obtain that in $D_k$, if $\Jm(\wm)$ has a unique critical point $\wmii{k}$ with non-degenerate index $s_{k}$, then $\Jhn(\wm)$ also has a unique critical point $\wmin{k}$ in $D_k$ with the same non-degenerate index $s_k$. The first conclusion is proved.

Now we bound the distance between the corresponding critical points of $\Jm(\wm)$ and $\Jhn(\wm)$. Assume that in $D_k$, $\Jm(\wm)$ has a unique critical point $\wmii{k}$ and $\Jhn(\wm)$ also has a unique critical point $\wmin{k}$.
Then, there exists $t\in[0,1]$ such that for any $\zm\in\partial \Bi{d}{1}$, we have
\begin{align*}
\epsilon \geq& \|\nabla \Jm(\wmin{k})\|_2\\
=&\max_{\zm^T\zm=1}
\langle \nabla \Jm(\wmin{k}), \zm \rangle\\
 = &\max_{\zm^T\zm=1} \langle \nabla \Jm(\wmii{k}), \zm \rangle +\langle \nabla^2 \Jm(\wmii{k}+t(\wmin{k}-\wmii{k}))(\wmin{k}-\wmii{k}), \zm \rangle\\
\overset{\text{\ding{172}}}{\geq} & \lr \left(\nabla^2 \Jm(\wmii{k})\right)^2(\wmin{k}-\wmii{k}), (\wmin{k}-\wmii{k})\rl^{1/2}\\
\overset{\text{\ding{173}}}{\geq} & \zeta \|\wmin{k}-\wmii{k}\|_2,
\end{align*}
where \ding{172} holds since $\nabla \Jm(\wmii{k})=\bm{0}$ and \ding{173} holds since $\wmii{k}+t(\wmin{k}-\wmii{k})$ is in $D_k$ and for any $\wm \in D_k$ we have $\inf_i |\lambda_i\left(\nabla^2 \Jm(\wm)\right)| \geq \zeta$. Consider the conditions in Lemma~\ref{thm:uniformconvergence1fds_sig} and Theorem~\ref{thm:uniformconvergence1_sig}, we can obtain that
if $n\geq c_s\max(c_d lr^2/(d\tau^2\varepsilon^2\log(l)), d\log(l)/\zeta^2)$ where $c_s$ is a constant, then
\begin{align*}
\|\wmin{k}-\wmii{k}\|_2\leq \frac{2\tau}{\zeta} \sqrt{\!\frac{512}{729}c_yc_r(l\!+\!2)\left(dc_r\!+\!(l\!-\!1) lc_dc_r\!+\!lc_d\right)} \sqrt{\!\frac{d\log(nl)\!+\!\log(2/\varepsilon)}{n}}
\end{align*}
holds with probability at least $1-\varepsilon$. The proof is completed.
\end{proof}

\subsection{Proof of Other Lemmas} \label{definitiondsfadf}
\subsubsection{Proof of Lemma~\ref{asfashdfhfdet}}
\begin{proof}[\hypertarget{lemma21}{Proof}]
Since $\mcode{G}(\umi{i})$ is a diagonal matrix and its diagonal values are upper bounded by $\sigma(\umi{i}_h)(1-\sigma(\umi{i}_h))\leq 1/4$ where $\umi{i}_h$ denotes the $h$-th entry of $\umi{i}$, we can conclude
\begin{equation*}\label{gradient1}
\begin{split}
\|\mcode{G}(\umi{i})\Mm\|_F^2\leq \frac{1}{16}\|\Mm\|_F^2\quad \text{and} \quad \|\Nm\mcode{G}(\umi{i})\|_F^2\leq \frac{1}{16}\|\Nm\|_F^2.
\end{split}
\end{equation*}
Note that $\Pm_k$ is a matrix of size $\dm_k^2\times \dm_k$ whose $((s-1)\dm_k+s,s)$ $(s=1,\cdots,\dm_k)$ entry equal to $\sigma(\umi{k}_{s})(1-\sigma(\umi{k}_{s}))(1-2\sigma(\umi{k}_{s}))$ and rest entries are all $0$. This gives
\begin{equation*}\label{gsfasgsdgsadasgaradient1}
\begin{split}
\sigma(\umi{k}_{s})(1-\sigma(\umi{k}_{s}))(1-2\sigma(\umi{k}_{s}))=&\frac{1}{3} (3\sigma(\umi{k}_{s}))(1-\sigma(\umi{k}_{s}))(1-2\sigma(\umi{k}_{s}))\\
\leq& \frac{1}{3} \left(\frac{3\sigma(\umi{k}_{s})+1-\sigma(\umi{k}_{s})+1-2\sigma(\umi{k}_{s})}{3}\right)^3\\
\leq& \frac{2^3}{3^4}.
\end{split}
\end{equation*}
This means the maximal value in $\Pm_k$ is at most  $\frac{2^3}{3^4}$. Consider the structure in $\Pm_k$, we can obtain
\begin{equation*}\label{gradient1}
\begin{split}
\|\Pm_k\Mm\|_F^2\leq \frac{2^6}{3^8}\|\Mm\|_F^2\quad \text{and}\quad \|\Nm\Pm_k\|_F^2\leq \frac{2^6}{3^8}\|\Nm\|_F^2.
\end{split}
\end{equation*}

As for $\Bm_{s:t}$, we have
\begin{equation*}\label{gradient1}
\begin{split}
\left\|\Bm_{s:t}\right\|_F^2\leq& \left\|\Am_s\right\|_F^2 \left\|\Am_{s+1}\right\|_F^2 \cdots \left\|\Am_{t}\right\|_F^2\\
=&\left\|(\Wm^s)^T\mcode{G}(\umi{s})\right\|_F^2\left\|(\Wmi{s+1})^T\mcode{G}(\umi{s+1})\right\|_F^2\cdots\left\|(\Wmi{t})^T\mcode{G}(\umi{t})\right\|_F^2\\
\leq&\frac{1}{16^{t-s+1}}\left\|\Wmi{s}\right\|_F^2\left\|\Wmi{s+1}\right\|_F^2\cdots\left\|\Wmi{t}\right\|_F^2\\
=&\frac{1}{16^{t-s+1}}\Dm_{s:t}.
\end{split}
\end{equation*}
Since the $\ell_2$-norm of each $\wmi{j}$ is bounded, \emph{i.e.} $\|\wmi{j}\|_2\leq r$, we can obtain
\begin{equation*}\label{gradient1}
\begin{split}
\frac{1}{16^{t-s+1}}\Dm_{s:t} \!\leq\! \frac{1}{16^{t-s+1}} r^{2(t-s+1)}=\left(\frac{r}{4}\right)^{2(t-s+1)}\triangleq \!c_{st}.
\end{split}
\end{equation*}
The proof is completed.
\end{proof}

\subsubsection{Proof of Lemma~\ref{sig_gdgsaghhradsdfsagfasgient}}
\begin{proof}[\hypertarget{lemma22}{Proof}]
By utilizing the chain rule in Eqn.~\eqref{chain_rule} in Sec.~\ref{afsafsaggjx}, we can easily compute $\frac{\partial  f(\wm,\xm)}{\partial \umi{i}}$ and $\frac{\partial  f(\wm,\xm)}{\partial \vmi{i}}$ as follows:
\begin{equation*}\label{gsfasgdasgaradient1}
\begin{split}
\frac{\partial  f(\wm,\xm)}{\partial \umi{i}}=\mcode{G}(\umi{i})\Am_{i+1}\cdots\Am_{l}(\vmi{l}-\ym)=\mcode{G}(\umi{i})\Bm_{i+1:l} (\vmi{l}-\ym)
\end{split}
\end{equation*}
and
\begin{equation*}\label{gsfasgdasgaradient1}
\begin{split}
\frac{\partial  f(\wm,\xm)}{\partial \vmi{i}}=\Am_{i+1}\cdots\Am_{l}(\vmi{l}-\ym)=\Bm_{i+1:l}(\vmi{l}-\ym).
\end{split}
\end{equation*}
Therefore, we can further obtain
\begin{equation*}\label{gsfasgsdgsadasgaradient1}
\begin{split}
&\frac{\partial  f(\wm,\xm)}{\partial \wmi{j}}\\
=&\vect{\left(\mcode{G}(\umi{j}) \Am_{j+1}\Am_{j+2}\cdots \Am_{l}(\vmi{l}-\ym)\right)(\vmi{j-1})^T}\\
=&\vect{\left(\mcode{G}(\umi{j}) \Am_{j+1}\Am_{j+2}\cdots\Am_{i-1}(\Wmi{i})^T\right)\! \!\left(\mcode{G}(\umi{i})\Am^{i+1} \cdots \Am_{l}(\vmi{l}-\ym)\right) \!(\vmi{j-1})^T}\\
=&\left(\vmi{j-1}\! \otimes \!\left(\mcode{G}(\umi{j}) \Am_{j+1}\Am_{j+2}\cdots\Am_{i-1}(\Wmi{i})^T\right)\! \right)\!\vect{\mcode{G}(\umi{i})\Am_{i+1}\cdots \Am_{l}(\vmi{l}-\ym)}\\
=&\left(\vmi{j-1} \otimes \left(\mcode{G}(\umi{j}) \Am_{j+1}\Am_{j+2}\cdots\Am_{i-1}(\Wmi{i})^T\right) \right)\left(\frac{\partial  f(\wm,\xm)}{\partial \umi{i}}\right).
\end{split}
\end{equation*}
Note that we have $\frac{\partial  f(\wm,\xm)}{\partial \wmi{j}}=\frac{\partial \umi{i}}{\partial \wmi{j}}\left(\frac{\partial  f(\wm,\xm)}{\partial \umi{i}}\right)$. This gives \begin{equation*}\label{gsfasgsdgsadasgaradient1}
\begin{split}
\frac{\partial \umi{i}}{\partial \wmi{j}}
=(\vmi{j-1})^T\otimes\left(\mcode{G}(\umi{j}) \Bm_{j+1:i-1}(\Wmi{i})^T\right)^T\in\Rss{\dm_i}{\dm_j\dm_{j-1}} \ (i>j).
\end{split}
\end{equation*}
When $i=j$, we have
\begin{equation*}\label{gsfasgsdgsadasgaradient1}
\begin{split}
\frac{\partial \umi{i}}{\partial \wmi{i}}
= (\vmi{i-1})^T\otimes \Im_{\dm_i}\in\Rss{\dm_i}{\dm_i\dm_{i-1}}.
\end{split}
\end{equation*}
Similarly, we can obtain
\begin{equation*}\label{gsfasgsdgsadasgaradient1}
\begin{split}
\frac{\partial \vmi{i}}{\partial \wmi{j}}
\!=\!(\vmi{j-1})^T \!\!\otimes\! \left(\mcode{G}(\umi{j}) \Am_{j+1}\Am_{j+2}\cdots\Am_{i}\right)^T\!\!=\!\!(\vmi{j-1})^T \!\!\otimes\! \left(\mcode{G}(\umi{j}) \Bm_{j+1:i}\right)^T\!\!\!\in\!\!\Rss{\dm_i}{\dm_j\dm_{j-1}} \ (i\!\geq\! j).
\end{split}
\end{equation*}
The proof is completed.
\end{proof}

\subsubsection{Proof of Lemma~\ref{sig_gdgsasdfkjl45asgient}}
\begin{proof}
By Lemma~\ref{sig_gdgsaghhradsdfsagfasgient}, we have
\begin{equation*}\label{gsfasgdasgaradient1}
\begin{split}
\frac{\partial  f(\wm,\xm)}{\partial \umi{i}}=\mcode{G}(\umi{i})\Bm_{i+1:l} (\vmi{l}-\ym)\quad \text{and}\quad \frac{\partial  f(\wm,\xm)}{\partial \vmi{i}}=\Bm_{i+1:l}(\vmi{l}-\ym).
\end{split}
\end{equation*}

Therefore, we can further obtain
\begin{equation*}\label{gsfasgsdgsadasgaradient1}
\begin{split}
\frac{\partial  f(\wm,\xm)}{\partial \umi{1}}=&\mcode{G}(\umi{1})\Am_{2}\cdots\Am_{l}(\vmi{l}-\ym)\\
=&\mcode{G}(\umi{1})\Am_{2}\cdots\Am_{j-1}(\Wm^j)^T\mcode{G}(\umi{j})\Am_{j+1}\cdots\Am_{l}(\vmi{l}-\ym)\\
=& \left(\mcode{G}(\umi{1})\Am_{2}\cdots\Am_{j-1}(\Wm^j)^T\right)
\left(\frac{\partial  f(\wm,\xm)}{\partial \umi{j}}\right).
\end{split}
\end{equation*}
Note that we have $\frac{\partial  f(\wm,\xm)}{\partial \umi{1}}=\left(\frac{\partial \umi{j}}{\partial \umi{1}}\right)^T\left(\frac{\partial  f(\wm,\xm)}{\partial \umi{j}}\right)$. This gives \begin{equation*}\label{gsfasgsdgsadasgaradient1}
\begin{split}
\frac{\partial \umi{j}}{\partial \umi{1}}
=\left(\mcode{G}(\umi{1})\Am_{2}\cdots\Am_{j-1}(\Wm^j)^T \right)^T=\left(\mcode{G}(\umi{1})\Bm_{2:j-1}(\Wm^j)^T \right)^T \in\Rss{\dm_j}{\dm_1} \ (j>1).
\end{split}
\end{equation*}

Similarly, we can obtain
\begin{equation*}\label{gsfasgsdgsadasgaradient1}
\begin{split}
\frac{\partial \vmi{j}}{\partial \umi{1}}
= \left(\mcode{G}(\umi{1})\Am_{2}\cdots\Am_{j} \right)^T= \left(\mcode{G}(\umi{1})\Bm_{2:j}\right)^T \in\Rss{\dm_j}{\dm_{1}} \ (j>1).
\end{split}
\end{equation*}
The proof is completed.
\end{proof}

\end{document}